\documentclass[sn-chicago]{sn-jnl}

\jyear{2021}%
\usepackage{verbatim}
\usepackage{microtype}
\usepackage{enumitem}
\usepackage{booktabs} 
\usepackage{color,graphicx}
\usepackage{comment}
\usepackage{thmtools, thm-restate}
\usepackage{natbib}
\usepackage{hyperref,xcolor}

\usepackage{amssymb,amsmath,dsfont,bm,mathtools, bbm} 
\newcommand{\red}[1]{{\leavevmode\color{red}{#1}}}

\newcommand{\green}[1]{{\leavevmode\color[RGB]{0,128,0}{#1}}}
\newcommand{\purple}[1]{{\leavevmode\color[RGB]{128,0,255}{#1}}}

\newcommand\todo[1]{{\red{TODO: {#1}}}}
\newcommand\tocite{\red{[CITE]}}

\newcommand\toref{\red{[REF]}}

\newcommand{\g}[1]{\green{[Ganesh: {#1}]}}
\newcommand{\gadd}[1]{\green{{#1}}}
\newcommand{\shweta}[1]{\purple{{#1}}}

\newcommand{\calC}{\mathcal{C}}

\newcommand{\balance}{\textsf{Balance}}
\DeclareMathOperator*{\argmax}{argmax}
\DeclareMathOperator*{\argmin}{argmin}

\newtheorem{theorem}{Theorem}

\newtheorem{definition}{Definition}
\newtheorem{lemma}[theorem]{Lemma}
\newtheorem{claim}[theorem]{Claim}
\newtheorem{proposition}[theorem]{Proposition}
\newtheorem{corollary}[theorem]{Corollary}
\newtheorem{conjecture}[theorem]{Conjecture}
\newcommand{\RR}{\textsc{FRAC}}
\newcommand{\bad}{\text{bad}}
\newcommand{\good}{\text{good}}

\newcommand{\tfair}{\textsf{$\tau$-ratio}}
\usepackage{multirow}
\usepackage{booktabs}
\usepackage{subcaption}
\usepackage{adjustbox}
\newcommand{\RROE}{\RR$_{OE}$}
\usepackage{cancel}
\usepackage{graphicx}
\usepackage[linesnumbered,ruled,vlined]{algorithm2e}
\SetKwInput{KwInput}{Input}                
\SetKwInput{KwOutput}{Output}

\theoremstyle{thmstyleone}%

\theoremstyle{thmstyletwo}%

\theoremstyle{thmstylethree}%

\begin{document}

\title[Efficient Algorithms For Fair Clustering with a New Fairness Notion]{Efficient Algorithms for Fair Clustering with a New Notion of Fairness }


\author*[1]{\fnm{Shivam} \sur{Gupta}}\email{shivam.20csz0004@iitrpr.ac.in}

\author[2]{\fnm{Ganesh} \sur{Ghalme}}\email{ganeshghalme@ai.iith.ac.in}

\author[1]{\fnm{Narayanan} \sur{C. Krishnan}}\email{ckn@iitrpr.ac.in}

\author[1]{\fnm{Shweta} \sur{Jain}}\email{shwetajain@iitrpr.ac.in}

\affil*[1]{\orgname{Indian Institute of Technology (IIT)} \orgaddress{ \city{Ropar}, \country{India}}}

\affil[2]{\orgname{Indian Institute of Technology (IIT) Hyderabad}, \orgaddress{\city{Kandi}, \country{India}}}


\abstract{
    We revisit the problem of fair clustering,  first introduced by  \citet{chierichetti2018fair}, that requires each protected attribute to have approximately equal representation in every cluster; i.e., a \balance\ property.  Existing solutions to fair clustering are either not scalable or do not achieve an optimal trade-off between clustering objective and fairness. In this paper, we propose a new notion of fairness, which we call  \tfair\ fairness, that strictly generalizes the  \balance\ property and enables a  fine-grained efficiency vs. fairness trade-off. Furthermore, we show that  simple greedy round-robin based algorithms achieve this trade-off efficiently. Under a more general setting of multi-valued protected attributes, we rigorously analyze the theoretical properties of the our algorithms. Our experimental results suggest that the proposed solution  outperforms all the state-of-the-art algorithms and works exceptionally well even for a large number of clusters.}



\maketitle

\section{Introduction}

Advances in machine learning research  have resulted in the development of increasingly accurate models, leading to the wide adoption of these algorithms in applications ranging from self-driving cars, approving home loan applications, criminal risk prediction, college admissions, and health risk prediction. While improving the accuracy is the primary objective of these algorithms, their use to allocate social goods and opportunities such as access to healthcare and  job and educational opportunities warrants a closer look at the societal impacts of their outcomes (\cite{carey2022fairness}; \cite{ntoutsi2020bias}). Recent studies have exposed a discriminatory outlook in the outcomes of these algorithms leading to  treatment disparity towards  individuals belonging to marginalized groups based on gender and race in real-world applications like automated resume processing \citep{amazon},  loan application screening, and criminal risk prediction \citep{julia2016propublica}.    Designing fair and accurate machine learning models is thus an essential and immediate requirement for these algorithms to make a meaningful real-world impact.

While fairness in  supervised learning  is studied 
\citep{introCiteFair, fairClassifier4, introCiteFair2,MLBiasSurvey,le2022survey,dwork2012fairness}, the fairness in unsupervised learning is still in its formative stages (\cite{ClusteringFairSurvey1}; \cite{SurveyClustering}). To emphasize the importance of fairness in unsupervised learning, we consider the following hypothetical scenario:  An employee-friendly company is looking to open branches at multiple locations across the city and distribute its workforce in these branches to improve work efficiency and minimize overall travel time to work. The company has employees with diverse backgrounds  based on, for instance, race and gender and does not prefer any group of employees over other groups based on these attributes. The company's diversity policy dictates hiring a minimum fraction of employees from each group in every branch. Thus, the natural question is: where should the branches be set up to maximize work efficiency, minimize travel time, and maintain diversity. In other words, the problem is to devise an unsupervised learning algorithm for identifying branch locations with the fairness (diversity) constraints applied to each branch. This problem can be naturally formulated as a clustering problem with additional fairness constraints on allocating the data points to the cluster centers.
Clustering, along with classification, forms the core of powerful machine learning algorithms with significant societal impact through applications such as automated assessment of job suitability  \citep{Whither} and facial recognition \citep{Li_2020_CVPR}.   These constraints arise naturally in applications where data points correspond to individuals, and cluster association signifies the partitioning of individuals based on features.

Typically, fairness in supervised learning is measured by the algorithm's performance over different groups based on protected(sensitive) attributes such as gender, race, and ethnicity. The first fairness notion for clustering was  proposed by  \cite{chierichetti2018fair}, wherein each cluster is required to exhibit a  \balance; defined as the ratio of protected attribute and non-protected attribute in each cluster  to the level of this ratio in the entire dataset. Their methodology--- apart from having significant computational complexity---applies only to binary-valued protected attributes and does not allow for trade-offs between the clustering objective and fairness guarantees. The subsequent literature \cite{backurs2019scalable,FairCoresetsStreaminKmeans,CoresetLowDimensionSpace,CoresetsFairnessHuang} improve efficiency; however, do not facilitate explicit  trade-off between the clustering objective  cost  and the fairness guarantee. 
In this paper we define a new notion of fairness which we call \tfair\ guarantee. To each cluster, a \tfair\ guarantee ensures a certain fraction of data points for a  given protected attribute. We show that this simple notion of fairness has several advantages. First, the definition of \tfair\  naturally extends to multi-valued protected attributes; second \tfair\ fairness strictly generalizes the \balance\ property; third,   it admits an intuitive and computationally efficient round-robin approach to fair  allocation; and fourth, it is straightforward  for the algorithm designer to  input the requirement into the algorithm as constraints and easy to interpret and evaluate it from the output. In our running example, if a company wants to have minimum fraction of employees from each group in every branch (clusters) then one can simply specify this  in the form of a  vector $\tau$ of size equal to number of protected groups. Through rigorous theoretical analysis, we show that the proposed algorithm \RROE\ provides a $2(\alpha + 2)$-approximate guarantee on the objective cost with  \tfair\ fairness guarantee up to three clusters. Here, $\alpha$ is the approximation factor achieved by the vanilla clustering algorithm. 
We further experimentally demonstrate that our approach can achieve better clustering objective costs  than any state-of-the-art (SOTA) approach on real-world data sets, even for a large number of clusters. Overall, the following are the contributions of our work.

\subsection{Our Contribution }
\label{sec:contri}
\paragraph{Conceptual Contribution}
 We introduce a new notion of fairness which we call a \tfair\ guarantee and show that any algorithm satisfying a \tfair\ guarantee also satisfies the  \balance\ property  (Theorem \ref{thm:balance}). Also, we show that every parameter setting of \balance\ collapses to a degenerate value of \tfair\ fairness showing generalisation of proposed notion.  We propose two simple and efficient round-robin-based algorithms for the \tfair\  fair allocation problem (see, Section \ref{sec:algo}). Our algorithms use the clustering algorithm as a black-box implementation and modify its output appropriately to ensure \tfair\ guarantee. The fairness guarantee is deterministic and verifiable, i.e., holds for every run of the algorithm, and can be verified from the outcome without explicit knowledge of the underlying clustering algorithm. The guarantee on objective cost, however, depends on the approximation guarantee of the clustering algorithm. 
 
Our algorithms can handle multi-valued protected attributes, allow  user-specified bounds on \balance, are computationally efficient, and incur only an additional time complexity of $O(kn\log(n))$, best in the current literature. Here, $n$ is the size of the dataset, and $k$ is the number of clusters. 

\paragraph{Theoretical Contributions}

We show theoretical guarantees for our first algorithm; $\textsc{Frac}_{OE}$.  First, we show that our algorithm achieves $2(\alpha+2)$-approximate fairness for clustering instances  upto three clusters  (Theorem \ref{thm:main} and Lemma \ref{thm:k=3}) with respect to optimal fair clustering cost for $\tau$=$1/k$; here $\alpha$ is a clustering algorithm specific  constant. That is, given a fair clustering instance with $k \leq 3$ clusters, $n$ datapoints and a  fairness vector $\tau$, our proposed algorithm returns an allocation that has objective cost of $2(\alpha +2)$ times the objective cost of optimal assignment that also satisfies the \tfair\ guarantee. We further show that this guarantee is tight (Proposition \ref{prop:tightness}).    
 For $k>3$ clusters we  show $2^{k-1}(\alpha + 2)$-approximation guarantee on the \tfair . We conjecture that the exponential dependence of the approximation guarantee on $k$ can be reduced to a constant. The guarantees are extended to work for any general $\tau$ vector (see Section \ref{sec:generalTauFrac}). We also theoretically analyse the convergence of \RROE\ (Lemma \ref{lemmaConverge}).
\paragraph{Experimental Contributions}
 Through extensive experiments on four datasets (Adult, Bank, Diabetes, and Census II), we show that the proposed algorithm outperforms all the existing algorithms on fairness and objective costs. Perhaps the most important insight from our experiments is that the performance of our proposed algorithms does not deteriorate with increasing $k$, experimentally validating our conjecture. We compare our algorithms with SOTA algorithms for their fairness guarantee, objective cost, and runtime analysis.  We also note that our algorithms do not require hyper-parameter tuning,  making our method easy to train and scalable. 
 While our algorithms are applicable to  center based clustering approach, we demonstrate its efficacy using $k$-means and $k$-median.

\section{Related Work}
\label{sec:relatedWork}
 While there is abundant literature on  fairness in supervised learning (\cite{fairClassifier4};
\cite{fairClassfier5}; 
\cite{fairClassfier6};
\cite{fairClassifier}; \cite{fairClassifier1};  \cite{fairClassifier2}; 
\cite{fairClassifier3};), research on fair clustering is still in infancy and is rapidly gathering attention (\cite{chierichetti2018fair}; \cite{FairDataSumma}; \cite{ziko2019variational};  \cite{liu2021stochastic}; \cite{davidson2020making};\cite{bercea2018cost}; \cite{SurveyClustering}). These studies include extending the existing fairness notions such as group and individual fairness to clustering (\cite{bera2019fair}; \cite{ANotionOfIndiFair4clust}; \cite{ProportionallyFairClustering}), proposing new problem-specific fairness notions such as social fairness (\cite{FairEquitableRepres}; \cite{approximationAlgoSociallyFair}),  characterizing the fairness and efficiency trade-off  (\cite{ziko2019variational}; \cite{FairnessMultipleAttri} ) and developing and analyzing fair and efficient algorithms  (\cite{CoresetApplication}; \cite{FairCoresetsStreaminKmeans}). 

The fairness in clustering is introduced at different stages of implementation namely -- pre-processing, in-processing and post-processing. 

\noindent \textbf{Pre-processing:}
Following a disparate impact doctrine (\cite{disparateImpact}), \citet{chierichetti2018fair}, in their pioneering work, defines fairness in clustering through a \balance\ property. \balance\ is the ratio of data points with different protected attribute values in a cluster. A balanced clustering ensures \balance\  in all the clusters equal to the  \balance\ in the original dataset (see Definition \ref{eqnBalanceRelated}).  \citet{chierichetti2018fair} achieve balanced clustering through the partitioning of the data into balanced sets called fairlets, followed by merging of the partitions. Subsequently,  \cite{backurs2019scalable} proposes an efficient algorithm to compute the fairlets.  Both the approaches have two major drawbacks: they are limited to the datasets having only binary-valued protected attributes, and can only create clusters exhibiting the exact \balance\ present in the original dataset, thereby not being flexible in achieving an optimal trade-off between \balance\ and accuracy. \cite{FairCoresetsStreaminKmeans} extend the notion of coresets to fair clustering and provide an efficient and scalable algorithm using \emph{composable} fair coresets (see also  \cite{ CoresetsFairnessHuang, CoresetLowDimensionSpace,CoresetApplication,DimenReducSumOfMetric}).  A coreset is a set of points approximating the optimal clustering objective value for any $k$ cluster centers. Though the coreset construction can be performed in a single pass over the data as opposed to the fairlets construction, storing coresets takes exponential space in terms of the dimension of the dataset. \citet{CoresetApplication} though reduces this exponential size requirement to linear in terms of space; the algorithm still has the running complexity that is exponential in the number of clusters. Our proposed approach is efficient because we do not need any additional space. Simultaneously, the running complexity is linear in the number of clusters and near-linear in the number of data points.   

\noindent \textbf{In-processing: } \cite{multiplecolor} propose an ($\alpha$+2)-approximate algorithm for fair clustering using minimum cost-perfect matching algorithm. While the approach works with a multi-valued protected attribute,  it has O($n^3$) time complexity and is not scalable. \cite{ziko2019variational} propose a  variational framework for fair clustering.   Apart from being applicable on datasets with multi-valued protected attributes, the approach works for both prototype-based ($k$-mean/$k$-median) and graph-based clustering problems ($N$-cut or Ratio-cut).  However, the sensitivity of the hyper-parameter to various datasets and the number of clusters necessitates extensive tuning rendering the approach computationally expensive. Further, the clustering objective also deteriorates significantly under strict fairness constraints when dealing with many clusters (refer Section \ref{sec:varyExpCluster}).  Along the same lines, \cite{FairnessMultipleAttri}  devise an optimization-based approach for fair clustering with multiple multi-valued protected attributes with a trade-off hyper-parameter similar to \cite{ziko2019variational}. 

\noindent \textbf{Post-processing: }\cite{bera2019fair} converted fair clustering into a fair assignment problem and formulated a linear programming (LP) based solution. The LP-based formulation leads to a higher execution time (refer to Section \ref{sec:runtimeAnalysis}). Also, the approach fails to converge when dealing with a large number of clusters. 
The proposed approach takes a similar route as  \cite{bera2019fair} to convert the fair clustering problem into a fair allocation problem. However, we give a simple polynomial-time algorithm which, in  $O(nk\log n)$ additional computations,  guarantees   a more general notion of fairness which we call \tfair\ fairness.  Our allocation algorithms  have following main advantages over the current state of the art;
\begin{enumerate}
    \item they are  computationally efficient,
    \item they  work for multi-valued protected attributes,
    \item no hyperparameter tuning is required and,
    \item they are simple and more interpretable (refer Section \ref{sec:prelimm}).
\end{enumerate}

 The work by \cite{bera2019fair} is extended by \cite{KFCScalable} for $k$-center problem whereas we in present study consider $k$-means and $k$-median based centering techniques. Similarly the works by (\cite{ahmadian2019clustering,FairKcenterMaxMatching,bandyapadhyay2019constant,jia2020fair,anegg2020technique,chakrabarti2022new,brubach2020pairwise}) are applicable only for $k$-center clustering.
 While we focus on the fairness notion of \balance\ based on the protected attribute value, other perspectives on fairness are defined in the literature. \cite{ANotionOfIndiFair4clust} define individual fairness: every data point on average is closer to the points in its cluster than to the points in any other cluster, while \cite{ProportionallyFairClustering,individualFairnessKclustering,vakilian2022improved,negahbani2021better} uses a radii-based approach to characterize fairness. \cite{SociallYFairKmeansClustering}; \cite{FairEquitableRepres}; \cite{RepresentativityFairnessClustering}; \cite{approximationAlgoSociallyFair}; \cite{goyal2021tight} study social fairness inspired by  equitable representation. This body of work mainly seeks to equalize the objective cost across all groups.  The notion of  proportionally fair clustering is proposed by  (\cite{proportionallyFairCluste,PropFairClusRevis}) wherein subset of points are allowed to form their own clusters if a center exists that is close to all points in subset. While existing works tightly integrate achieving fairness with the clustering algorithms, \cite{FairAntidoteData} recently devised the idea to use a pre-processing technique by addition of a small number of extra data points called antidotes. Vanilla clustering techniques applied to this augmented dataset result in fair clusters with respect to the original data. The pre-processing technique to add antidotes requires solving a bi-level optimization problem. While the pre-processing routine makes fair clustering algorithms irrelevant, its high running time limits its usability.  

Another line of related works studying fairness in clustering revolves around hierarchical clustering, spectral clustering algorithms for graphs, and hypergraph clustering  (\cite{bose2019compositional};\cite{kleindessner2019guarantees}).  \cite{FairKcenterMaxMatching} define fairness on the cluster centers, wherein each center comes from a demographic group. Clustering has also been used for solving fair facility location problems (\cite{serviceInNeighFacilityLocation, PropFairClusRevis, ProportionallyFairClustering}).  Recently,  \cite{ApproxGroupFairnessClustering} propose a new fairness notion of core fairness that is motivated by both group and individual fairness (\cite{kar2021feature}). \cite{FairAlgoAllocProblem} use fair clustering for resource allocation problems. \cite{FairDataSumma} use fair clustering for data summarization. Fair clustering is also being studied in dynamic (\cite{FullyDynamicKcenter}), capacitated (\cite{fairCapaciClus}), bounded cost (\cite{FairBoundedCost}), budgeted (\cite{byrka2014improved}), privacy preserving (\cite{rosner2018privacy}), probabilistic (\cite{esmaeili2020probabilistic}), correlated (\cite{CorrelatedClustering}), diversity aware (\cite{thejaswi2021diversity}) and distributed environments (\cite{DistribCluster}).  Finally, our fairness notion (\tfair), resembles to that of  balanced (in terms of number of points in each cluster) clustering studied by  \cite{BalanceScalableAlgoACMend} without fairness constraint.  However, their proposed sampling technique  is not designed to guarantee \tfair\  fairness and does not analyze loss incurred due to having these fairness constraint.

\section{Preliminaries} 
\label{sec:prelimm}
Let $X \subseteq \mathbb{R}^{d}$ be a finite  set of points that needs to be partitioned into $k$ clusters. Each data point $x_i \in X$ is a feature  vector described using $d$ real valued features.  A $k$-clustering \footnote{Throughout the paper, for simplicity, we call a $k$-clustering algorithm as a clustering algorithm.} algorithm $\mathcal{C} = (C, \phi)$ produces a partition of $X$ into $k$ subsets ($[k]$) with centers $C = \{c_j\}_{j=1}^{k}$ using an assignment function $\phi:X \rightarrow C$ which maps each point to corresponding cluster center.  Throughout this paper we consider that each point $x_{i} \in X$ is associated with a \emph{single}  protected attribute $\rho_i$ (say ethinicity from a pool of other available protected attributes) which takes values from the set $m$ values denoted by $[m]$. The number of distinct protected attribute values is finite and much smaller than size of set $X$ \ \footnote{Otherwise,  the problem is uninteresting as the balanced clustering may not be feasible.}.   Furthermore, let $d: X \times X  \rightarrow \mathbb{R}_{+}$ be a distance metric defined on $X$ and  measures the dissimilarity between features.  Additionally, we are also given a vector $\tau = \{\tau_{\ell}\}_{\ell=1}^m$ where each component $\tau_{\ell}$ satisfies $0 \le \tau_{\ell} \le \frac{1}{k}$ and denotes the fraction of data points from the protected attribute value $\ell \in [m]$ required to be present in each cluster. An end-user can  simply specify a $\ell$ dimensional vector with values between $0$ to $1/k$ as fairness target.  Also, let us denote $X_{\ell}$, $n_{\ell}$ as set of datapoints and number of points having value $\ell$ in $X$. Let $\mathbb{I}{(.)}$ denote the indicator function.  A  vanilla (an unconstrained) clustering algorithm determines the cluster centers as to minimize the clustering objective cost which is defined as follows:

\begin{definition}[Objective Cost] Given $p$, the cluster objective cost with respect to the metric space $(X,d)$ is defined as:
\begin{equation}
L_p(X, C, \phi) = \left(\sum_{x_i \in  X }\sum_{j \in [k] } \mathbb{I}(\phi(x_i) =j)d(x_i,c_j)^p\right)^\frac{1}{p}
\label{eqn:objCostPrelim}
\end{equation}
\end{definition}
Different values of $p$, will result in different objective cost: $p=1$ for $k$-medians, $p=2$ for $k$-means, and $p =\infty$ for $k$-centers. Our aim is to develop an algorithm that minimizes the  objective cost irrespective of $p$ while ensuring the fairness.

\noindent \textbf{Group Fairness Notions: }We begin with first defining the most popular notion of group fairness which is called \balance. 
The notion is first put forward for binary protected groups by \cite{chierichetti2018fair} and extended to multi-valued group by \cite{bera2019fair,ziko2019variational}. The balanced fairness notion is defined as follows.


\begin{definition}[\balance][\cite{chierichetti2018fair}]
The \balance\ of an assignment function $\phi$ is defined as
\begin{equation}
\begin{split}
\balance(\phi) = \min_{j \in [k]}\left( 
 \min\left( \frac{\sum_{x_i \in X}\mathbb{I}{(\phi(x_i)=j)} \mathbb{I}{(\rho_i=a)}}{\sum_{x_i \in X}\mathbb{I}{(\phi(x_i)=j)} \mathbb{I}{(\rho_i=b)}} \right)
\right)\ \forall a,b \in [m]
\end{split}
\label{eqnBalanceRelated}
\end{equation}
\end{definition}

\balance\ is computed by finding the minimum possible ratio of protected (say. male) and non-protected group (say. female) over all clusters. 
Any fair clustering algorithm using \balance\ as a measure of fairness would produce clusters that maximize the \balance.  Note that the maximum \balance\ achieved by an algorithm is equal to the ratio of points available in the dataset having $a$ and $b$ as the protected attribute values and is known as dataset balance. Further, the clusters maximizing the \balance\ are not unique.

A generalization of \balance\ to multi-valued protected attributes is proposed by \cite{bera2019fair} in terms of cluster sizes. The fairness notion constraints the upper and lower bound on the number of points from each protected group in every cluster.

\begin{definition}[Minority Protection] A clustering $\mathcal{C}$ is $\tau$\textsc{-MP} if 
\begin{equation}
\label{tauMP}
    \sum_{x_i \in X } \mathbb{I}{(\phi(x_i)=j)}\mathbb{I}{( \rho_i = \ell )} \ge \tau_\ell\  \sum_{x_i \in X } \mathbb{I}{(\phi(x_i)=j)} 
   \ \forall \ell \in [m], \forall j \in [k]
\end{equation}
\end{definition}

\begin{definition}[Restricted Dominance] A clustering $\mathcal{C}$ is $\tau$\textsc{-RD} if
\begin{equation}
\label{tauRD}
   \sum_{x_i \in X} \mathbb{I}{( \rho_i = \ell )} \mathbb{I}{(\phi(x_i)=j)} \le \tau_\ell\  \sum_{x_i \in X } \mathbb{I}{(\phi(x_i)=j)}
   \ \forall \ell \in [m], \forall j \in [k] 
\end{equation}
\end{definition}

The generalization by \cite{bera2019fair} needs cluster sizes that are not known beforehand. Thus, \cite{bera2019fair} proposes a linear programming-based solution.

We now define our proposed \tfair\ fairness notion which ensures that each cluster has a predefined fraction of points for each protected attribute value. \tfair\ requires only priorly known dataset composition, which helps achieve polynomial-time algorithms.

\begin{definition}[\tfair\ Fairness]
An assignment function $\phi$ satisfies \tfair\ fairness if
\begin{equation}
\sum_{x_i \in X}\mathbb{I}{( \phi(x_i) = j)}\mathbb{I}{(\rho_i=\ell)} \ge \tau_{\ell}\sum_{x_i \in X} \mathbb{I}{(\rho_i=\ell)}\ \forall j \in [k] \text{ and} \  \forall \ell\ \in [m]
\end{equation}
\end{definition}

The \tfair\ fairness is different from the balanced fairness of \cite{chierichetti2018fair} that tries to \balance\ the ratio of points for any pair of values corresponding to the protected attribute in each cluster. 

 Our first theorem (Theorem \ref{thm:balance}) in Section \ref{sec:TheorySection} shows that an algorithm satisfying \tfair\ fairness notion produces one set of clusters that maximizes the \balance. In particular, when $\tau_{\ell} = \frac{1}{k}$, then \tfair\ fairness achieve the \balance\ equal to the dataset ratio. We also show that a perfectly balanced cluster need not imply \tfair\ fairness for arbitrary $\tau$ (Lemma \ref{lemmaObervation} in Section \ref{sec:TheorySection}). Hence \tfair\ is a more generalized fairness notion.
 
 We now define the fair clustering problem with respect to the proposed fairness notion:

\begin{definition}[\tfair\ Fair Clustering Problem]
The objective of a \tfair\ fair clustering problem $\mathcal{I}$ is to estimate $\mathcal{C} = (C, \phi)$ that minimizes the objective cost $L_p(X, C, \phi)$ subject to the \tfair\ fairness guarantee. The optimal  objective cost of a \tfair\ fair clustering problem is denoted by $\mathcal{OPT}_{clust}(\mathcal{I})$.
\end{definition}

A solution to this problem is to rearrange the points (learn a new $\phi$) with respect to the cluster centers obtained after a traditional clustering algorithm to guarantee \tfair\ fairness. The problem of rearrangement of points with respect to the fixed centers is known as the fair assignment problem, which we define below:
\begin{definition}[\tfair\ Fair Assignment Problem]
Given $X$ and $C = \{c_j\}_{j=1}^{k}$, the solution to the fair assignment problem $\mathcal{T}$ produces an assignment $\phi: X \rightarrow C$ that ensures \tfair\ fairness and   minimizes $L_p(X, C, \phi)$. The optimal objective function value to a \tfair\ fair assignment problem is denoted by $\mathcal{OPT}_{assign}(\mathcal{T})$. 
\label{defTaurRatioProblem}
\end{definition}

However, this transformation of the fair clustering problem $\mathcal{I}$ into a fair assignment problem $\mathcal{T}$ should ensure that $\mathcal{OPT}_{assign}(\mathcal{T})$ is not too far from $\mathcal{OPT}_{clust}(\mathcal{I})$. The connection between fair clustering and fair assignment problem is established through the following lemma.  
\begin{lemma}
\label{lemmabera}
Let $\mathcal{I}$ be an instance to fair clustering problem and $\mathcal{T}$ is an instance to \tfair\ fair assignment problem after applying $\alpha$-approximate algorithm to the vanilla clustering problem, then $\mathcal{OPT}_{assign}(\mathcal{T}) \le (\alpha + 2)\mathcal{OPT}_{clust}(\mathcal{I})$.
\label{bera_stmt}
\end{lemma}
\begin{proof}
Let $C$ the cluster centers obtained by running a vanilla clustering algorithm on instance $\mathcal{I}$. The proof of the Lemma depends on the existence of an assignment $\phi'$ satisfying \tfair\ fairness such that $L_p(X, C, \phi') \le (\alpha + 2) \mathcal{OPT}_{clust}(\mathcal{I})$. As $\mathcal{OPT}_{assign}(\mathcal{T}) \le L_p(X, C, \phi') \le (\alpha + 2) \mathcal{OPT}_{clust}(\mathcal{I})$.  Let $(C^*,\phi^*)$ denote the optimal solution to $\mathcal{I}$. Define $\phi'$ as follows: for every $c^* \in C^*$, let $nrst(c^*) = \argmin_{c \in C} d(c,c^*)$ be the nearest center to $c^*$. Then, for every $x\in X$, define $\phi'(x) = nrst(\phi^*(x))$. Then we have the following two claims:
\begin{claim}
$\phi'$ satisfies \tfair\ fairness.
\end{claim}
\begin{proof}
Let set of points having protected attribute value ${\ell}$ in cluster $c^* \in C^*$ be $n_{\ell}(c^*)$. Since $(C^*, \phi^*)$ satisfy \tfair\ fairness we have $\vert$ $n_{\ell}(c^*)$ $\vert$ $\ge \tau_{\ell}n_{\ell}$ $\forall c^* \in C^*$. For any center $c \in C$, let $N(c) = \{c^* \in C^*: nrst(c^*) = c\}$ be all the centers in $C^*$ for which $c$ is the nearest center.
Then: $\vert \{x\in X_{\ell}: \phi'(x) = c\} \vert = \vert \cup_{c^* \in N(c)}n_{\ell}(c^*) \vert \ge n_{\ell}\tau_{\ell}$  that is union over combined assignments for each center in $N(c)$ and since each set of assignments satisfy \tfair\ so union will also satisfy \tfair\ fairness. 
\end{proof}
\begin{claim}
$L_p(X, C, \phi') \le (\alpha+2)\mathcal{OPT}_{clust} (\mathcal{I})$.
\end{claim}
The proof of this claim uses triangle inequality and is exactly same as claim 5 of \cite{bera2019fair}.
\end{proof}
A similar technique of converting fair clustering to a fair assignment problem was proposed by \cite{bera2019fair}. However, \cite{bera2019fair} proposed a linear programming based solution to obtain the \balance\ fair assignment. Although, the solution is theoretically strong, there are two issues with the algorithm. Firstly, the time complexity is high (as can be seen from the experiments in Section \ref{sec:runtimeAnalysis}) and secondly, the solution obtained is not easy to interpret due to the use of the complicated linear program. By interpretability we try to find the answer to the following question --  Why is a point assigned to a specific cluster to maintain fairness? What criteria did the algorithm decide for a data-point to go to a particular cluster? To answer these, our paper proposes a simple round-robin algorithm for fair assignment problem with a time complexity of $O(kn\log(n))$. 


\section{Fair Round-robin Algorithm for Clustering Over End (\RROE)}
\label{sec:algo}

\begin{algorithm}[h]
\SetAlgoLined
\KwInput{set of datapoints $X$, number of clusters $k$, fairness requirement vector $\tau$, range of protected attribute values $m$, clustering objective norm $p$}
 \KwOutput{cluster centers $\hat{C}$ and assignment function $\hat{\phi}$}
  Solve the vanilla $(k,p)$-clustering problem and let $(C, \phi)$ be the solution obtained.\\
  \If{\tfair\ fairness is met}{
  return $(C, \phi)$\\
  \Else{
  $(\hat{C}, \hat{\phi})$ = \textsc{FairAssignment}($C, X, k, \tau, m, p, \phi$)\\
  return $(\hat{C}, \hat{\phi})$
  }}
 \caption{$\tau$-\RROE\ }
 \label{algo:FRAC-OE}
\end{algorithm}

\begin{algorithm}[h]
\SetAlgoLined
\KwInput{Cluster centers $C$, Set of datapoints $X$, Number of clusters $k$, Fairness requirement vector $\tau$, Range of protected attribute $m$, clustering objective norm $p$, Assignment function $\phi$} 
 \KwOutput{Cluster centers $\hat{C}$ and assignment function $\hat{\phi}$}
Fix a random ordering on centers and let the centers are numbered from $1$ to $k$ with respect to this random ordering.\\
Initialize $\hat\phi(x)\gets 0\ \forall x \in X$.\\
\For{$\ell\gets1$ \KwTo $m$}{
$n_{\ell}\gets$ number of datapoints having value of protected attribute ${\ell}$.\\
$X_{\ell}\gets$ set of datapoints having value of protected attribute ${\ell}$.\\
 \For{$t\gets1$ \KwTo $\tau_{\ell}n_{\ell}$}{
 \For{$j\gets1$ \KwTo $k$}{
 $x_{min}\gets \argmin_{x \in X_{\ell}: \hat{\phi}(x)=0} d(x, c_j)$\\
 $\hat\phi(x_{min}) = j$\\
 }
 }
 For all $x \in X_{\ell}$ such that $\hat\phi(x) = 0$, set $\hat\phi(x) = \phi(x)$
 }
 Recompute the centers $\hat{C}$ with respect to the new allocation function $\hat\phi$.\\
 Return ($\hat{C}, \hat{\phi}$).
 \caption{\textsc{FairAssignment}}
 \label{algo:fair-assignment}
\end{algorithm}

Fair Round-robin Algorithm for Clustering Over End (\RROE) algorithm  first runs a vanilla clustering algorithm to produce the initial clusters $\mathcal{C} = (C,\phi)$  and then \emph{make corrections} as follows.  The algorithm first checks if \tfair\ fairness is met with the current allocation $\phi$, in which case it returns $\hat{\phi} = \phi$ and $\hat{C} = C$. If the assignment $\phi$ violates the \tfair\ fairness constraint then the new assignment function $\hat{\phi}$ is computed according to \textsc{FairAssignment} procedure in Algorithm \ref{algo:fair-assignment}. 

 Algorithm \ref{algo:fair-assignment} iteratively allocates the data points with respect to each protected attribute value. Let $X_{\ell}$ and $n_{\ell}$ denote the set of data points and the number of data points having $\ell$ as the protected attribute value. The algorithm allocates $\lfloor \tau_{\ell} \cdot  n_{\ell} \rfloor $ number of points \footnote{For the sake of simplicity we assume $\tau_{\ell} \cdot n_{\ell} \in \mathbb{N}$ and ignore the floor notation.  } to each cluster in a round-robin fashion as follows. Let $\{c_1, c_2, \ldots, c_k\}$ be a random ordering of the cluster centers. At each round $t$, each center $c_j$ picks the point $x$ of its preferred choice from $X_{\ell}$ i.e.  $\hat{\phi}(x) = j$. Once the $\tau_{\ell}$ fraction of points are assigned to the centers, i.e., after $ \tau_{\ell} \cdot n_{\ell}  $ number of rounds, the allocation of remaining data points is set to its original assignment $\phi$. Note that this algorithm will certainly satisfy \tfair\ fairness as, in the end, the algorithm assures that at least $\tau_{\ell}$ fraction of points are allotted to each cluster for a protected attribute value ${\ell}$. We defer to theoretical results to assert the quality of the clusters.  

\RROE\ ensures fairness at the last step. The run time complexity of Algorithm \ref{algo:fair-assignment} is $O(kn\log(n))$ as step $4$ requires the data points to be sorted in the increasing order of their distances with the cluster centers.
\section{Theoretical Results}
\label{sec:TheorySection}
Our first result provides the relationship between the two notions of fairness, namely \tfair\ fairness and the \balance\ fairness.

\begin{theorem}
\label{thm:balance}
 Let $a$ and $b$ be two values  of a given binary protected attribute with $n_a$ and $n_b$ being the total number of datapoints respectively. Suppose an allocation returned by a clustering algorithm satisfies \tfair\ guarantee,  then  the \balance\ of the given allocation is atleast $\frac{\tau_a n_a}{n_b(1 - k\tau_b + \tau_{b})}$.  
\end{theorem}

\begin{proof}
Suppose an algorithm satisfies \tfair\ fairness then for any cluster $C_j$ and protected attribute value $a$, we have:
\begin{align*}
\tau_an_a \le \sum_{x_i \in X}\mathbb{I}{(\phi(x_i) = j)}\mathbb{I}(\rho_i = a) \le n_a(1-k\tau_a+\tau_a) 
\end{align*}
Here, the lower bound comes directly from the fairness definition and upper bound is derived from the fact that all the clusters together will be allocated at least $k\tau_an_a$ number of points. The extra points that a particular cluster can take is upper bounded by $n_a - kn_a\tau_a$. Thus, the \balance\ of the cluster with respect to the two values $a$ and $b$ should follow 
\begin{align*}
    \frac{\sum_{x_i \in X}\mathbb{I}{(\phi(x_i) = j)}\mathbb{I}(\rho_i = a)}{\sum_{x_i \in X}\mathbb{I}{(\phi(x_i) = j)}\mathbb{I}(\rho_i = b)} \ge \frac{\tau_an_a}{n_b(1-k\tau_b+\tau_b)}
\end{align*}
\end{proof}
 We remark here that the notion of \balance\  which is concerned with allocation of the points to clusters such that each cluster satisfies the  dataset balance.  We now show that  the  \tfair\ guarantee strictly  generalizes \balance\ as follows. We first show  that setting $\tau_{i} = 1/k$ for all attributes values $i$ implies dataset balance. \begin{corollary}
For $\tau_a = \tau_b = \frac{1}{k}$, \tfair\ fairness guarantee ensures the dataset \balance\ for all the clusters.
\end{corollary}
This result follows from trivially by replacing the attribute constraints in Theorem \ref{thm:balance}. We now show that the converse is not true. That is, a  clustering satisfying \balance\ equal to dataset balance can result in arbitrary bad \tfair\ fairness.   

\begin{restatable}{lemma}{de}
\label{lemmaObervation}
There exists a fair  clustering instance and an allocation of points such that the allocation satisfies the \balance\ property and has arbitrarily low  \tfair\ fairness. 
\end{restatable}
\begin{proof}
 Consider a fair clustering instance with $k=2$ and let the protected attribute be binary; call them $a$ and $b$. Further, let $n_a = n_b = n/2$. It is easy to see that  the dataset balance is $1$. Consider the following allocation that satisfies the dataset balance for each cluster. Cluster $1$ is assigned two points, one belonging to each attribute value and rest of the points are allocated to cluster $2$. Note that for this allocation, $\tau_a = \tau_b = 1/n_a = 1/n_b = 2/n$. For large value of $n$ this value can be made arbitrarily small.     
\end{proof}



Along with Theorem \ref{thm:balance}, Lemma \ref{lemmaObervation} shows that  \tfair\  is a more general fairness  notion than \balance. Apart from above technical difference, these fairness notions differ  conceptually in the way they  induce fair clustering. The \balance\ property requires a certain minimum representation ratio guarantee to hold in each cluster without any additional constraint on relative size of each of the cluster. This may lead  to (potentially) skewed cluster sizes. Whereas under \tfair\ the algorithm can appropriately control the minimum number of points to be assigned to each cluster. 

We now provide the theoretical guarantees of \RROE\ with respect to \tfair\ fairness. We begin by providing  guarantees for a perfectly balanced clusters i.e. $\tau_\ell = 1/k\ \forall  \ell \in [m]$.
\subsection{Guarantees for \RROE\ for $\tau$=$\{1/k\}_{l=1}^m$ }
\label{SectionFRACoeGuran}
\begin{theorem}
\label{thm:main}
Let  $k=2$ and $\tau_{\ell} = \frac{1}{k}$  for all ${\ell} \in [m]$. An allocation returned by \RROE\   guarantees \tfair\ fairness and satisfies  $2$-approximation guarantee with respect to an optimal fair assignment upto an instance-dependent  additive constant. \end{theorem}
\begin{proof}

 \textbf{Correctness and Fairness:} Clear from the construction of the algorithm. \newline 
 \textbf{Proof of (approximate) Optimality:} We will prove $2$-approximation with respect to each value $\ell$ of protected attribute separately. Let $n_\ell$ be the number of data points corresponding the value $\ell$. Let  $c_1$ and $c_2$ be the cluster centers and $\mathcal{C}_1$, $\mathcal{C}_2$ be the optimal fair assignment of data points with respect to these centers.\footnote{Note that an optimal fair allocation need not be unique. Our result holds for any optimal fair allocation.}

We now show that \RROE($\mathcal{T}$)\ $\leq 2$ $\mathcal{OPT}_{assign}(\mathcal{T})$ + $\beta$, where \RROE($\mathcal{T}$) and $\mathcal{OPT}_{assign}(\mathcal{T})$ denote the objective value of the solution returned by \RROE\ and optimal assignment algorithm respectively on given instance $\mathcal{T} = (C, X)$. Let, $\beta := 2\sup_{x,y \in X} d(x,y)$ be the diameter of the feature space. We begin with the following useful definition. 
\begin{definition}
Let $\mathcal{C}_1$ and $\mathcal{C}_2$ represent the set of points assigned to $c_1$ and $c_2$ by optimal assignment algorithm. The $i^{th}$ round (i.e. assignments $g_{i}$ to $c_1$ and $h_{i}$ to $c_2$) of  \RROE\  is called \begin{itemize}
    \item $1$-\bad\ if exactly one of 1) $g_{i} \notin \mathcal{C}_1$ and 2) $h_{i} \notin \mathcal{C}_2$ is true, and    
    \item $2$-\bad\ if both  1)  and 2) above are true. 
\end{itemize}
 Furthermore, a round is called \bad\ if it is either $1$-\bad\ or $2$-\bad\ and called \good\ otherwise.
\end{definition}
Let all incorrectly assigned points in a \bad\ round be called \bad\ assignments. We use following convention to distinguish between different \bad\ assignments.  If $g_{i} \notin \mathcal{C}_1$ holds we refer to it as  type 1 \bad\ assignment i.e. if point $g_i$ is currently assigned to $\mathcal{C}_1$ but should belong to  optimal clustering $\mathcal{C}_2$.  Similarly if $h_i \notin \mathcal{C}_2$ holds it is a type 2 \bad\ assignment  i.e. $h_i$ should belong to optimal clustering  $\mathcal{C}_1$ but is currently assigned to $c_2$. Hence a $2$-\bad\ round results in 2 \bad\ assignments one of each i.e.  $g_i \in \mathcal{C}_2$ and $h_i$  $\in \mathcal{C}_1$.   Finally let  $B$ be the set of  all \bad\ rounds.
\begin{definition}{(Complementary Bad Pair) }
\label{defCmpPair}
A pair of points  $w, z \in B$ such that  $w$ is a \bad\ point of type $t$ and  
 $z$ is a \bad\ point of type $ \vert 3-t \vert $ is called a complimentary \bad\ pair if,  
 
 1) $w$ and $z$ are allocated in same round (i.e. a $2$-\bad\ round) or 
 
 2) if they are allocated in $i^{th}$ and $j^{th}$ $1$-\bad\ rounds respectively with $i<j$, then $z$ is the first \bad\ point of type $ \vert 3-t \vert $ which has not been assigned a complementary point. 
 \label{defPointToAssign}
 \end{definition}
\begin{restatable}{lemma}{ComplimentaryBad}
\label{lem:complimentary}
If $n_\ell$ is even, every \bad\ point in the allocation returned by  \RROE\  has a complementary point. If $n_\ell$ is odd, at most one \bad\ point will be left without a complementary point.  
\end{restatable}
\begin{proof}
 Let $B  = B_1 \cup B_2 $, where $B_t$ is a set of $t$-\bad\ rounds. Note that the claim is trivially true if $B_1 = \emptyset$. Hence, let $\vert B_1 \vert > 0$ and write $B_1 = B_{1,1} \cup B_{1,2}$. Here $B_{1,t}$ is a $1$-\bad\ round that resulted in type $t$ bad point. Let $H_{1, t}$ be  the set of \good\ points  of type $t$ (i.e. correctly assigned to the center $c_t$)  allocated in $1$-\bad\ rounds. 
 When $n_\ell$ is even, $\vert \mathcal{C}_1 \vert = \vert \mathcal{C}_2 \vert$ we have $\vert B_{1,2} \vert + \vert  H_{1,1} \vert=\vert B_{1,1} \vert + \vert H_{1,2} \vert $. This is true because one can ignore \good\ rounds and $2$-\bad\ rounds as   every $2$-\bad\ round can be converted into a \good\ round by switching the assignments.  
Further observe that, as \RROE\ distributes two points per round and each round assigns exactly one \bad\ point,  each round must assign exactly one \good\ point i.e. $\vert H_{1,t} \vert = \vert B_{1,(3-t)}\vert$. Together, we have $ \vert B_{1,1} \vert=  \frac{ \vert B_{1,2} \vert + \vert H_{1,1} \vert}{2} =  \vert B_{1,2} \vert $. When $n_\ell$ is odd, we might have one additional point left in the last 1-bad round that is not being assigned any complementary point. This completes the proof of the lemma. 
\end{proof}
\noindent We will bound the optimality of $1$-\bad\ rounds and $2$-\bad\ rounds separately.
\begin{figure}
\centering
\begin{subfigure}[b]{0.45\linewidth}
  \includegraphics[scale=0.2]{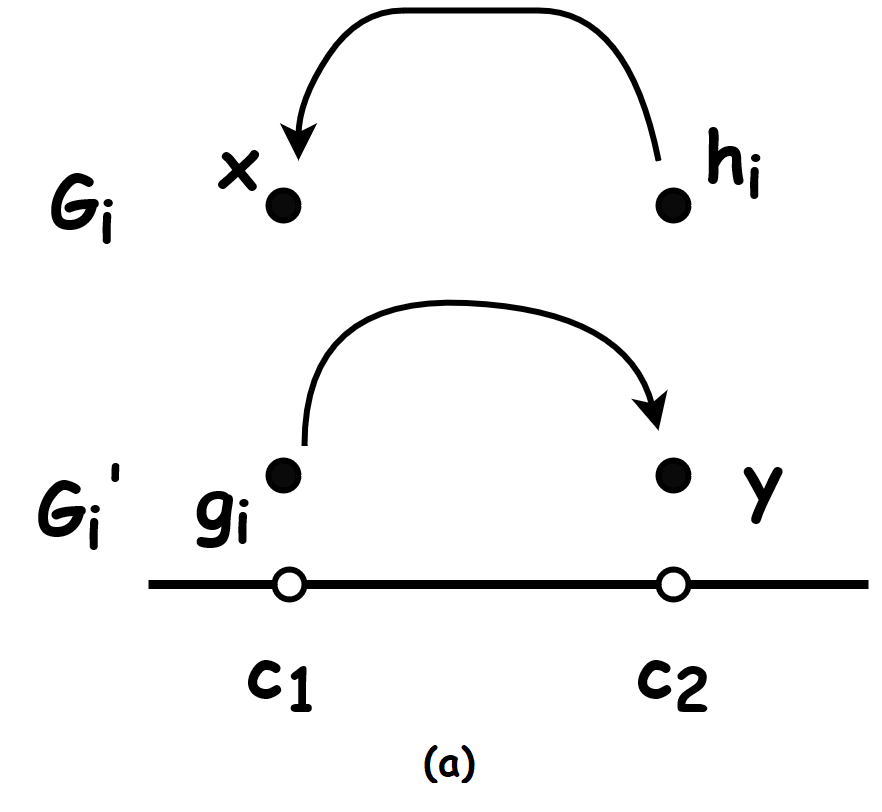}
  \caption{Two $1$-\bad\ round pairs}
  \label{fig:sub1}
\end{subfigure}%
\begin{subfigure}[b]{0.45\linewidth}
  \includegraphics[scale=0.2]{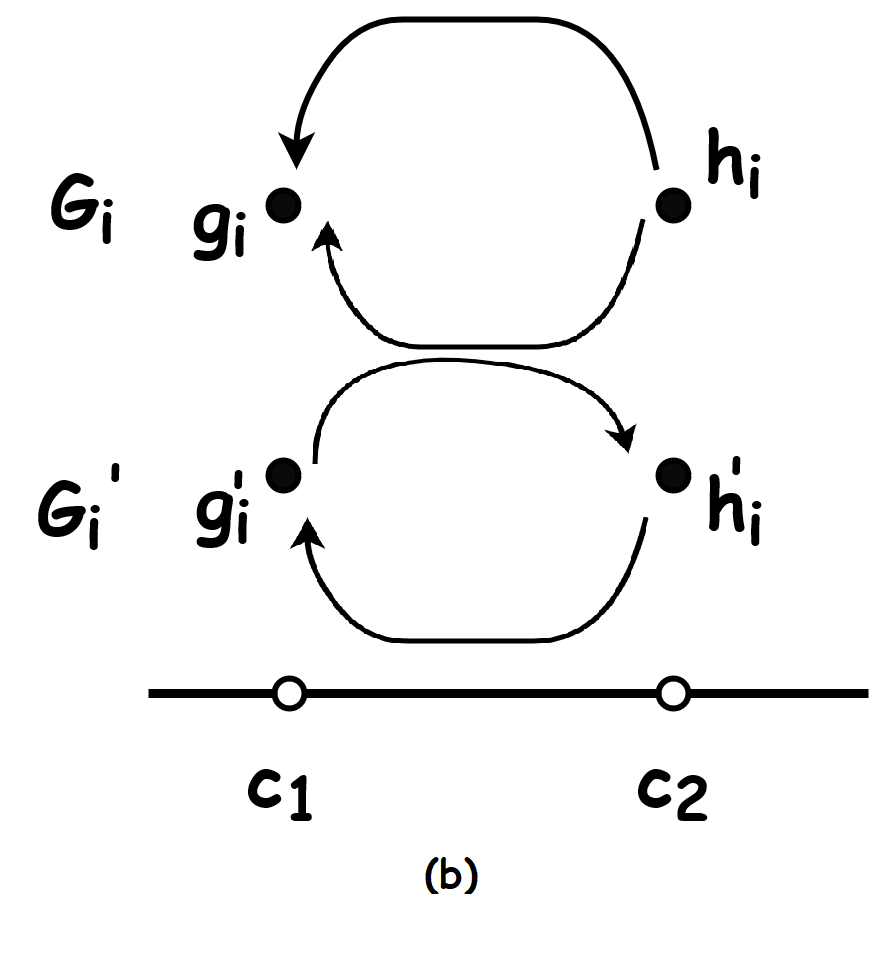}
  \caption{Two $2$-\bad\ round pairs}
  \label{fig:sub2}
\end{subfigure}
\caption{Different cases for $k=2$. (a) Shows two 1-bad rounds with four points such that $x$, $y$ are good points and allocated to the optimal center by algorithm, whereas $g_i$ and $h_i$ are bad points with an arrow showing the direction to the optimal center from the assigned center. (b) Shows four bad points such that $g_i$, $g_i'$ are assigned to $c_1$ but should belong to $c_2$ in optimal clustering (the arrow depicts the direction to optimal center). Similarly $h_i$, $h_i'$ should belong to $c_1$ in optimal clustering. }
\label{fig:k=2}
\end{figure}

\paragraph{Bounding $1$-\bad\ rounds:}
When $n_\ell$ is even, from Lemma \ref{lem:complimentary}, there are even number of $1$-\bad\ rounds; two for each complimentary bad pair. Let the $4$ points of corresponding two $1$-\bad\ rounds be  $G_{i}: (x, h_i)$  and $G_{i}^{'}: (g_i, y)$  as  shown in Fig. \ref{fig:sub1}. Note that $x \in \mathcal{C}_1$ and $y \in \mathcal{C}_2$  ie. both are good points  and $g_i \notin \mathcal{C}_1$, $h_i \notin \mathcal{C}_2$ ie. are bad points. Now, consider an instance $\mathcal{T}_i = \{C,\{x,h_i,g_i,y\}\}$, then 
 $ \mathcal{OPT}_{assign}(\mathcal{T}_i) =  d(x,c_1) + d(h_i, c_1)  + d(g_i, c_2) + d(y, c_2).$ We consider, without loss of generality, that the round $G_{i}$ takes place before $G_{i}^{'}$ in the execution of \RROE. The proof is similar for the other case.  First note that since \RROE\  allocated the point $h_i$ to cluster 2 while both the points $g_i$ and $y$ were available, we have 
\begin{equation}
    d(h_i, c_2) \leq d(g_i, c_2) \ \text{and} \ d(h_i, c_2) \leq d(y, c_2)
    \label{eq:AggAlloc}
\end{equation}
So,  
\begin{align*}
&\text{\RROE}(\mathcal{T}_i)\\
&=\ d(x,c_1) + d(h_i, c_2)  +  d(g_i,c_1)+ d(y, c_2)\\
&\le\ d(x, c_1) + d(h_i, c_2) + d(g_i,c_2) + d(c_1,c_2) + d(y,c_2) \tag{triangle  inequality}\\
&\le\ d(x, c_1) + d(h_i, c_2) + d(g_i,c_2) + d(h_i,c_2) + d(h_i,c_1) + d(y,c_2)\\
&\le\ d(x, c_1) + d(y, c_2) + d(g_i,c_2) + d(g_i,c_2) + d(h_i,c_1) + d(y,c_2)
\tag{ Eqn. \ref{eq:AggAlloc} }\\ 
&\le 2\ \mathcal{OPT}_{assign}(\mathcal{T}_i)\\
\end{align*}

If $n_\ell$ is odd, then all the other rounds can be bounded using the above cases except one extra $1$-\bad\ round. Let the two points corresponding to this round $G_i$ be $(g_i, y)$. Thus, 
$\text{\RROE}(\mathcal{T}_{i}) \le 2\mathcal{OPT}_{assign}(\mathcal{T}_{i}) + \beta$.Here $\beta$=$2 \sup_{x,y \in \mathcal{X}} d(x,y)$ is the diameter of the feature space.

\paragraph{Bounding $2$-\bad\ rounds:}
First assume that there are even  number of $2$-\bad\ rounds. In this case consider the pairs of  consecutive 2-\bad \  rounds as $G_i : (g_i, h_i)$ and $G_i^{'} = (g_i', h_{i}')$ with $G_i^{'}$ bad round followed by $G_{i}$ (Fig. \ref{fig:sub2}). Note that $g_i, g_i' \in \mathcal{C}_2$ and $h_i, h_{i}' \in \mathcal{C}_1 $. Now consider instance $\mathcal{T}_i = \{C, \{g_i, g_i', h_i, h_i'\}\}$, then , $\mathcal{OPT}_{assign}(\mathcal{T}_i) = d(h_{i}, c_1) + d(h_{i}', c_1) + d(g_{i}, c_2) + d(g_{i}', c_2) $. As a consequence of allocation rule used by \RROE\ we have 
\begin{equation}
\begin{split}
    &d(g_i, c_1) \leq d(h_i, c_1), \ d(g_i', c_1) \leq d(h_i', c_1)\ \text{and} \ d(h_i, c_2) \leq d(h_i', c_2). 
    \label{eq:algAllocTwo}
    \end{split}
\end{equation}
Furthermore, 
\begin{align*}
    \text{\RROE}(\mathcal{T}_i) &= d(g_i, c_1) + d(g_i', c_1) + d(h_i, c_2) + d(h_i', c_2) \\ 
    & \leq d(h_i, c_1) + d(h_i', c_1) + d(g_{i}', c_2) + d(h_i', c_2) \tag{using Eqn. \ref{eq:algAllocTwo}} \\ 
    & \leq d(h_i, c_1) + d(h_i', c_1) + d(g_{i}', c_2) + d(h_{i}', c_1) + d(c_1, c_2) \tag{triangle inequality}\\ 
    & \leq d(h_i, c_1) + d(h_i', c_1) + d(g_{i}', c_2)+ d(h_{i}', c_1) +  d(g_i, c_1)\\
    &+ d(g_i, c_2) \tag{triangle inequality}\\ 
    & \leq d(h_i, c_1) + d(h_i', c_1) + d(g_{i}', c_2) + d(h_{i}', c_1) +  d(h_i, c_1)\\
    &+ d(g_i, c_2) \tag{Using Eqn. \ref{eq:algAllocTwo}} \\ 
    & \leq 2d(h_i, c_1) + 2d(h_i', c_1) + d(g_i, c_2) + d(g_i', c_2)\\
    &\leq 2 \mathcal{OPT}_{assign}(\mathcal{T}_i)
\end{align*}
If there are odd number of $2$-\bad\ rounds then, let  $G = (g_i,h_i)$ be the last $2$-\bad\  round. It is easy to see that \RROE$(\mathcal{T}_i) - \mathcal{OPT}_{assign}(\mathcal{T}_i)$ =  $d(g_i,c_1) + d(h_i,c_2) - d(g_i,c_2) - d(h_i,c_1) \leq d(g_i,c_1) + d(h_i,c_2) \leq \beta$. 
Thus, 

\begin{align*}
    \text{\RROE}(\mathcal{T})\   & = \begin{cases}  \sum_{i=1}^{r/2}\text{\RROE}(\mathcal{T}_i)  & \text{if even no. of $2$-\bad\ rounds} \\ \sum_{i=1}^{\lfloor r/2 \rfloor}\text{\RROE}(\mathcal{T}_i) + \beta  & \text{Otherwise} \end{cases}\\ 
    & \leq 2 \sum_{i=1}^{\lfloor r/2 \rfloor} \mathcal{OPT}_{assign}(\mathcal{T}_i) +  \beta = 2 \mathcal{OPT}_{assign}(\mathcal{\mathcal{T}}) + \beta
\end{align*}
Here, $r$ is the number of $2$-bad rounds. and $\beta$=$2 \sup_{x,y \in \mathcal{X}} d(x,y)$ is the diameter of the feature space.
\end{proof}

\begin{corollary}
For $k=2$ and $\tau_{\ell} = \frac{1}{k}$  for all $\ell \in [m]$, we have \RROE($\mathcal{I}$) $\le (2(\alpha + 2)\mathcal{OPT}_{clust}(\mathcal{I}) + \beta)$-approximate where $\alpha$ is approximation factor for vanilla clustering problem for any given instance $\mathcal{I}$.
\end{corollary}
The above corollary is a direct consequence of Lemma \ref{lemmabera} and the fact that \RROE($\hat{C}, X$) $\le$ \RROE($C, X$). 
The result can easily be extended for $k$ clusters to directly obtain  $2^{k-1}$-approximate solution with respect to \tfair\ fair assignment problem. 
\begin{theorem}
\label{thm:generalk}
When $\tau_{\ell} = \frac{1}{k}$  for all ${\ell} \in [m]$, an allocation returned by \RROE\ for given centers and data points  is \tfair\ fair and satisfies  $2^{k-1}$-approximation guarantee with respect to an optimal \tfair\ fair assignment problem up to an instance-dependent  additive constant. \end{theorem}
\begin{proof}
In the previous proof we basically considered two length cycles. Two $1$-\bad\ allocations resulted in 1 cycles and one $2$-\bad\ allocations resulted in another type of cycles. When the number of clusters are greater than two, then any $ 2 \le q \le k$ length cycles can be formed. Without loss of generality, let us denote $\{c_1, c_2, \ldots, c_{q}\}$ as the centers that are involved in forming such cycles. Further denote by set $X_{i}^j$ to be the set of points that are allotted to cluster $i$ by \RROE\ but  should have been allotted to cluster $j$ in an optimal fair clustering. The $q$ length cycle can then be visualized in the Fig. \ref{fig:2^k-1proof}. Since the cycle is formed with respect to these points, we have $ \vert X_{1}^q \vert  =  \vert X_{2}^{1} \vert  = \ldots =  \vert X_{q}^{q-1} \vert $ The cost by \RROE\ algorithm is then given as:
\begin{align*}
    &\sum_{i=2}^q\sum_{x \in X_{i}^{i-1}}d(x, c_i) + \sum_{x \in X_{1}^{q}}d(x, c_1)\\
    &\le 2\left(\sum_{x \in X_{2}^1}d(x, c_1) +\sum_{x \in X_{1}^q}d(x, c_2 ) + \beta \right) + \sum_{i=3}^q\sum_{x \in X_{i}^{i-1}}d(x, c_i)\\
    &\le 2(\sum_{x \in X_{2}^1}d(x, c_1) + \beta) + 2^2 \left(\sum_{x \in X_3^2}d(x, c_2) + \sum_{x \in X_{1}^q}d(x, c_3) + \beta\right) + \sum_{i=4}^q\sum_{x \in X_{i}^{i-1}}d(x, c_i)\\
    &\le 2^{q-1}\left(\sum_{i=2}^q \sum_{x \in X_i^{i-1}} d(x, c_{i-1}) + \sum_{x\in X_1^q}d(x, c_{q})\right) + 2^q \beta
\end{align*}
Here, the first inequality follows by exchanging the points in $X_2^1$ and $X_1^q$ using Theorem \ref{thm:main}. Since the maximum length cycle possible is $k$, we straight away get the proof of $2^{k-1}$- approximation.
\end{proof}
\begin{figure}[ht!]
\centering
\includegraphics[height= 0.3 \textwidth]{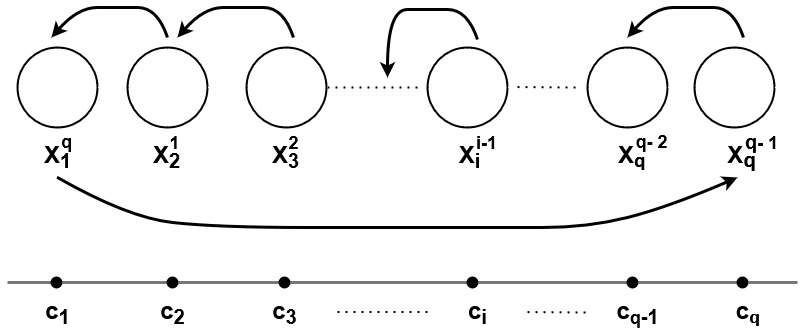}

\caption{Visual representation of set $X_{i}^{j}$ and cycle of length $q$ for Theorem \ref{thm:generalk}. The arrow represents the direction from the assigned center to the center in optimal clustering. Thus, for each set $X_i^j$ we have $c_i$ as the currently assigned center and $c_j$ as the center in optimal assignment.   }
\label{fig:2^k-1proof}
\end{figure}
\noindent Next, in contrast with Theorem \ref{thm:generalk} which guarantees a 4-approximation for $k=3$,  we show that one can achieve a 2-approximation  guarantee. The proof of this result relies on explicit case analysis and, as the number of cases to be solved increase exponentially with $k$, one needs a better proof technique for larger values of $k$. We leave this analysis as an interesting future work.

\begin{theorem}
For k=3 and $\tau_{\ell}$ = $\frac{1}{k}$ allocation returned by \RROE\  with arbitrary centers and data points  is $2$-approximate with respect to optimal \tfair\ fair assignment. 
\label{thm:k=3}
\end{theorem}

\begin{proof}
We will here find the approximation for $k=3$ using number of possible cases where one can have cycle of three length. Let the centers involved in three cycles be denoted by $c_i, c_j$, and $c_k$. Note that if there is only one cycle involving these three centers, then it will lead to only constant factor approximation. The challenge is when multiple such cycles are involved. Unlike $k=2$ proof, here we bound the cost corresponding to each cycle with respect to the cost of another cycle. The three cases shown in Fig. \ref{fig:k=3_diagrams} depicts multiple rounds when the two $3$-length cycles can be formed. In the figure, if $c_i$ is taking a point from $c_j$ it is denoted using an arrow from $c_i$ to $c_j$. It can further be shown that it is enough to consider these three cases. Further, let $\mathcal{T}_i = \{C, \{x_i, x_j, x_k, g_i, g_j, g_k\}\}$ and $\mathcal{T}_i' = \{C, \{y_i, y_j, y_k, g_i', g_j', g_k'\}\}$ denote the two instances.

\begin{figure}[hpt!]
\centering
\includegraphics[width= 0.8 \textwidth]{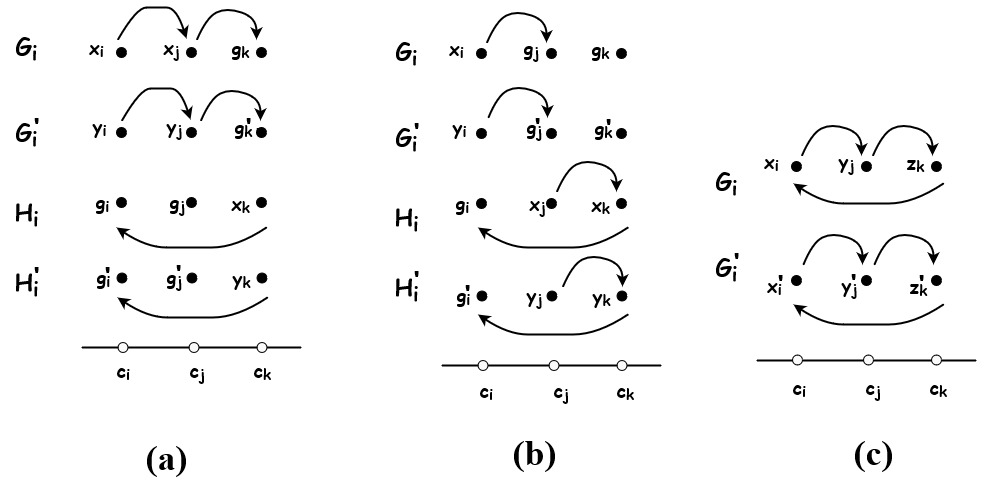}
\caption{Different use cases for 3-length cycle involving k=3 clusters  (a) Case 1: Two-three length cycle pair ($G_i,H_i$) and ($G_i',H_i'$)  (b) Case 2: Second possibility of two-three length cycle pair ($G_i,H_i$) and ($G_i',H_i'$)    (c) Case 3:Three length cycle pair ($G_i,G_i'$).}
\label{fig:k=3_diagrams}
\end{figure}
\textbf{Case 1}:
In this case we bound the rounds shown in Fig. \ref{fig:k=3_diagrams}(a). Let, one cycle completes in rounds $G_i, H_i$ and another cycle completes in rounds $G_i', H_i'$. Then,
\begin{align*}
\mathcal{OPT}_{assign}(\mathcal{T}_i) &=  d(x_i, c_j) + d(x_j, c_k) + d(g_k, c_k) + d(g_i, c_i) + d(g_j, c_j)  + d(x_k, c_i)\\
\mathcal{OPT}_{assign}(\mathcal{T}_i') &= d(y_i, c_j)
+ d(y_j, c_k) + d(g_k', c_k) + d(g_i', c_i) + d(g_j', c_j)  + d(y_k, c_i) 
\end{align*}
Further,
\begin{align*}
\text{\RROE}(\mathcal{T}_i) &= d(x_i, c_i)+ d(x_j, 
c_j)+ d(g_k, c_k)+ d(g_i, c_i)+ d(g_j, c_j)+ d(x_k, c_k)\\
&\le d(g_i', c_i) + d(g_j', c_j) + d(g_k, c_k) + d(g_i, c_i) + d(g_j, c_j) + d(x_k, c_k)
\end{align*}
Now,
\begin{align*}
d(x_k, c_k) &\le d(x_k, c_i) + d(c_i, c_k) \le d(x_k, c_i) + d(c_i, c_j) + d(c_j, c_k)\\
    &\le d(x_k, c_i) +d(x_i, c_i) + d(x_i, c_j) + d(x_j, c_j) + d(x_j, c_k)\\
    &\le d(x_k, c_i) +d(y_k, c_i) + d(x_i, c_j) + d(y_i, c_j) + d(x_j, c_k)
\end{align*}
Combining the above two, we get:
\begin{align*}
    \text{\RROE}(\mathcal{T}_i) \le \mathcal{OPT}_{assign}(\mathcal{T}_i) + \mathcal{OPT}_{assign}(\mathcal{T}_i')
\end{align*}
Thus, the cost of each cycle can be bounded by the sum of optimal cost of its own and the optimal cost of the next cycle. If we take sum over all such cycles, we will get $2$-approximation result plus a constant due to the last remaining cycle.

\textbf{Case 2}:
In this case we bound the rounds shown in Fig. \ref{fig:k=3_diagrams}(b). The optimal assignments in this case will be 
\begin{align*}
\mathcal{OPT}_{assign}(\mathcal{T}_i) &=  d(x_i, c_j) + d(g_j, c_j) + d(g_k, c_k) + d(g_i, c_i) + d(x_j, c_k)  + d(x_k, c_i)\\
\mathcal{OPT}_{assign}(\mathcal{T}_i^{'}) &= d(y_i, c_j)
+ d(g_j', c_j) + d(g_k', c_k) + d(g_i', c_i) + d(y_j, c_k)  + d(y_k, c_i) 
\end{align*}
Also, we know that 
\begin{align*}
\text{\RROE}(\mathcal{T}_i) &= d(x_i, c_i)+ d(g_j, 
c_j)+ d(g_k, c_k)+ d(g_i, c_i)+ d(x_j, c_j)+ d(x_k, c_k)\\ 
&\le d(g_i',c_i) + d(g_j, 
c_j)+ d(g_k, c_k)+d(g_i, c_i)+ d(y_k,c_i) + d(c_i,c_j) +\\ 
&d(y_j,c_k)\\
&\le d(g_i',c_i) + d(g_j, 
c_j)+ d(g_k, c_k)+d(g_i, c_i)+ d(y_k,c_i) + d(x_k,c_i) +\\ 
&d(x_i,c_j)+d(y_j,c_k)\\
\end{align*}
Combining the above two, we get:
\begin{align*}
    \text{\RROE}(\mathcal{T}_i) \le \mathcal{OPT}_{assign}(G_i,H_i) + \mathcal{OPT}_{assign}(G_i',H_i')
\end{align*}
\textbf{Case 3}:
Here again we will have two allocation rounds namely $G_i, G_i^{'}$ as shown in Fig. \ref{fig:k=3_diagrams} (c). It is easy to see that for this case, 
\begin{align*}
    \text{\RROE}(G_i) \le \mathcal{OPT}_{assign}(G_i')
\end{align*}
This completes the proof for $k=3$. 
\end{proof}

The following proposition proves that $2$-approximation guarantee is tight with respect to \RROE\ algorithm.

\begin{proposition}
\label{proposition}
There is an instance  with arbitrary centers and data points  on which \RROE\ achieves  $2$-approximation  with respect to optimal assignment. 
\label{prop:tightness}
\end{proposition}
\begin{figure}[hpt!]
\centering
\includegraphics[width= 0.8 \textwidth]{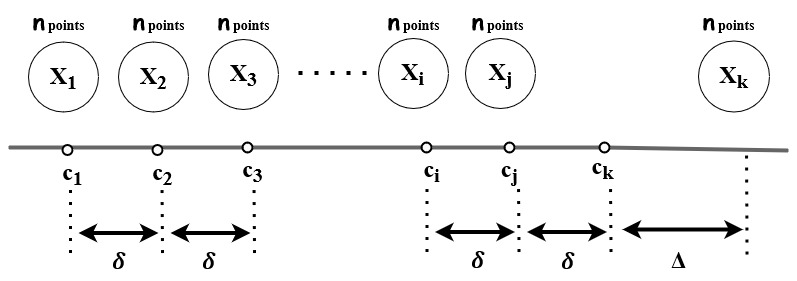}
\caption{The worst case example for fair clustering instance.}
\label{fig:worst_case_eg}
\end{figure}

\begin{proof}
The worst case for any fair clustering instance can be the situation wherein rather than choosing the points from the center's own set of optimal points, it prefers points from other centers. One such example is depicted in Fig. \ref{fig:worst_case_eg}. In this example we consider $k$ centers, and for each of these centers we have set of $n$ optimal points that are at a negligible distance (say zero) and these set are denoted by $X_i$ for center $c_i$ except the last center $c_k$. The set of optimal points for center $c_k$ is located at a distance $\Delta$ such that $\Delta$= $(k-1)\delta$ where $\delta$ is the distance between all the centers. Now we will try to approximate the tightest bound on cost that one can achieve. 
In optimal assignment each cluster center will take points from its optimal set of points. Thus optimal cost can be summed up as 
\begin{align*}
\mathcal{OPT}_{assign}  &= \sum_{x_i \in X_1} d(x_i,c_1) + \sum_{x_i \in X_2} d(x_i,c_2) + \ldots + \sum_{x_i \in X_k} d(x_i,c_k)\\
&= 0 + 0 + 0 + n\Delta
\end{align*}
If one uses round-robin based \RROE\ to solve assignment problem then at the start of $t=0^{th}$ round, each of the set $X_i$ has $n$ points. Now since $\Delta$ is quite large as compared to $\delta$ so $c_k$ will prefer to chose points from the set of previous center $c_{k-1}$. Rest all centers will take points from their respective set of optimal points as those points will be at the least cost of zero. This type of  assignment will continue until all the points in set $X_{k-1}$ gets exhausted. Thus the cost after $n/2$ rounds will be
\begin{align*}
Cost_1 &= \sum_{x_i \in X_1} d(x_i,c_1) + \ldots + \sum_{x_i \in X_{k-1}} d(x_i,c_{k-1}) + \sum_{x_i \in X_{k-1}} d(x_i,c_{k-1})\\
& = 0 + 0 + 0 + \frac{n\delta}{2}
\end{align*}
Now since all the points in set $X_{k-1}$ are exhausted, both $c_{k-1}$ and $c_k$ will prefer to choose the  points from set $X_{k-2}$. Other centers will still continue to choose the points from their respective optimal sets. It should be noted that now $\frac{n}{2}$ points are left with the center $X_{k-2}$ that are being distributed amongst $3$ clusters. Such assignments will be take place for next $\frac{n}{6}$ rounds and after that the set $X_{k-2}$ will get exhausted. The cost incurred to different centers in such assignment will be
\begin{align*}
Cost_2 &= \sum_{x_i \in X_1} d(x_i,c_1) + \ldots + \sum_{x_i \in X_{k-2}} d(x_i,c_{k-2}) + \sum_{x_i \in X_{k-2}} d(x_i,c_{k-1})\\
&+ \sum_{x_i \in X_{k-2}}d(x_i, c_{k})\\
& = \frac{n\delta}{6} + \frac{2n\delta}{6}\\
& = \frac{3n\delta}{6} = \frac{n\delta}{2}
\end{align*}

It is easy to see that the additional cost that is incurred at each phase will be $\frac{n\delta}{2}$ until the only left out points are from $X_k$. The total number of such phases will be $k-1$. Thus, exhibiting a cost of $\frac{n(k-1)\delta}{2}$. Further, at the last round all the points from $X_k$ need to be equally distributed amongst $X_1, X_2, \ldots, X_k$, thus incurring the total cost of $((k-1)\delta + \Delta + (k-2)\delta + \Delta + \ldots + \delta + \Delta + \Delta)\frac{n}{k}$. Thus, the total cost by \RROE\ is given as:

\begin{align*}
    Cost_{\text{\RROE}} &=  \frac{n(k-1)\delta}{2} + ((k-1)\delta + \Delta + (k-2)\delta + \Delta + \ldots + \delta + \Delta + \Delta)\frac{n}{k}\\
    &=\frac{n(k-1)\delta}{2} + \frac{nk(k-1)\delta}{2k} + \frac{nk\Delta}{k}\\
    &= n(k-1)\delta + n\Delta\\
    &= 2n\Delta
\end{align*}
\end{proof}
\textbf{Research gap:} Theorem \ref{thm:generalk} suggests that the approximation ratio with respect to the number of clusters $k$ can be  exponentially bad. However, our experiments show---agreeing with our finding on small values of $k (\leq 3)$---that   the performance of  \RROE\ does not degrade  with $k$.  To assert a $2$-approximation bound for general $k$ a  novel proof technique is needed and we leave this analysis as an interesting future work. Here, we  provide the following conjecture. 
\begin{conjecture}
\RROE\ is $2$-approximate with respect to optimal \tfair\ fair assignment problem for any value of $k$.
\label{conjectureLabel}
\end{conjecture}

We note that \RROE\ uses vanilla $k$-means/$k$-median algorithm followed by one round of fair assignment procedure. It remains to be shown that  given a convergence guarantee of a clustering algorithm, the output of the  returned by the \RROE\ algorithm indeed converges to approximately optimal \tfair\ allocation in finite time.   Convergence guarantees of vanilla clustering algorithms are well known in the literature (\cite{convergenceKmeans1,convergenceKmeans,convergenceKmedian}).  Since, fair assignment procedure performs corrections for all available data points only once, \RROE\ is bound to converge. This gives us the following lemma. 

\begin{restatable}{lemma}{ConvergenceFracOE}
\RROE\ algorithm converges.
\label{lemmaConverge}
\end{restatable}

\subsection{Guarantees for \RROE\ for general $\tau$ }
\label{sec:generalTauFrac}

We first begin with a simple observation that the problem of solving \tfair\ fair assignment problem on instance $\mathcal{T}$ for given centers $C$ and set of points $X$. The problem can be divided into two subproblems:
\begin{enumerate}
    \item Solving $1/k$-ratio fair assignment problem on subset of points $X_1 \in X$ such that $\lvert X_1\rvert = \sum_{\ell \in [m]} k.\tau_\ell.n_\ell$. 
    \item Solving optimal fair assignment problem on $X_2 \in X\setminus X_1$ without any fairness constraint. 
\end{enumerate}
Let us denote the first instance by $\mathcal{T}^{1/k}$ and second instance with $\mathcal{T}^0$.


\begin{figure}[hpt!]
\centering
\includegraphics[width= 0.5 \textwidth]{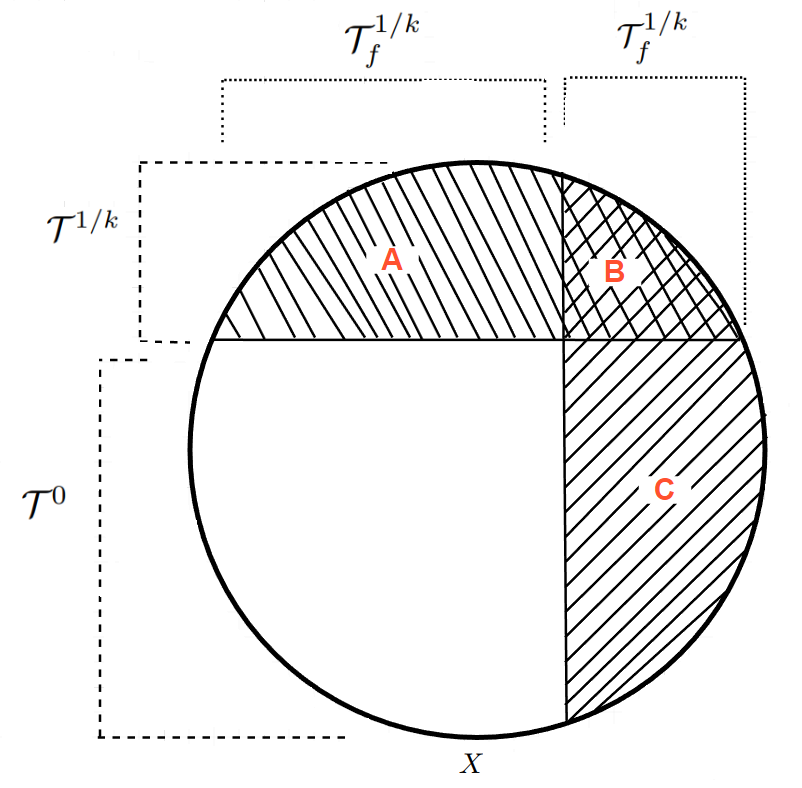}
\caption{Set of points $X$ divided into instance $\mathcal{T}^{1/k}$ and $\mathcal{T}^0$. Further the instances $\mathcal{T}_f^{1/k}$ and $\mathcal{T}_f^{0}$ are depicted in the same set of points $X$ leading to formation of regions $A,B,C$. }
\label{fig:setXgeneralTau}
\end{figure}

\begin{restatable}{lemma}{lemma2General}
\label{lemma2General}
There exists two separate instances $\mathcal{T}^{1/k}$ with $\tau$=$\{1/k\}_{\ell=1}^m$ and $\mathcal{T}^0$ with $\tau$=$\{0\}_{\ell=1}^m$ such that fair assignment problem on instance $\mathcal{T}$ can be divided into solving two problem on these two instances, i.e., $\mathcal{OPT}_{assign}(\mathcal{T}) = \mathcal{OPT}_{assign}(\mathcal{T}^{1/k}) + \mathcal{OPT}_{assign}(\mathcal{T}^0)$.
\end{restatable}
\begin{proof}
The $\mathcal{T}$ basically ensures that each cluster should have atleast $\tau_\ell.n_\ell$ number of points. Rest all the points can be allocated in the optimal manner without any fairness constraint. Therefore in optimal assignment, there exists a set $X_1^{OPT}$ such that $\vert X_1^{OPT}\vert = \sum_{\ell=1}^m\tau_\ell.n_\ell.k$ that satisfy the $\tau-$ratio fairness  with $\tau_\ell = 1/k\ \forall \ell \in [m]$.
\end{proof}

Let $X_1^f$ be the set of points that are allocated in line number 4 by Algorithm \ref{algo:fair-assignment}. Further, let $\mathcal{T}_f^{1/k}$  be an instance to $\tau$-ratio fair assignment problem with $\tau=\{1/k\}_{\ell=1}^m$  and $\mathcal{T}_f^0$ be instance when $\tau$=$\{0\}_{\ell=1}^m$ by \RROE\ (depicted in Fig. \ref{fig:setXgeneralTau}). Then, our next lemma shows that the partition returned by \RROE\ is the optimal one.

\begin{restatable}{lemma}{lemma1General}
\label{lemma1General}
$\mathcal{OPT}_{assign}( \mathcal{T}_f^{1/k}) + \mathcal{OPT}_{assign}(\mathcal{T}_f^0) \le \mathcal{OPT}_{assign}( \mathcal{T}^{1/k}) + \mathcal{OPT}_{assign}(\mathcal{T}^0)$ for any partition $\mathcal{T}^{1/k}$ and $\mathcal{T}^0$. Thus, $\mathcal{OPT}_{assign}(\mathcal{T}) = \mathcal{OPT}_{assign}(\mathcal{T}_f^{1/k}) + \mathcal{OPT}_{assign}(\mathcal{T}_f^0)$.
\end{restatable}
\begin{proof}
We divide the complete set of points $X$ into three regions $A$, $B$, and $C$ as shown in Fig. \ref{fig:setXgeneralTau}. The region $B$ contains the points in the overlap of $\mathcal{T}^{1/k}$ and $\mathcal{T}_f^{1/k}$. Since, we are talking about the optimal assignment problem, these points will be assigned to same centers and hence we can ignore these points. Let the points allocated to any center $c_j$ in $\mathcal{T}_f^{1/k}$ by \RROE\ be $P=\{ x_1, x_2, x_3,\ldots, x_{m_j}\}$ and points allocated to $c_j$ in partition $\mathcal{T}^{1/k}$ be  $Q=\{ y_1, y_2, y_3,\ldots, y_{m_j}\}$. Let $g$ be a mapping function from $P \to Q$. It maps any point $x_j$ assigned to center $i$ to some point $y_j$ assigned to same center when partition under consideration is $\mathcal{T}^{1/k}$. Then, we have $\mathcal{OPT}_{assign}(\mathcal{T}_f^{1/k}) \le \text{\RROE}(\mathcal{T}_f^{1/k}) = \sum_{j=1}^k\sum_{i=1}^{m_j}d(x_i, c_j) \le \sum_{j=1}^k\sum_{i=1}^{m_j}d(y_i, c_j) = \mathcal{OPT}_{assign}(\mathcal{T}^{1/k})$. This is because despite point $y_i$ being available to center $c_j$, it chose the point $x_i$. 
 Since other points have no such constraint, we have,  $\mathcal{OPT}_{assign}(\mathcal{T}_f^0) \le \mathcal{OPT}_{assign}(\mathcal{T}^0)$. 

\end{proof}

\begin{theorem}
For $k$=$2,3$ and any general $\tau$ vector, an allocation returned by \RROE\  guarantees \tfair\ fairness and satisfies  $(2(\alpha + 2)\mathcal{OPT}_{clust} )$-approximate guarantee with respect to an fair clustering problem  where $\alpha$ is approximation factor for vanilla clustering problem.  
\end{theorem}

\begin{proof}
With the help of Lemma \ref{lemma2General} the cost of \RROE\ on instance $\mathcal{T}_f$ can be computed as,
\begin{equation}
\label{eqnCosty}
\text{\RROE}(\mathcal{T})=\text{\RROE}(\mathcal{T}_f^{1/k})+\text{\RROE}(\mathcal{T}_f^0)
\end{equation}
Now, from Section \ref{SectionFRACoeGuran},  $\text{\RROE}(\mathcal{T}_f^{1/k}) \le 2.\mathcal{OPT}_{assign}(\mathcal{T}_f^{1/k})$. 

Also, since $\mathcal{T}_f^0$ is solved for $\tau$=$\{0\}_{\ell=1}^m$ i.e. assignment is carried solely on the basis of $k-$means clustering, so we have $\text{\RROE}(\mathcal{T}_f^0) = \mathcal{OPT}_{assign}(\mathcal{T}_f^0) \le 2.\mathcal{OPT}_{assign}(\mathcal{T}_f^0)$.

So Equation \ref{eqnCosty} becomes, 
\begin{align*}
    \text{\RROE}(\mathcal{T})&= 2.\mathcal{OPT}_{assign}(\mathcal{T}_f^{1/k}) + 2.\mathcal{OPT}_{assign}(\mathcal{T}_f^0)\\
    &= 2. \mathcal{OPT}_{assign}(\mathcal{T})   \tag{using Lemma \ref{lemma2General}} \\
    &= 2. (\alpha + 2)\mathcal{OPT}_{clust}(\mathcal{I})   \tag{Using Lemma \ref{lemmabera}}
\end{align*} 
\end{proof}

\section{Fair Round Robin Algorithm for Clustering (\RR) --A Heuristic Approach}
\label{sec:heur}
We now propose another algorithm, a general version of \RROE\ where the fairness constraints are satisfied at each allocation round: Fair Round-Robin Algorithm for Clustering \RR\ (described in  Algorithm \ref{algo:FRAC}).  \RR\ runs a fair assignment problem at each iteration of a vanilla clustering algorithm. 

\begin{algorithm}[h]
\SetAlgoLined
\KwInput{Set of datapoints $X$, Number of clusters $k$, Fairness requirement vector $\tau$, Range of protected attribute $m$, clustering objective norm $p$}
 \KwOutput{Cluster centers $C$ and assignment function $\phi$}
 Choose the random centers as $C$\\ 
  \While{$Until Convergence$ }{
  \For{each $x_i \in X$}{
  $\phi(x_i) = \argmin_m d(x_i, c_m)$
  }
  $(C, \phi)$ = \textsc{FairAssignment}($C, X, k,  \tau, m, p, \phi$)
  }
 \caption{$\tau$-\RR\ }
 \label{algo:FRAC}
\end{algorithm}

 It is theoretically hard to analyze \RR\  as it is an in-processing algorithm and  each round's allocation depends upon previous rounds, i.e., the rounds are not independent of each other. 
Thus, we experimentally show the convergence of both \RR\ and \RROE\ on real-world datasets. We also show that \RR\ achieves the best objective cost amongst all the available algorithms in the literature.
Since both \RROE\  and \RR\ solve the fair assignment problem on the top of the vanilla clustering problem. Thus, one can use them to find fair clustering for center-based approaches, i.e., $k$-means and $k$-median.

\section{Experimental Result and Discussion}
\label{sec:Experimental}
We validate the performance of proposed algorithms across many benchmark  datasets and compare it against the SOTA approaches. We observe in  Section \ref{FRACvsFRACOE_ablation} that the performance of \RR\ is better than \RROE\ in terms of objective cost. It is also evident that \RR\ applies the fairness constraints after each round. 

The bench marking datasets used in the study are
\begin{itemize}
  \item \textbf{Adult\footnote{https://archive.ics.uci.edu/ml/datasets/Adult} (Census)}- The data set contains information of 32562 individuals from the 1994 census, of which 21790 are males and 10771 are females. We choose five attributes as feature set: age, fnlwgt, education\_num, capital\_gain, hours\_per\_week; the binary-valued protected attribute is sex, which is consistent with prior literature. The \balance\ in the dataset is 0.49.
  \item \textbf{Bank\footnote{https://archive.ics.uci.edu/ml/datasets/Bank+Marketing}}- The dataset consists of marketing campaign data of portuguese bank. It has data of 41108 individuals, of which 24928 are married, 11568 are single, and 4612 are divorced. We choose six attributes as the feature set: age, duration, campaign, cons.price.idx, euribor3m, nr.employed;  the ternary-valued feature martial status is chosen as the protected attribute to be consistent with prior literature, resulting in a \balance\ of 0.18.
  \item \textbf{Diabetes\footnote{https://archive.ics.uci.edu/ml/datasets/Diabetes+130-US+hospitals+for+years+1999-2008}}- The dataset contains clinical records of 130 US hospitals over ten years. There are 54708 and 47055 hospital records of males and females, respectively. Consistent with the prior literature, only two features: age, time\_in\_hospital are used for the study. Gender is treated as the binary-valued protected attribute yielding a \balance\ of 0.86.
  \item \textbf{Census II\footnote{https://archive.ics.uci.edu/ml/datasets/US+Census+Data+\%281990\%29}}- It is the largest dataset used in this study containing 2458285 records from of US 1990 census, out of which 1191601 are males, and 1266684 are females. We choose 24 attributes commonly used in prior literature for this study. Sex is the binary-valued protected attribute. The \balance\ in the dataset is 0.94.
\end{itemize}

\begin{table}[h]
\scalebox{0.72}{
\begin{tabular}{@{}llcllllll@{}}
\toprule
\multicolumn{1}{l}{\textbf{\begin{tabular}[l]{@{}l@{}}Dataset \\ Name\end{tabular}}} & \textbf{\#Cardinality} & \textbf{\begin{tabular}[l]{@{}l@{}}\#Feature \\ Attributes\end{tabular}} & \textbf{\begin{tabular}[c]{@{}l@{}}Protected\\ Attribute\end{tabular}} & \textbf{\begin{tabular}[c]{@{}l@{}}Protected\\ Attribute\\ Cardinality\end{tabular}} & \multicolumn{3}{c}{\textbf{\begin{tabular}[c]{@{}c@{}}Protected Attribute\\ Composition\end{tabular}}} & \multicolumn{1}{c}{\textbf{\begin{tabular}[c]{@{}c@{}}Dataset\\ Balance\end{tabular}}} \\ \midrule
\begin{tabular}[c]{@{}l@{}}Adult \\ (Census)\end{tabular} & $32562$ & $5$ & gender & binary & \begin{tabular}[c]{@{}l@{}}$21790$ \\ males\end{tabular} & \begin{tabular}[c]{@{}l@{}}$10771$ \\ females\end{tabular} & -- & $0.49$ \\ \midrule
Bank & $41108$ & $6$ & \begin{tabular}[c]{@{}l@{}}marital \\ status\end{tabular} & ternary & \begin{tabular}[c]{@{}l@{}}$24928$ \\ married\end{tabular} & \begin{tabular}[c]{@{}l@{}}$11568$ \\ unmarried\end{tabular} & \begin{tabular}[c]{@{}l@{}}$4612$ \\ divorced\end{tabular} & $0.18$ \\ \midrule
Diabetes & $101763$ & $2$ & gender & binary & \begin{tabular}[c]{@{}l@{}}$54708$ \\ males\end{tabular} & \begin{tabular}[c]{@{}l@{}}$47055$ \\ females\end{tabular} & -- & $0.86$ \\ \midrule
Census II & $2458285$ & $24$ & gender & binary & \begin{tabular}[c]{@{}l@{}}$1191601$ \\ males\end{tabular} & \begin{tabular}[c]{@{}l@{}}$1266684$ \\ females\end{tabular} & -- & $0.94$ \\ \bottomrule
\end{tabular}
}
\caption{Characteristics for real-world datasets commonly used in evaluation of fair clustering algorithms. Number of feature attributes exclude protected attribute and for complete list of feature attributes see Section \ref{sec:Experimental}. }
\label{tab:my-table-datasets}
\end{table}
The dataset characteristics are summarized in Table \ref{tab:my-table-datasets}.
We compare the application of FRAC to $k$-means and $k$-median against the following baseline and SOTA approaches
\begin{itemize}
\item \textbf{Vanilla $k$-means}: A Euclidean distance-based $k$-means algorithm that does not incorporate fairness constraints
\item \textbf{Vanilla $k$-median}: A Euclidean distance-based $k$-median algorithm that does not incorporate fairness constraints.
\item \textbf{\cite{bera2019fair}}: 
The approach solves the fair clustering problem through an LP formulation. The fairness is added as an additional constraint in LP by bounding the minimum (minority protection see Definition \ref{tauMP} ) and maximum (restricted dominance see Definition \ref{tauRD}) fraction of points belonging to the particular protected group in each cluster. 
Due to the high computational complexity of the $k$-median version of the approach, we restrict the comparison to the $k$-means version. Furthermore, the algorithm fails to converge in a reasonable time when the number of clusters is greater than 10 for larger datasets.
\item \textbf{\cite{ziko2019variational}}: This approach formulates a regularized optimization function incorporating clustering objective and fairness error. It does not allow the user to give an arbitrary fairness guarantee but computes the optimal trade-off by tuning a hyper-parameter $\lambda$. We compare against both the $k$-means and $k$-median version of the algorithm. We observed that the hyper-parameter $\lambda$ is extremely sensitive to the datasets and the number of clusters. Tuning this hyper-parameter is computationally expensive. We were able to tune value of $\lambda$ in a reasonable amount of time only for adult and bank datasets for $k$-means clustering for varying number of clusters. Due to the added complexity of $k$-medians, we were able to fine tune $\lambda$ only for the adult dataset. For the other cases, we have used the hyper-parameter value reported by \citeauthor{ziko2019variational} We have used the same value across varying number of cluster centers. The paper does not report any results for diabetes dataset; we have chosen the best $\lambda$ value over a single run of fine-tuning. This value is used across all experiments related to diabetes dataset.
\item \textbf{\cite{backurs2019scalable}}: This approach computes the fair clusters using fairlets in an efficient manner and is the extension to that of \cite{chierichetti2018fair}. This approach could only be integrated with $k$-median clustering. Further, we could not compare against $k$-median version of \cite{chierichetti2018fair} due to high computational  ($O(n^2)$) and space complexities. We offset this comparison using \cite{backurs2019scalable} that has shown to result in better performance than \cite{chierichetti2018fair}. 
\end{itemize}

\noindent We use the following popular metrics in the literature for measuring the performance of the different approaches.
\begin{itemize}
  \item \textbf{Objective Cost}: We use the squared euclidean distance ($p=2$) as the objective cost to estimate the cluster's compactness (see Definition \ref{eqn:objCostPrelim}).

  \item \textbf{Balance}: The \balance\ is calculated using Definition \ref{eqnBalanceRelated}  

\item \textbf{Fairness Error} This notion of fairness constraints is introduced by \cite{ziko2019variational}. It is 
 the Kullback-Leibler (KL) divergence between the required protected group proportion $\tau$ and achieved proportion within the clusters:

\begin{equation}
\begin{split}
FE(\mathcal{C}) = \sum_{C \in \mathcal{C}} \sum_{\ell \in [m]} \left( - \tau_{\ell} \log \left( \frac{q_\ell}{\tau_{\ell}} \right) \right) where\ q_\ell = \left( \frac{\sum_{x_i \in C}\mathbb{I}{(\rho_i=\ell)}}{\sum_{x_i \in X}\mathbb{I}{(\rho_i=\ell)}} \right) 
\end{split}
\label{eqnNewFairnessEror}
\end{equation} 

\end{itemize}

The $\tau$ vector in fairness error captures the target proportion in each cluster for different protected groups $\ell \in [m]$. It can be any arbitrary $\ell$ dimensional vector. In the experimental setting with $\tau=1/k$, target reduces to dataset proportion for different groups to evaluate all baselines. In a generalized setting, when $\tau < 1/k$, it is the same as the input vector $\tau$ for \RR\ and \RROE\ algorithms that achieve $\tau$-ratio fairness constraints. Similarly, in \cite{bera2019fair}, the target vector is $\delta$ (refer Section \ref{sec-tauRatio} for details on the parameter $\delta$). We report the average and standard deviation of the performance measures across 10 independent trials for every approach. The code for all the experiments is publicly available\footnote{https://github.com/shivi98g/Fair-k-means-Clustering-via-Algorithmic-Fairness}.
We begin the empirical analysis of various approaches under both $k$-means and $k$-median settings for a fixed value of $k$ (=10) in line with the previous literature. The top and bottom row in Fig. \ref{fig:FixedK} summarize the results obtained for the $k$-means and $k$-median settings respectively. The plots for $k$-means clustering clearly reveal the ability of \RR\ and \RROE\ to maintain the perfect \balance\ and zero fairness error. While  \cite{bera2019fair} is also able to achieve similar fairness performance, \RR, \RROE\ has significantly lower objective cost. Though \cite{ziko2019variational} returns tighter clusters ie., the objective cost is lower than \RR, \RROE\ and \cite{bera2019fair}, the lower objective comes at the cost of poor performance on both the fairness measures. It is also observed that the cost of fairness is relatively high in the Census-II dataset, which has the largest number of points and features among all datasets. It may be due to the shifting of an increased number of points compared to vanilla clustering for satisfying the hard constraint.

\begin{figure}[ht!]
\centering
\begin{tabular}{@{}c@{}c@{}c@{}}
\includegraphics[width=0.35\textwidth]{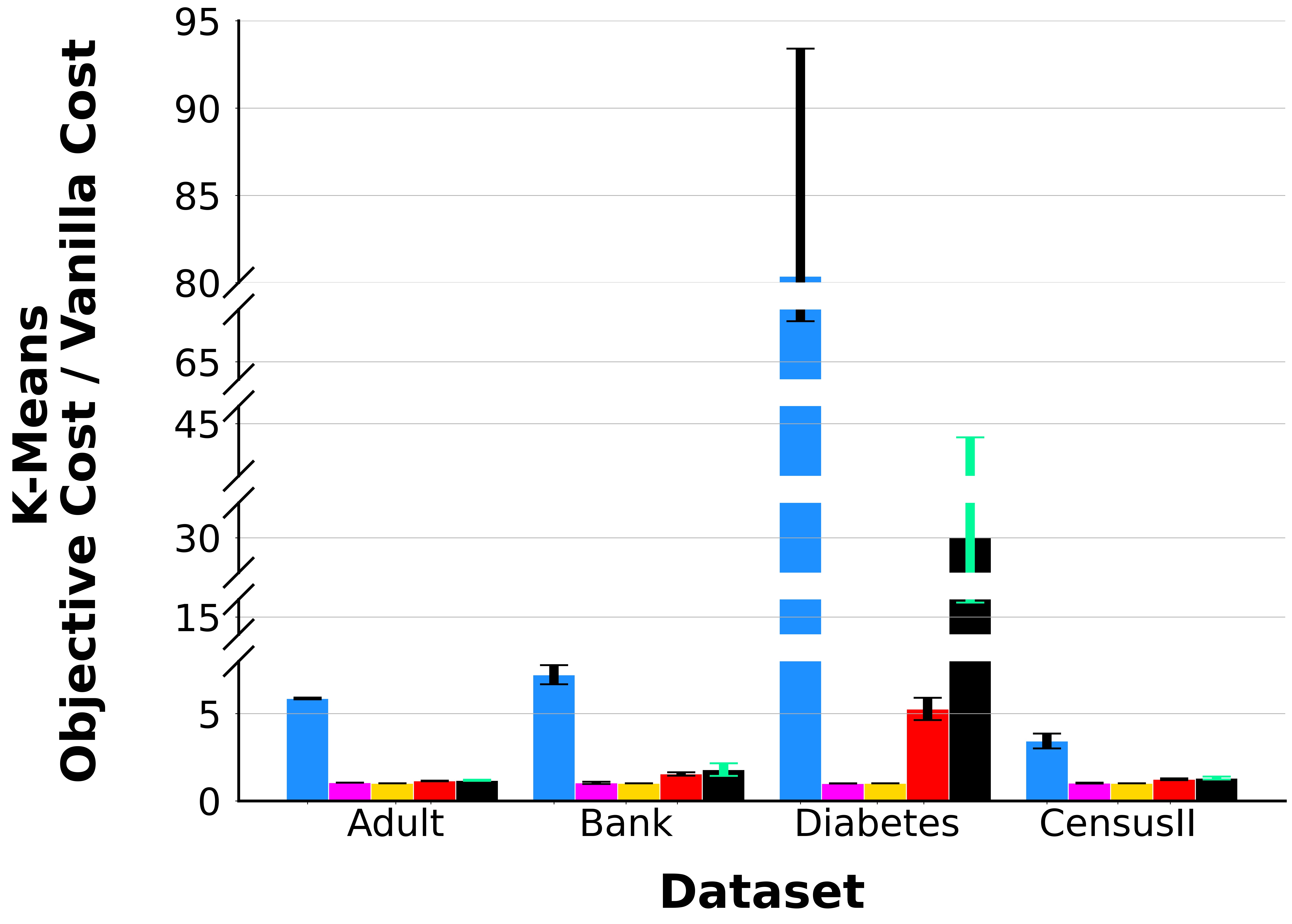}&\includegraphics[width=0.330\textwidth]{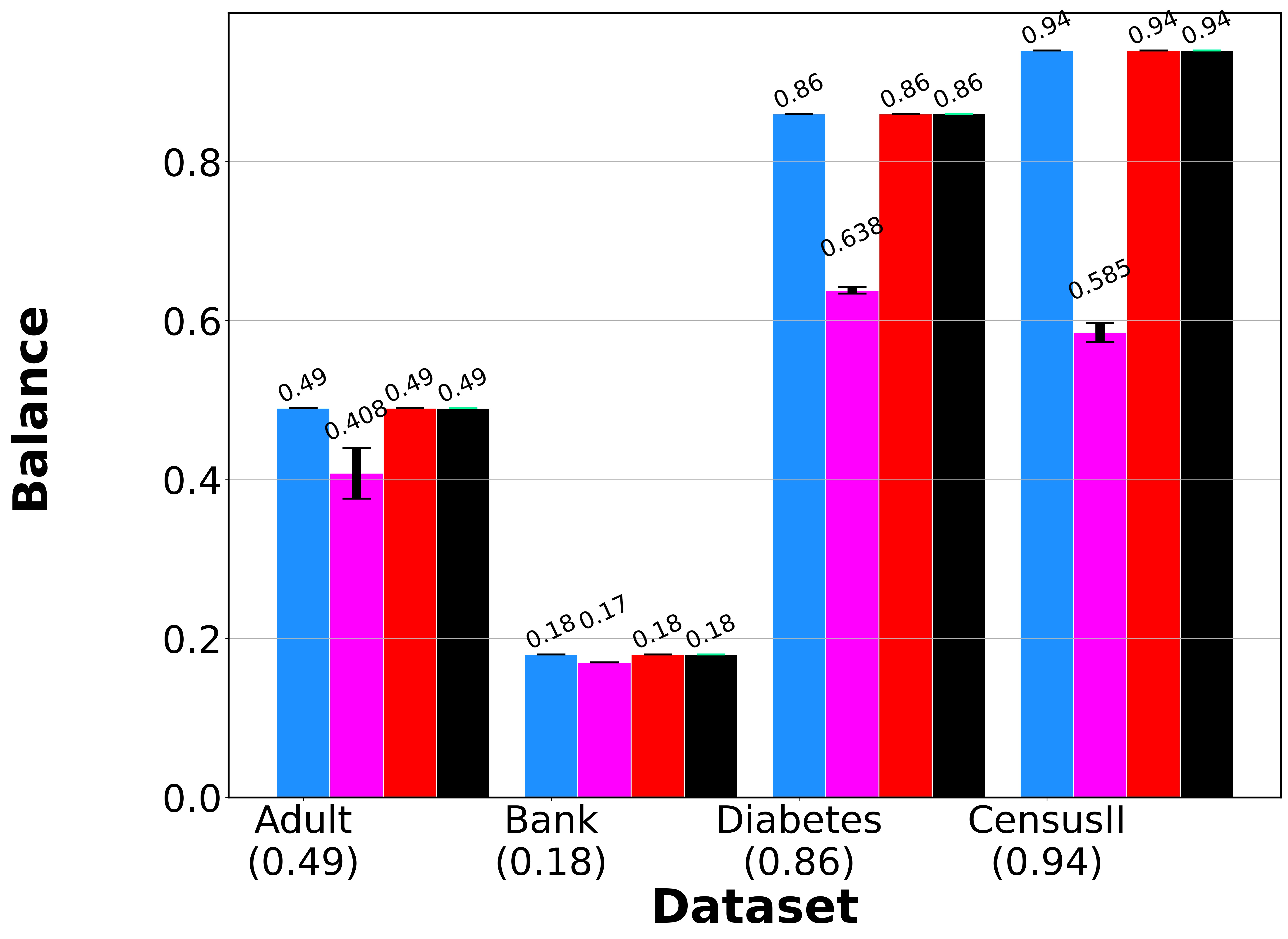}&\includegraphics[width=0.330\textwidth]{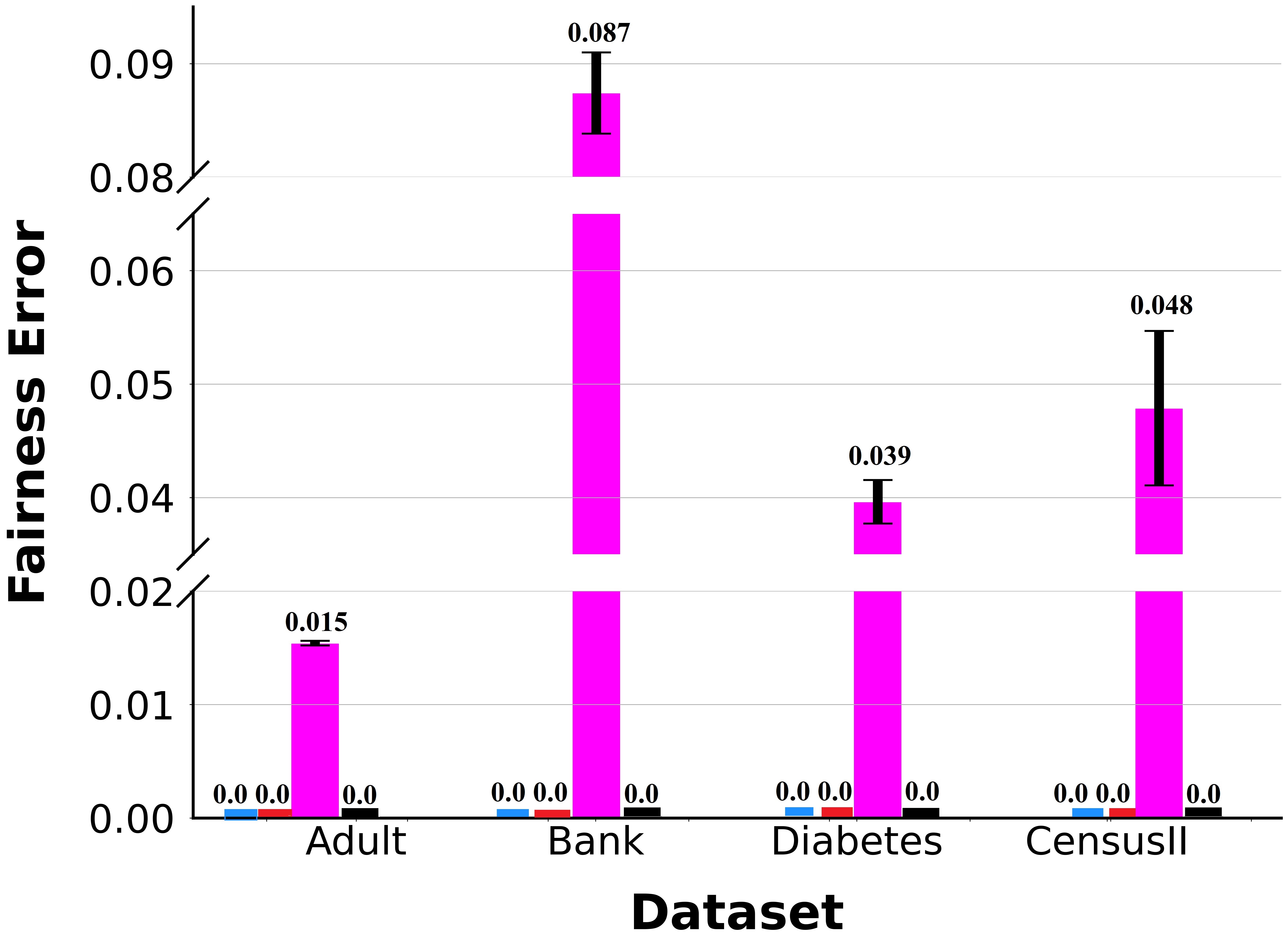}\\ 
\includegraphics[width=0.35\textwidth]{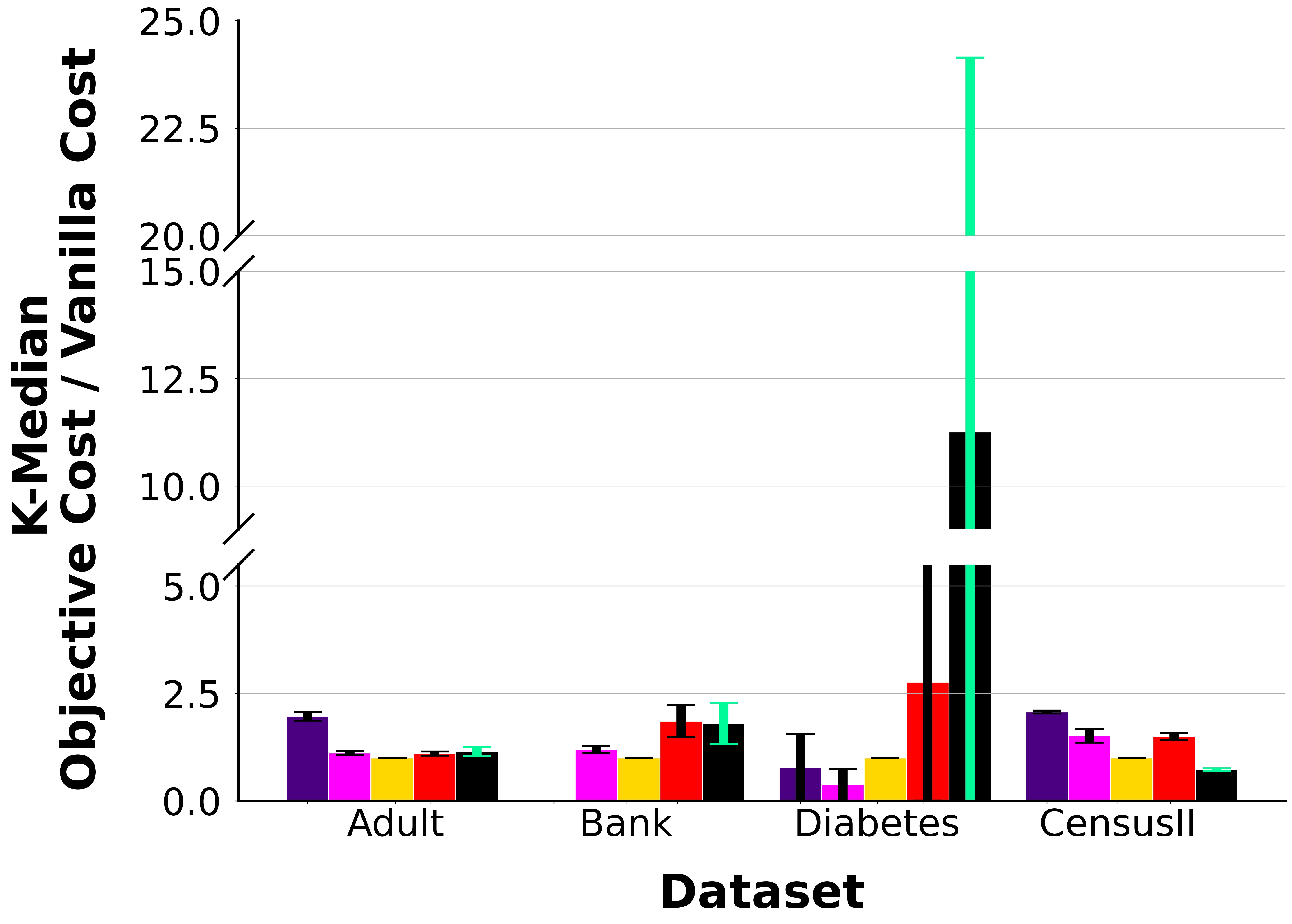}&\includegraphics[width=0.33\textwidth]{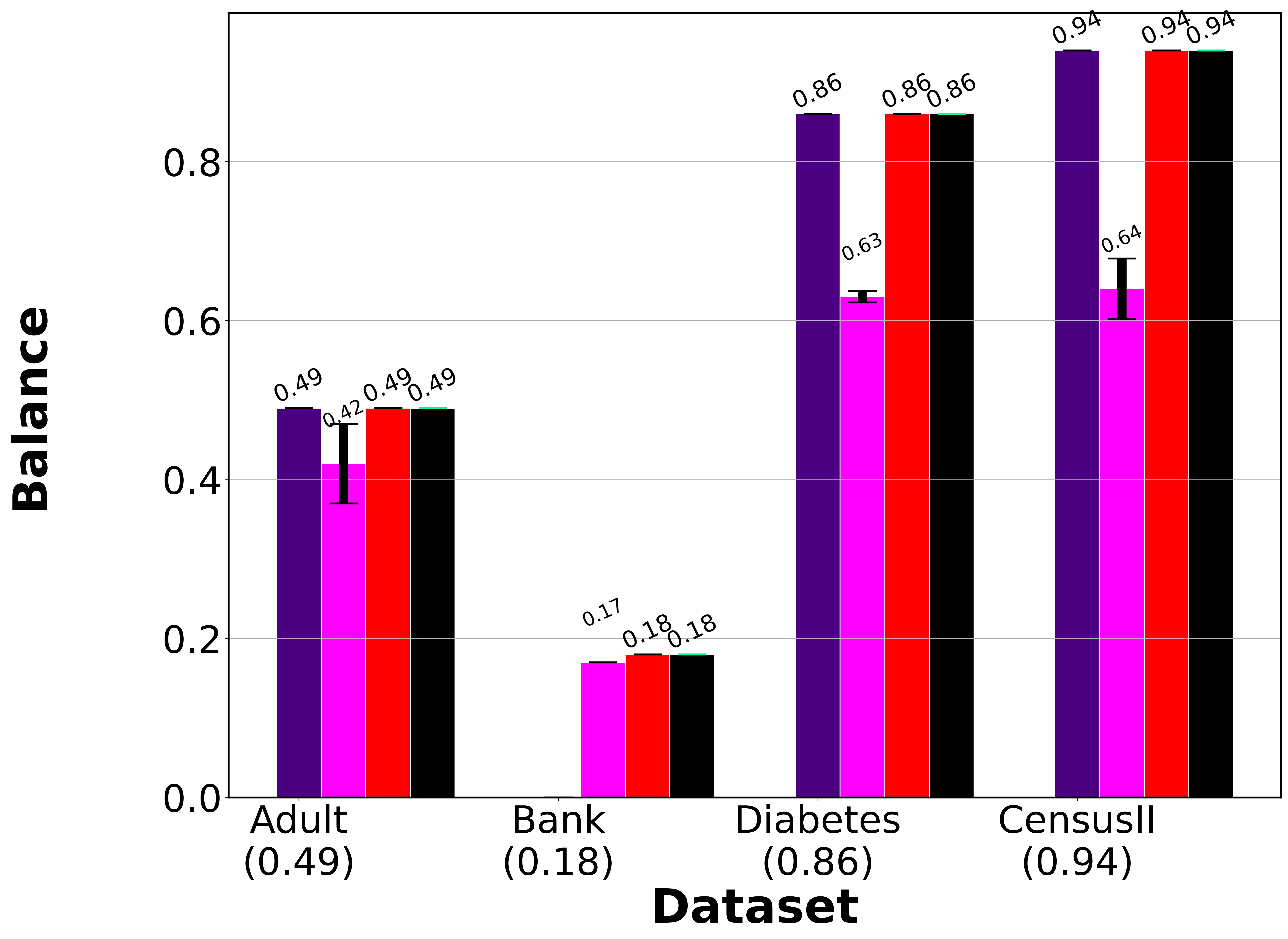}&\includegraphics[width=0.33\textwidth]{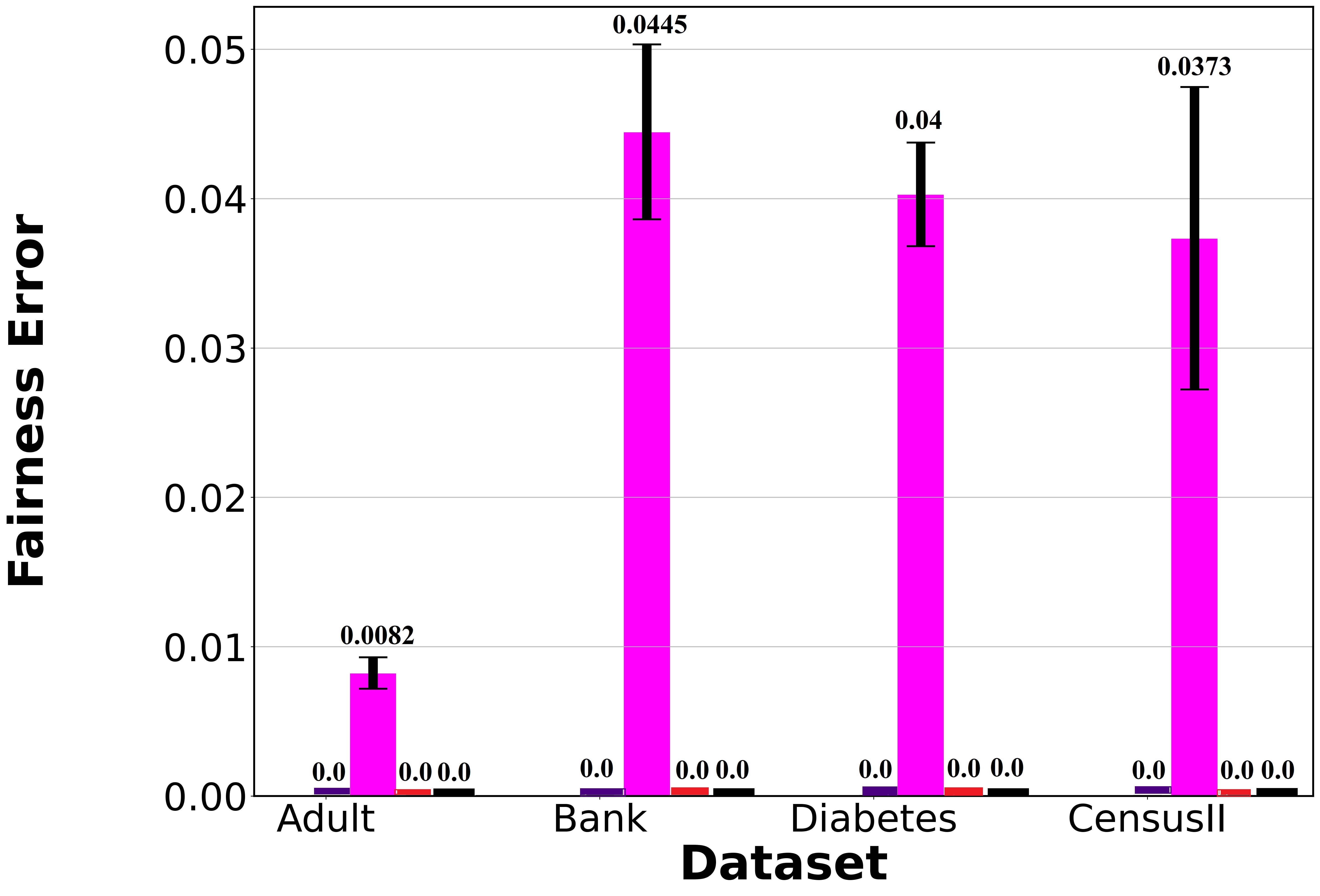}\\
\multicolumn{3}{l}{\includegraphics[width=\textwidth]{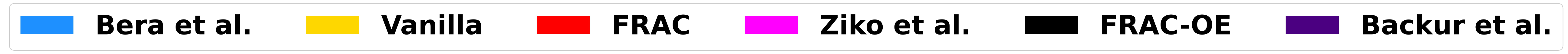}}
\end{tabular}
\caption{The plot in the first row shows the variation in evaluation metrics for $k$=10 clusters.The objective cost is scaled against vanilla objective cost. For Ziko et al. the $\lambda$ values for $k$-means and $k$-median  are taken to be  same as in their paper. The second row comprises of plots for $k$-median setting on same $k$ value. It should be noted that Backur et al. does not work for bank dataset which has ternary valued protected group. The target \balance\  of each dataset is evident from the axes of the plot. (Best viewed in color)}
\label{fig:FixedK}
\end{figure}


 \begin{figure}[th!]
\centering
\begin{tabular}{@{}c@{}c@{}c@{}}
\includegraphics[width=0.340\textwidth]{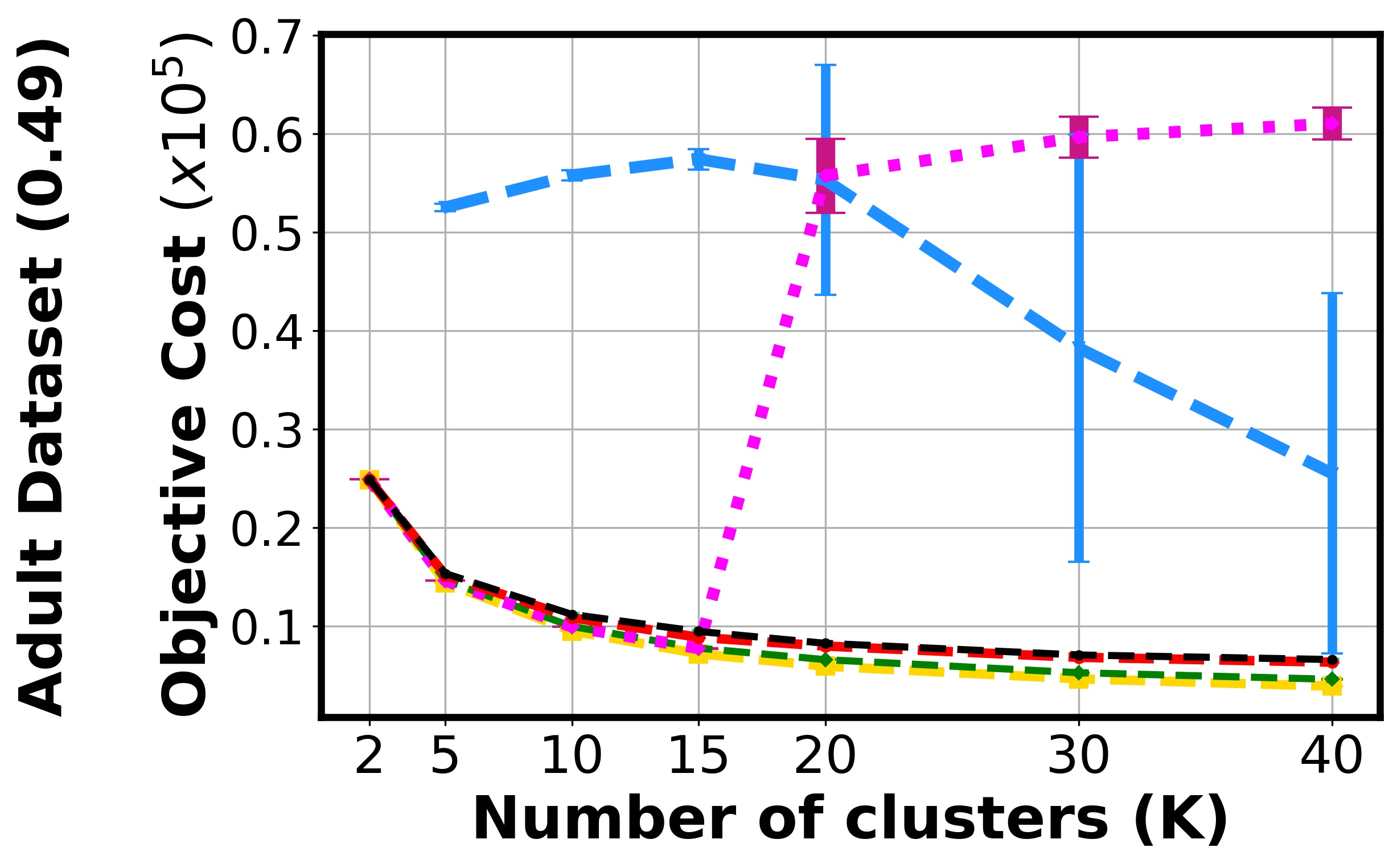}&\includegraphics[width=0.330\textwidth]{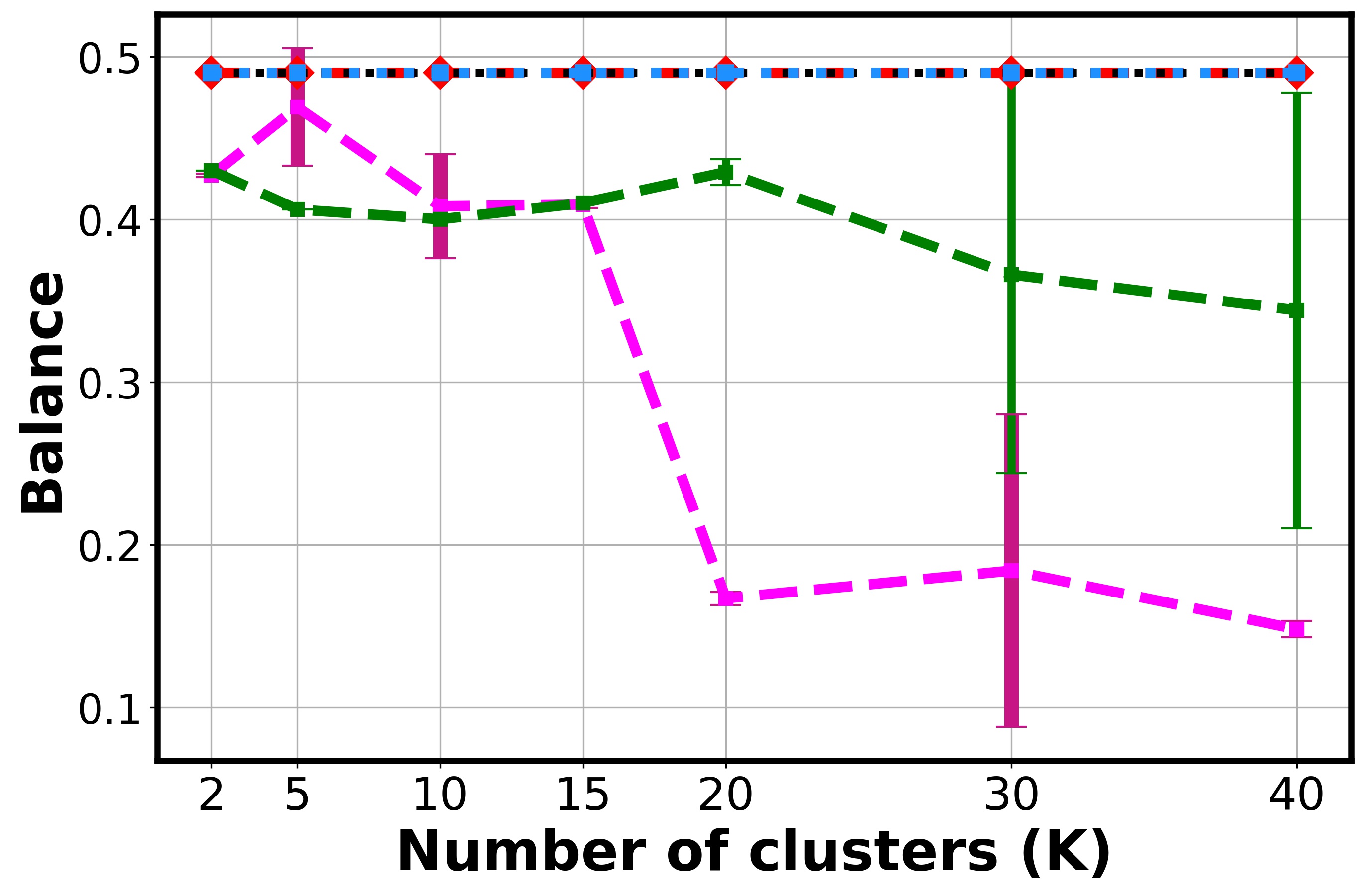}&\includegraphics[width=0.330\textwidth]{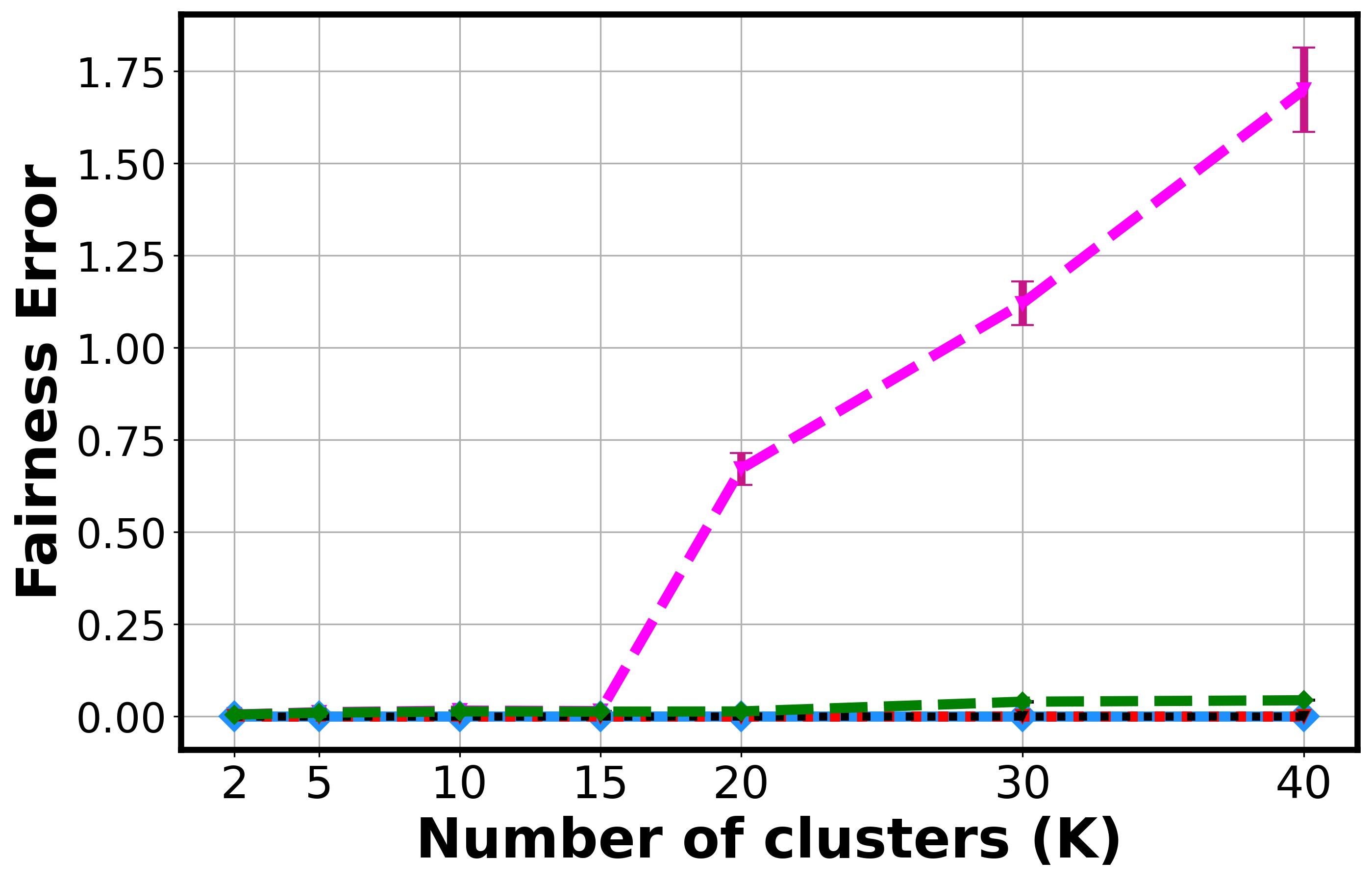}\\ 
\includegraphics[width=0.340\textwidth]{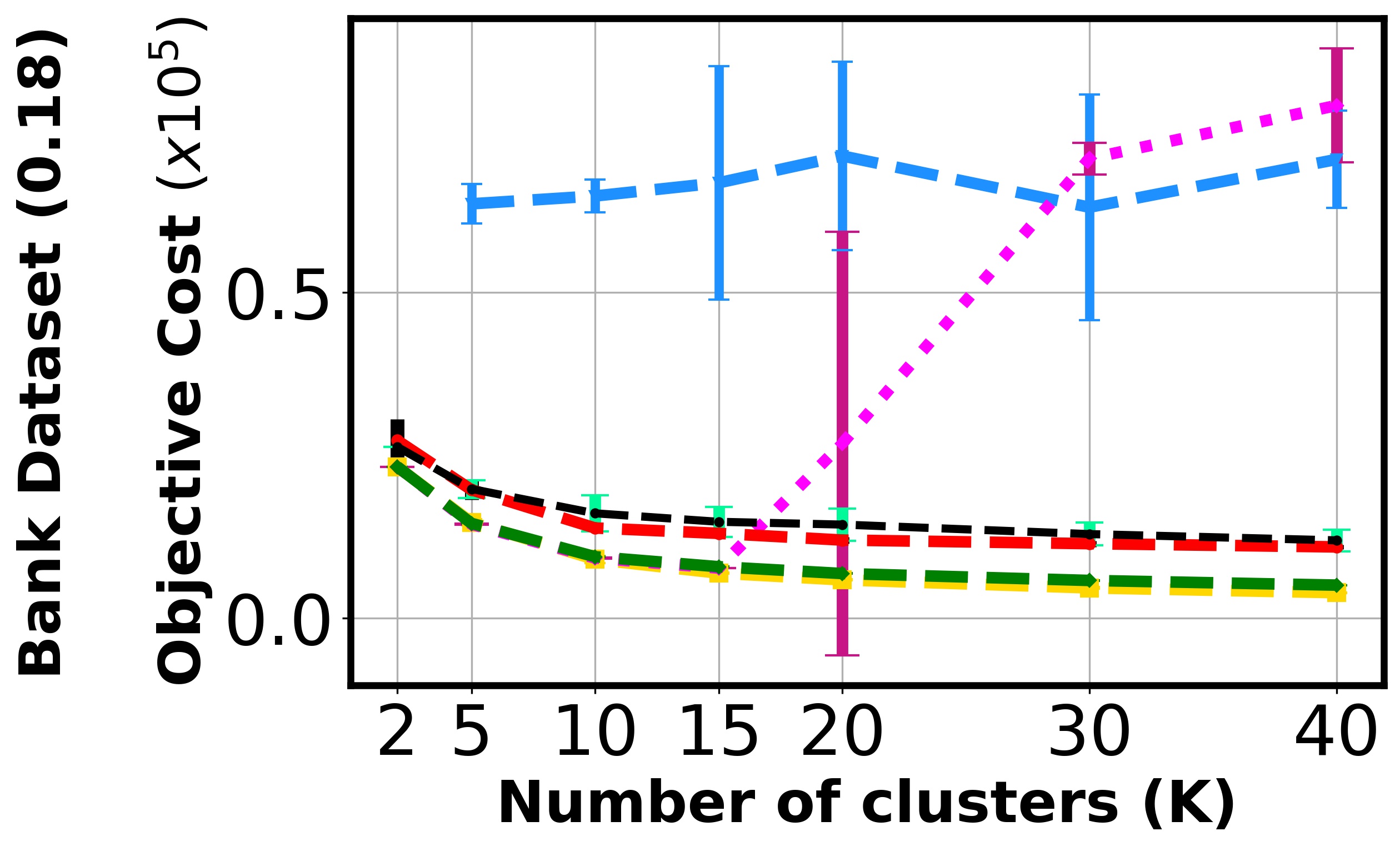}&\includegraphics[width=0.33\textwidth]{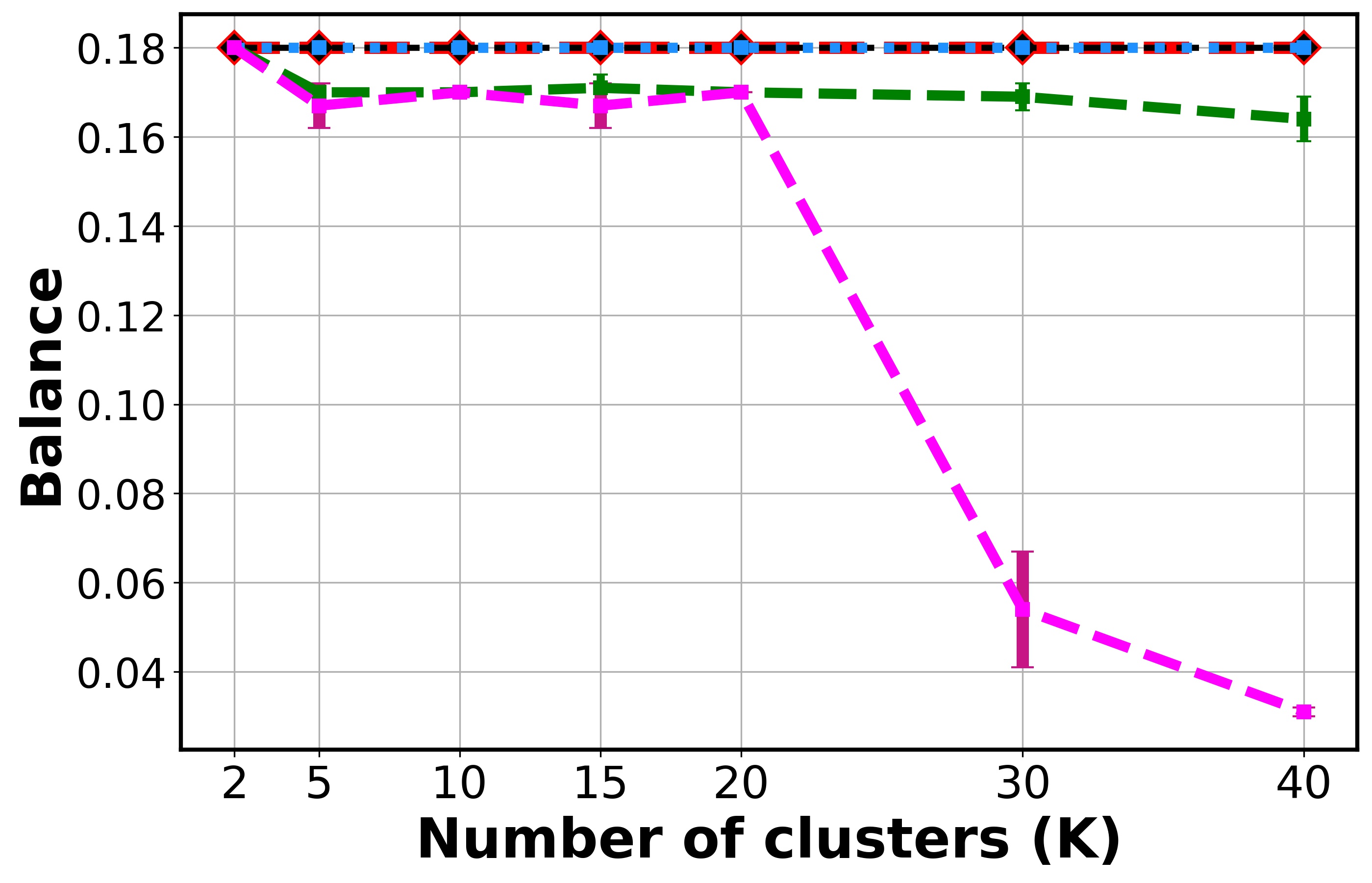}&\includegraphics[width=0.33\textwidth]{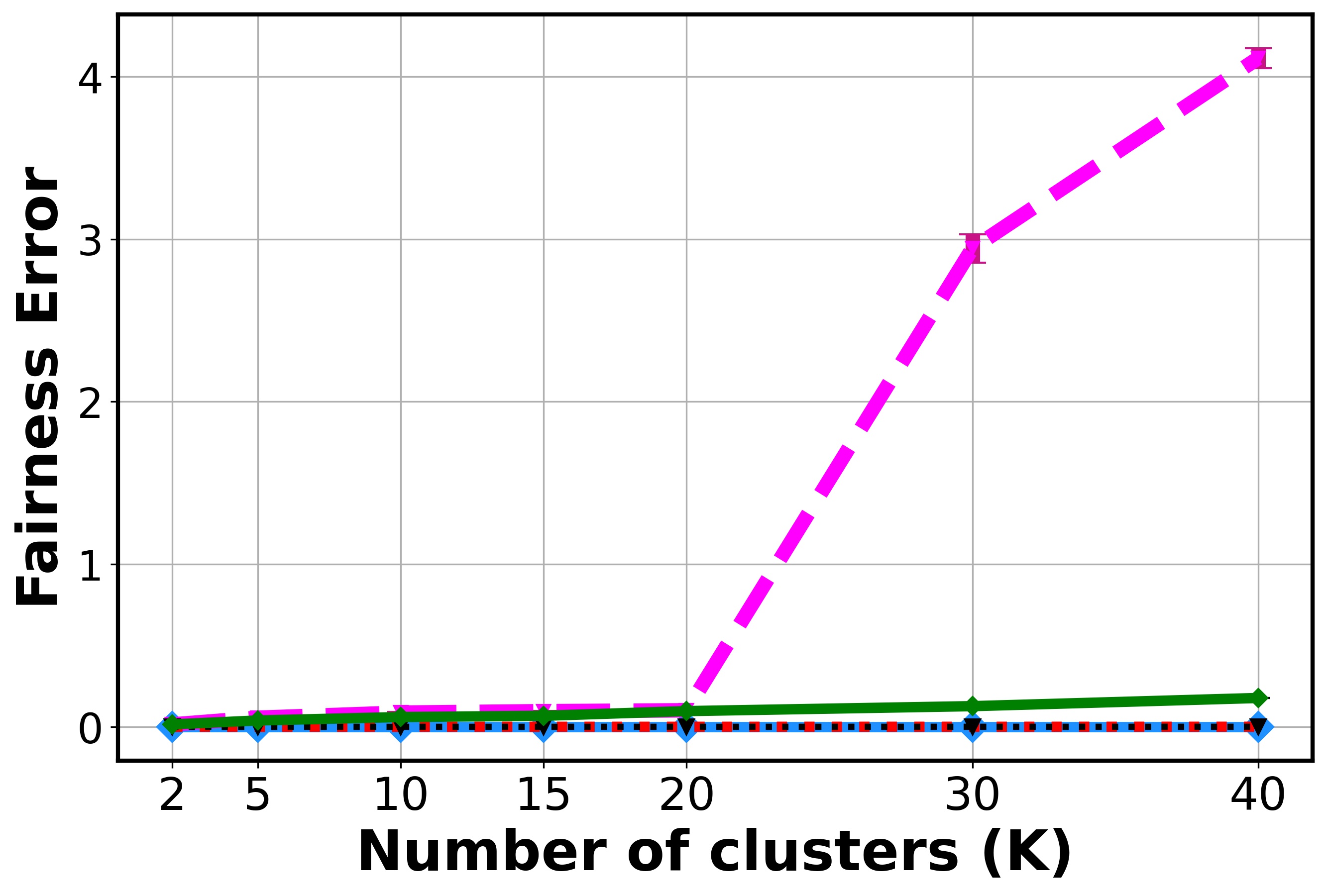}\\ 
\includegraphics[width=0.340\textwidth]{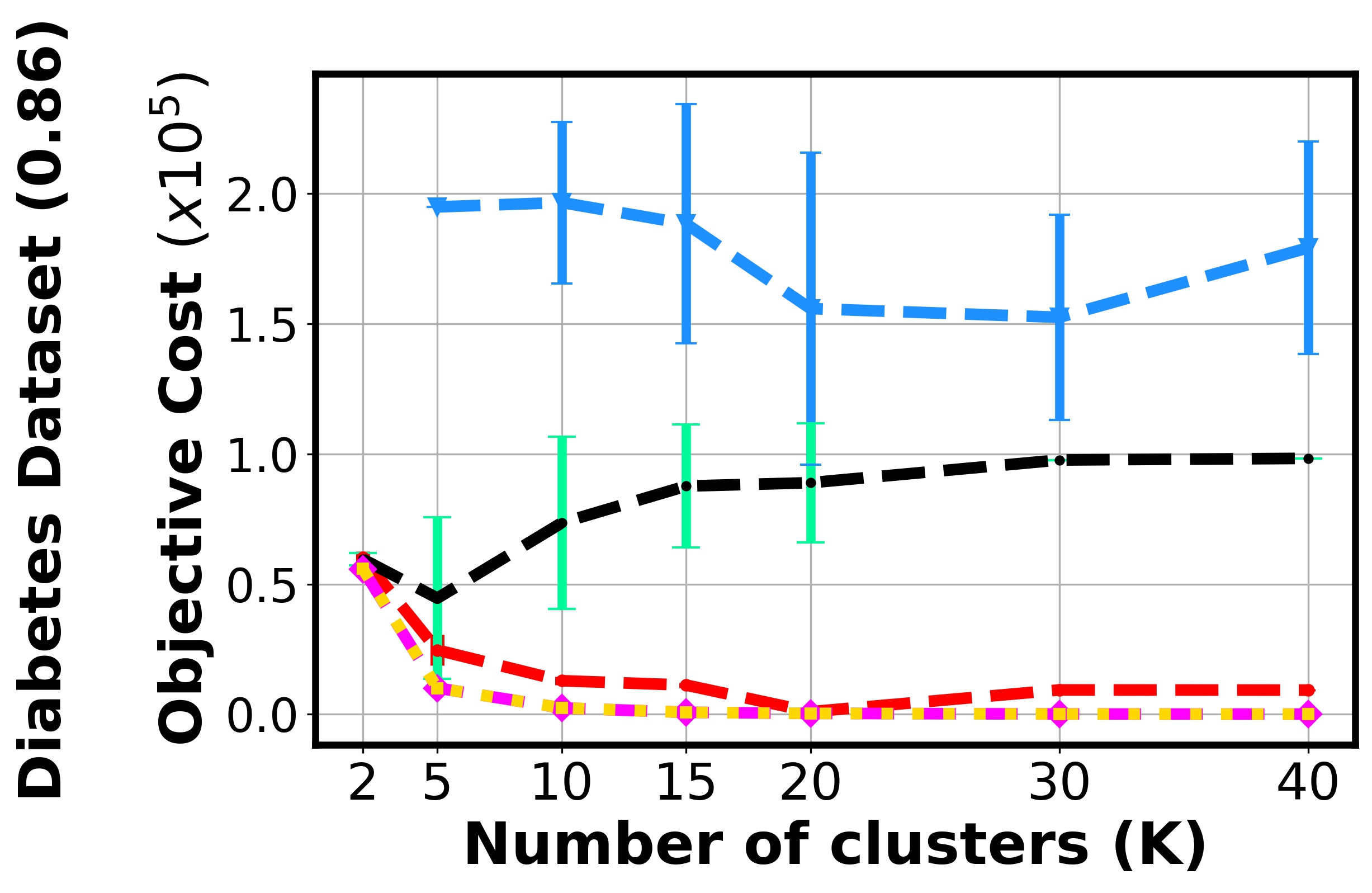}&\includegraphics[width=0.33\textwidth]{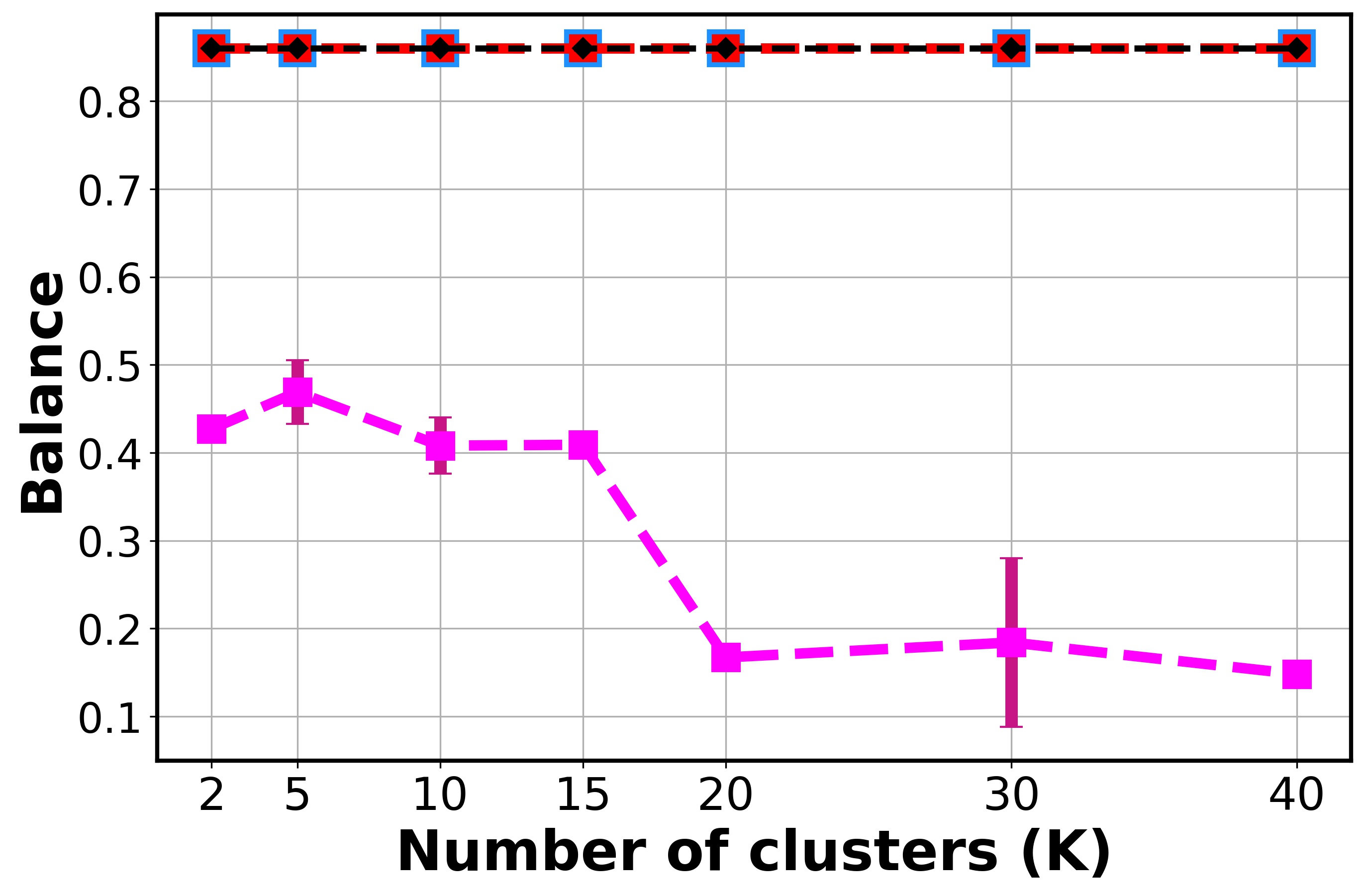}&\includegraphics[width=0.33\textwidth]{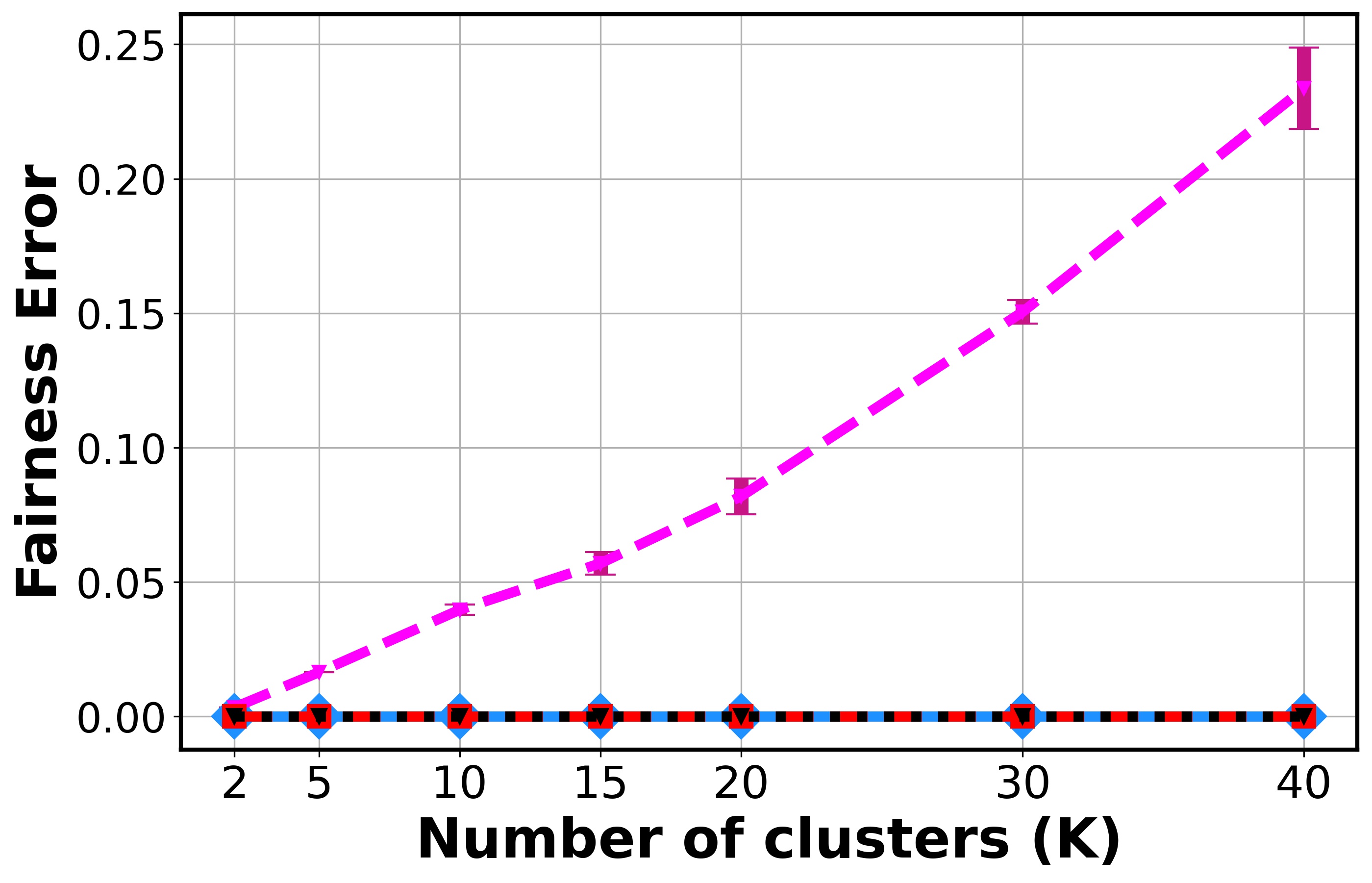}\\
\includegraphics[width=0.340\textwidth]{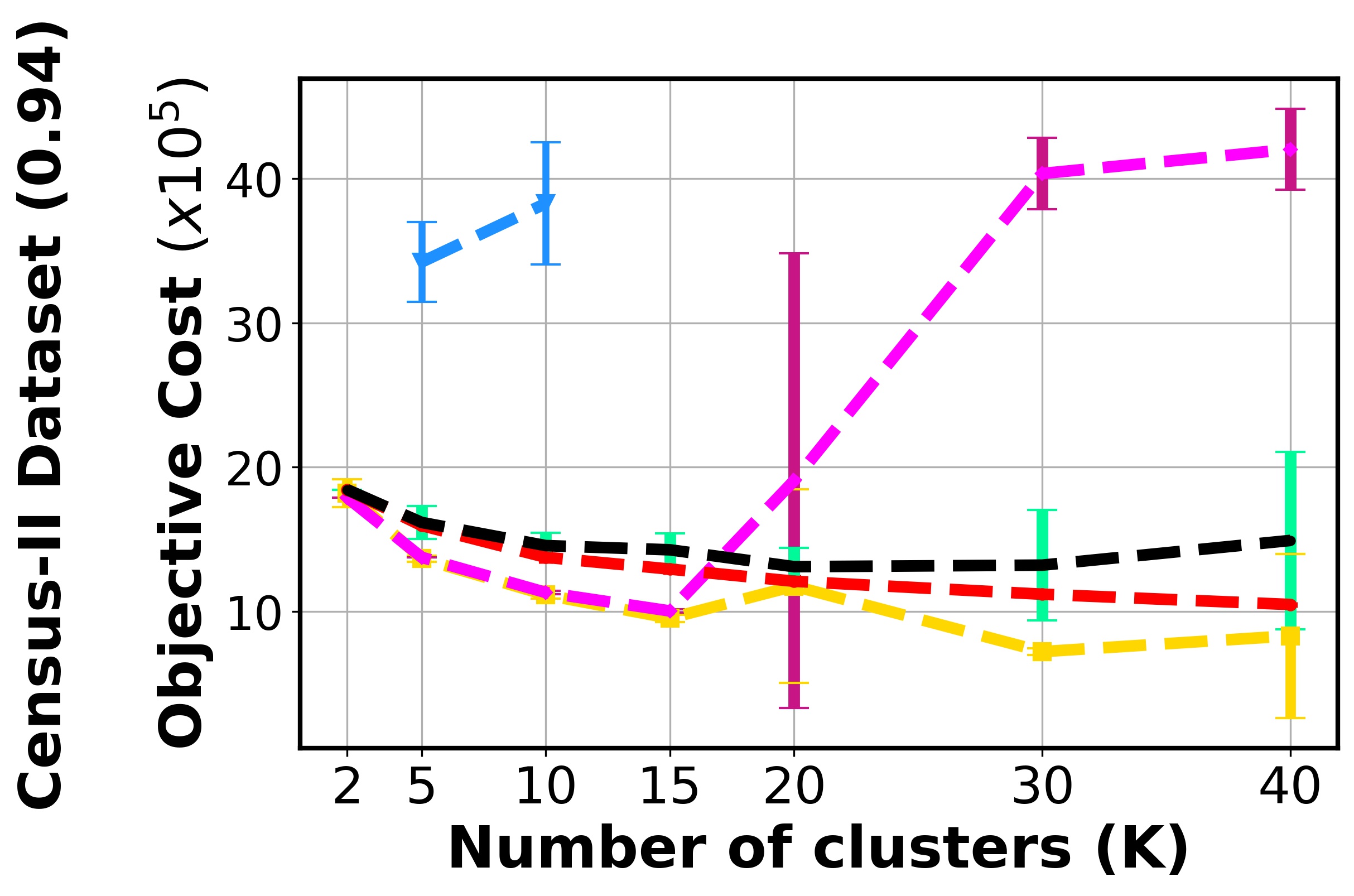}&\includegraphics[width=0.33\textwidth]{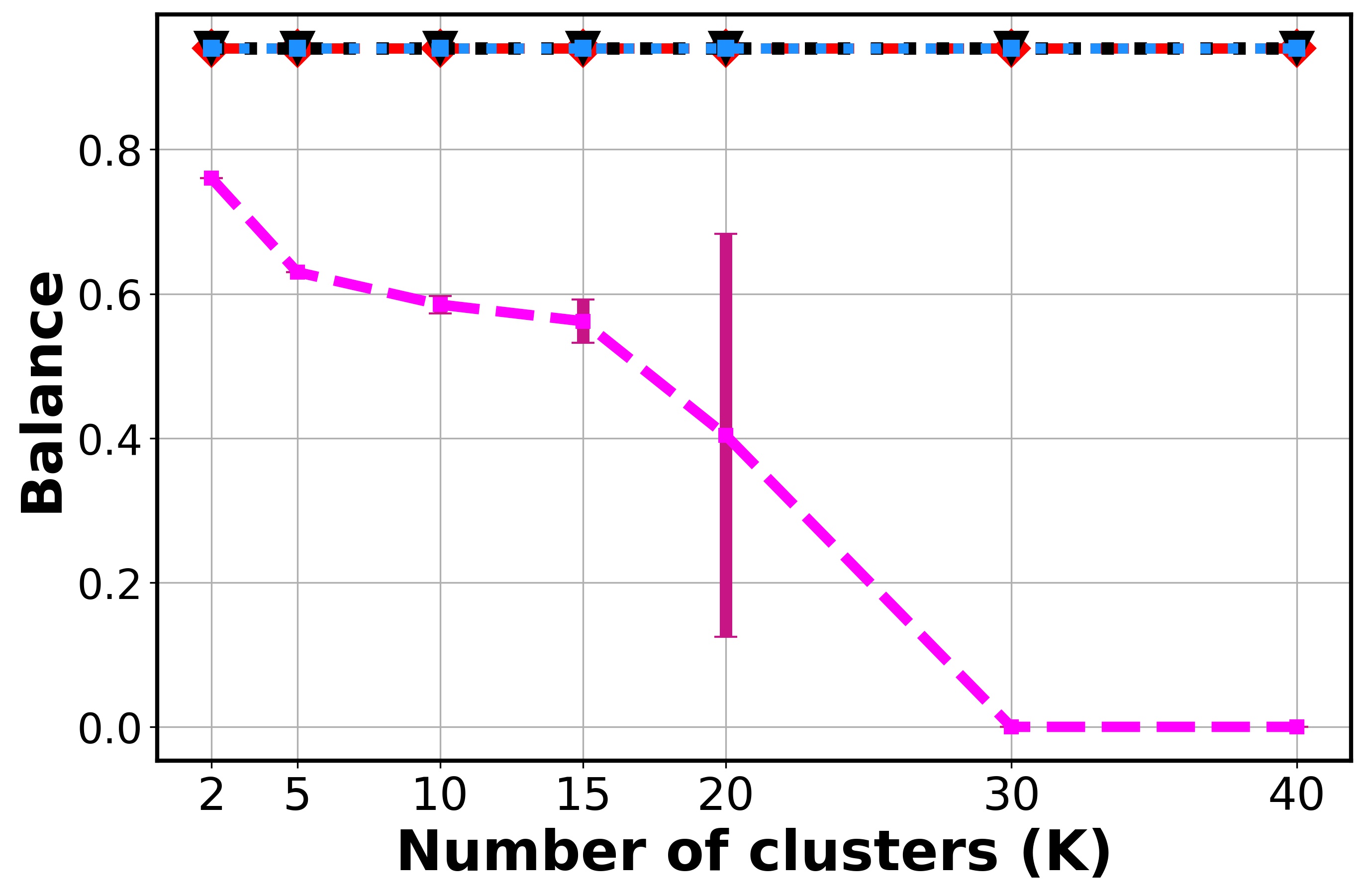}&\includegraphics[width=0.33\textwidth]{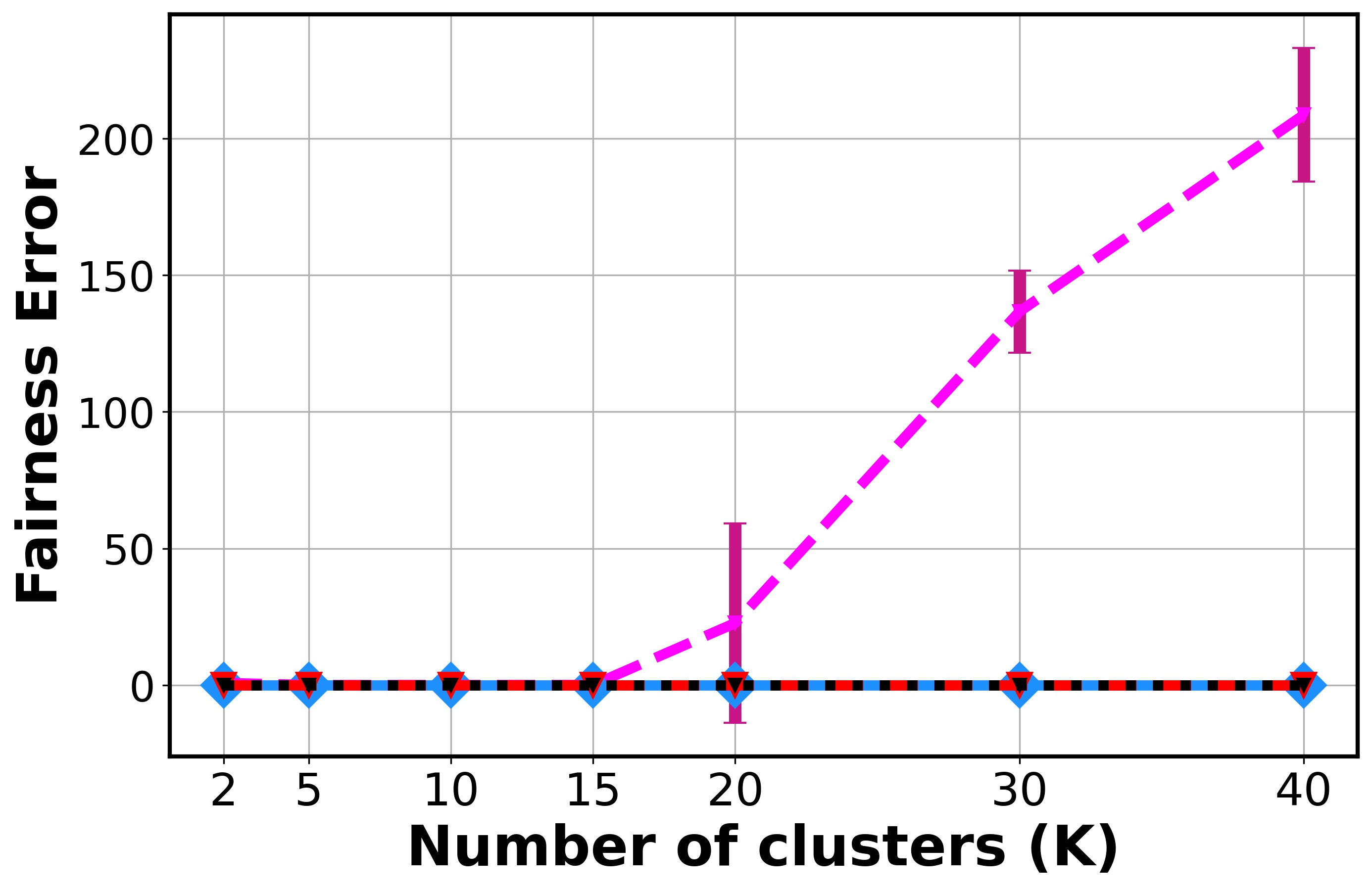}\\
\multicolumn{3}{l}{\includegraphics[width=\textwidth]{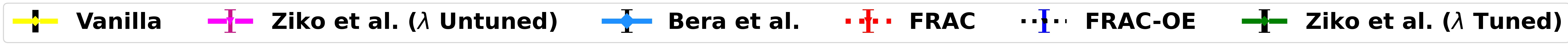}}
\end{tabular}
\caption{The line plot shows variation of evaluation metrics over varying number of cluster center for $k$-means setting. The hyper-tuned variation of Ziko et al. is available only for adult and bank dataset due to expensive computational requirements. For other datasets the hyper-parameter $\lambda$ is taken same as that is reported in Ziko et al. paper ie. $\lambda$=$9000$, $6000$, $6000$, $500000$ for Adult, Bank, Diabetes and Census II dataset respectively. On the similar reasons Bera et al. results for Census-II are evaluated for $k$=5 and $k$=10. (Best viewed in color)}
    \label{fig:VaryingK}
\end{figure}

In the $k$-median setting, it can been observed from the plots that \cite{backurs2019scalable} results in fair clusters with high objective cost. On the other hand \citeauthor{ziko2019variational} achieves better objective costs trading off for fairness. The $k$-median version of \RR, \RROE\  obtains the least fairness error and a \balance\ that is equal to the required dataset ratio ($\tau_{\ell}=\frac{1}{k}$) while having comparable objective cost.

\subsection{Comparison across varying number of clusters ($k$)}
\label{sec:varyExpCluster}
In this experiment, we measure the performance of the $k$-means version of the different approaches across all the datasets as the number of clusters increased from 2 to 40.  Fig. \ref{fig:VaryingK} summarises the results obtained for 2, 5, 10, 15, 20, 30, and 40 number of clusters on all datasets. It can be observed for all datasets that \cite{bera2019fair} maintain fairness constraints but with a much higher objective cost and standard deviation. For the largest dataset, Census-II, results are obtained for only $k=5$ and $k=10$ due to the large time complexity of solving the LP problem. Another interesting observation is that the LP-solver fails to return any solution for $k=2$. When we allow fine tuning of hyperparameter $\lambda$, it can be observed that the trend in the objective cost value for  \citeauthor{ziko2019variational} with increasing the number of clusters follows closely to that of the vanilla $k$-means objective cost on the Adult and Bank datasets. However, there is a significant deterioration in the \balance\ and fairness error measures. The results of \citeauthor{ziko2019variational} when using the $\lambda$ value reported in the paper for a particular $k$ for all the number of clusters show higher objective costs as well as fairness error. This indicates the sensitivity of the approach to the hyper-parameter $\lambda$. The proposed approach \RR\ gives the best result maintaining a relatively low objective cost without compromising fairness. Similarly, \RROE\ has marginal cost difference from \RR\ with same                        fairness guarantees over most of the datasets showing efficacy of the approach.

\begin{figure}[h!]
\centering
\begin{tabular}{@{}c@{}c@{}c@{}}
\includegraphics[width=0.340\textwidth]{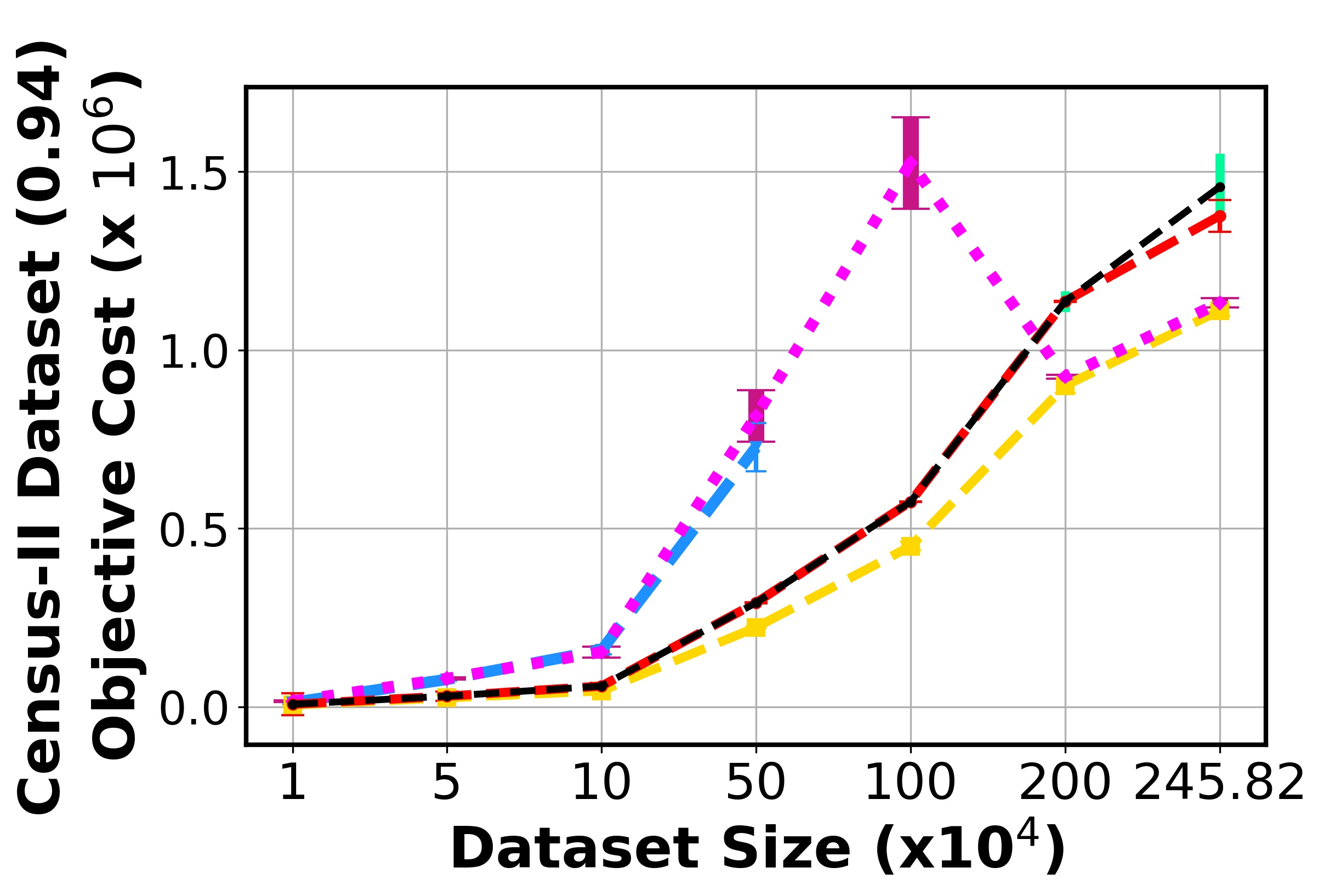}&\includegraphics[width=0.330\textwidth]{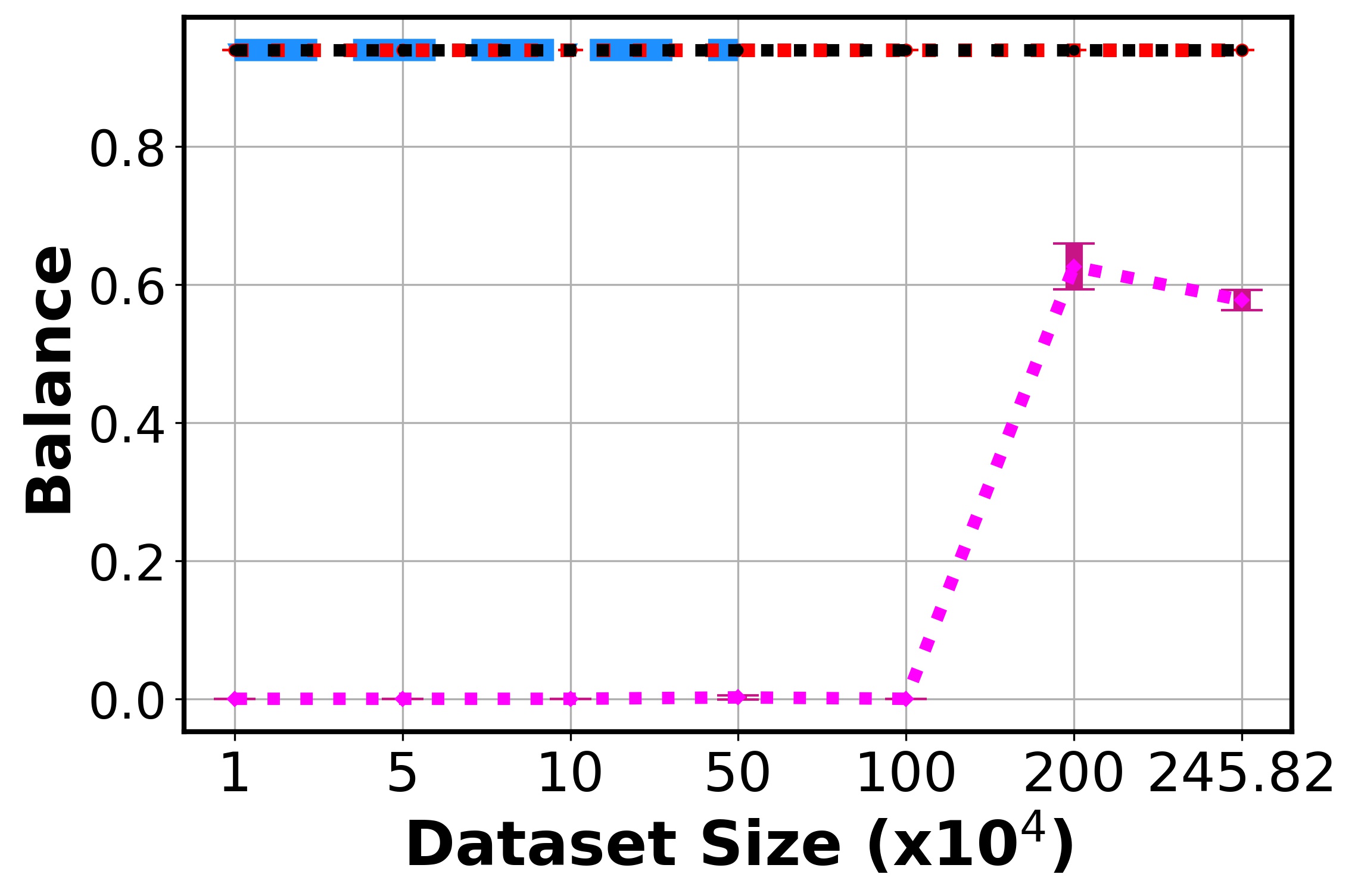}&\includegraphics[width=0.330\textwidth]{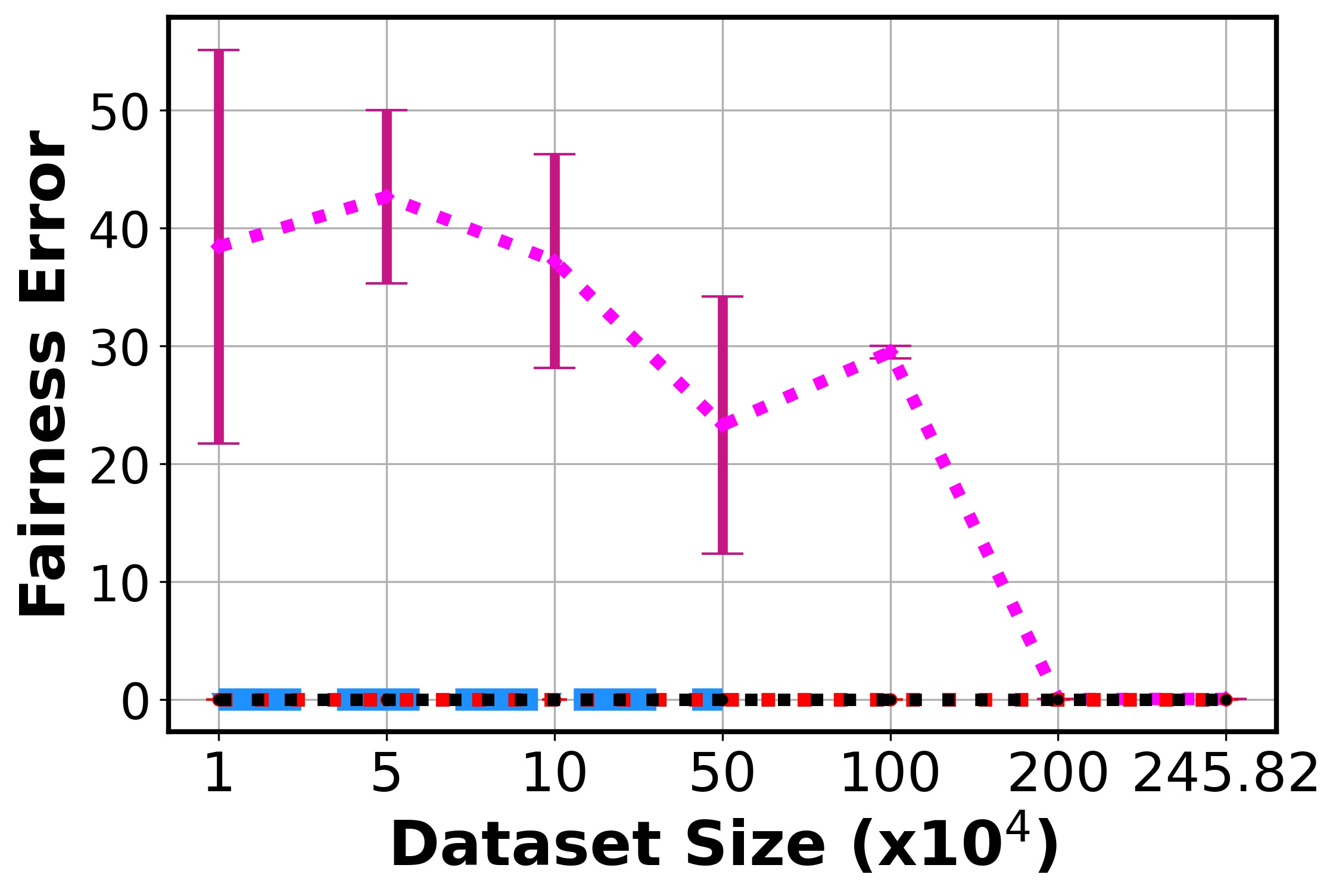}\\ 
\multicolumn{3}{l}{\includegraphics[width=\textwidth]{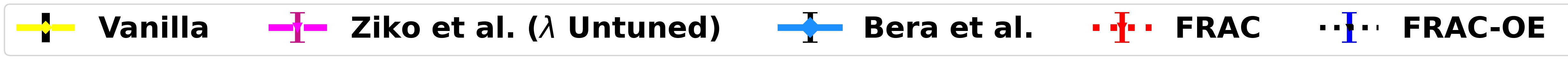}}
\end{tabular}
  \caption{The line plot shows variation of evaluation metrics over varying data set size for $k$(=$10$)-means setting. The hyper-parameter $\lambda$=$500000$ is taken same as that is reported in Ziko et al. paper for Census-II data set due to expensive computational requirements.  On the similar reasons Bera et al. results for Census-II are evaluated up to $50 \times 10^4$. The target balance for Census-II is evident from plot axes and complete data set size is $245.82 \times 10^4$. (Best viewed in color) }
    \label{fig:VaryingDatasetSize}
\end{figure}
\subsection{Comparison across varying data set sizes}
\label{sec:VaryDatasetSizeMetrics}
In this experiment, we measure the performance of $k$(=$10$)-means version of different approaches as number of points in data set increases in largest data set -- Census-II. Fig. \ref{fig:VaryingDatasetSize}. plots the results for evaluation metrics on data set size increasing from $10000$ to complete size of $2458285$ points. The plot clearly reveals that \RR, \RROE, and \cite{bera2019fair} are able to maintain strict fairness constraints. But  \cite{bera2019fair} is able to achieve fairness guarantees at higher objective cost. Due to high computation requirements for \cite{bera2019fair} (refer Section \ref{sec:runtimeAnalysis}), we limit the results up to $500,000$ number of points. For \cite{ziko2019variational}, owning to high tuning time (refer run time analysis section \ref{sec:runtimeAnalysis}) we use the hyper-parameter value for Census-II same as that reported in \cite{ziko2019variational} ie. $\lambda$=$500000$ for complete data set. Though initially \cite{ziko2019variational} is having performance close to other approaches but objective cost increases as data set size increases. One reason for this can be the hyper-parameter value used for approach. It may also be noted that, as the data set size reaches to completion, the objective cost improves to that of vanilla clustering but this comes at significant deterioration in fairness metrics. Both \balance\ and fairness error is quite far from the required target of $0.94$ and $0.0$ respectively. On the other hand  our proposed algorithms \RR\ and \RROE\ achieves strict fairness guarantees with slight increase in objective cost from vanilla clustering. Among \RR\ and \RROE, both have marginal difference in objective cost.

\subsection{ Additional Analysis on Proposed Algorithms    }
In this section we perform additional study on \RR\ and \RROE\ to illustrate their effectiveness.

\subsubsection{\RR\ vs \RROE}
\label{FRACvsFRACOE_ablation}
While \RR\ uses round-robin allocation after every clustering iteration, \RROE\ applies the round-robin allocation only at the end of clustering. Both the approaches will result in a fair allocation, but might exhibit different objective costs. We conduct an experiment under the $k$-means setting with $k=10$ to study the difference in the objective costs for the two approaches. Like other experiments, we conduct this experiment over ten independent runs and plot the mean objective cost (line) and standard deviation (shaded region) at each iteration over different runs. The plots in Fig. \ref{fig:costVariation} indicates that \RR\ has a lower objective cost at convergence than \RROE. The plot for \RROE\ follows the same cost variation as that of vanilla $k$-means in the initial phase, but at the end there is a sudden jump that overshoots the cost of \RR\ (to accommodate fairness constraints). Thus, applying fairness constraints after every iteration is better than applying it only once at the end. The plot also helps us experimentally visualize the convergence of both \RR\ and \RROE\ algorithms. It may be observed that the change in objective cost becomes negligible after a certain number of iterations. 

\begin{figure}[ht!]
\centering
\begin{tabular}{@{}c@{}c@{}c@{}}
\includegraphics[width=0.50\textwidth]{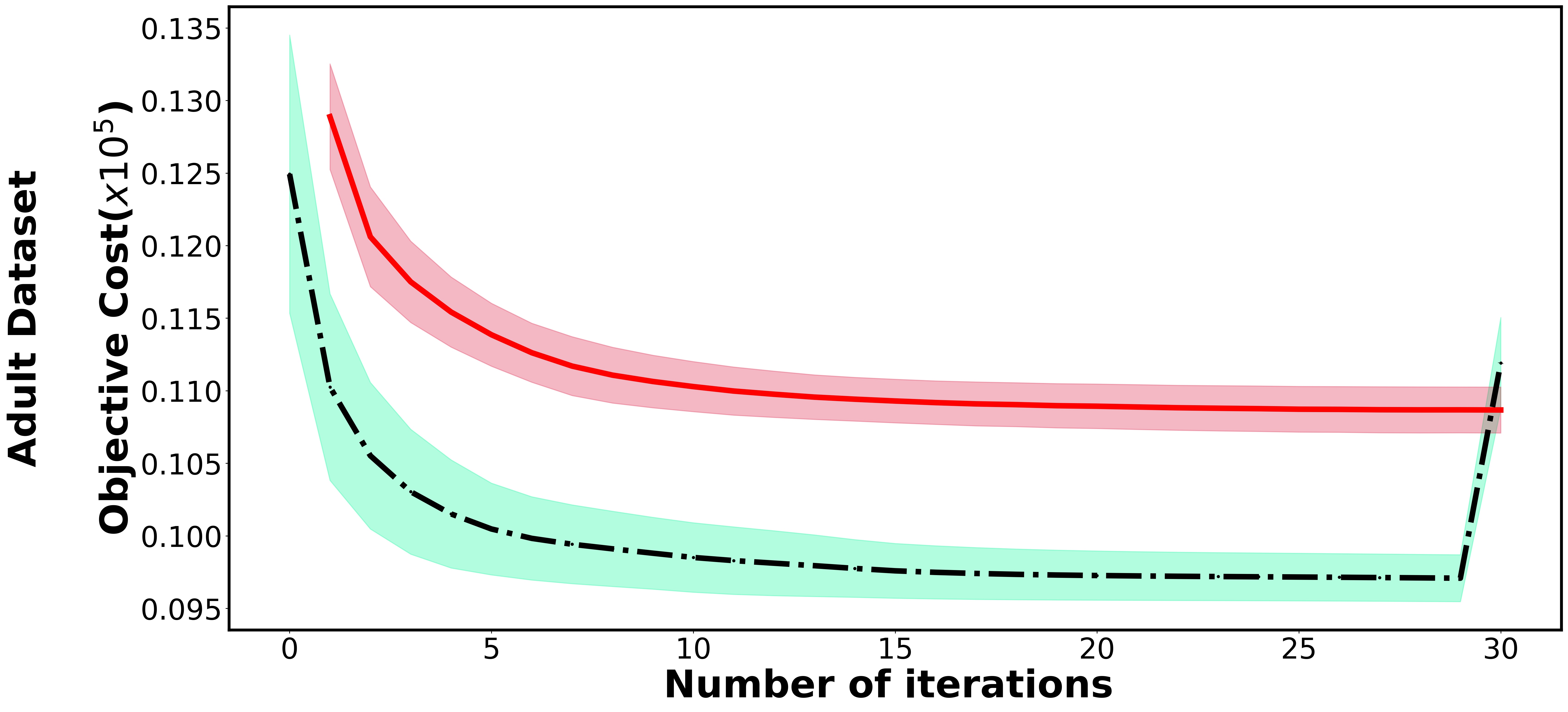}&\includegraphics[width=0.50\textwidth]{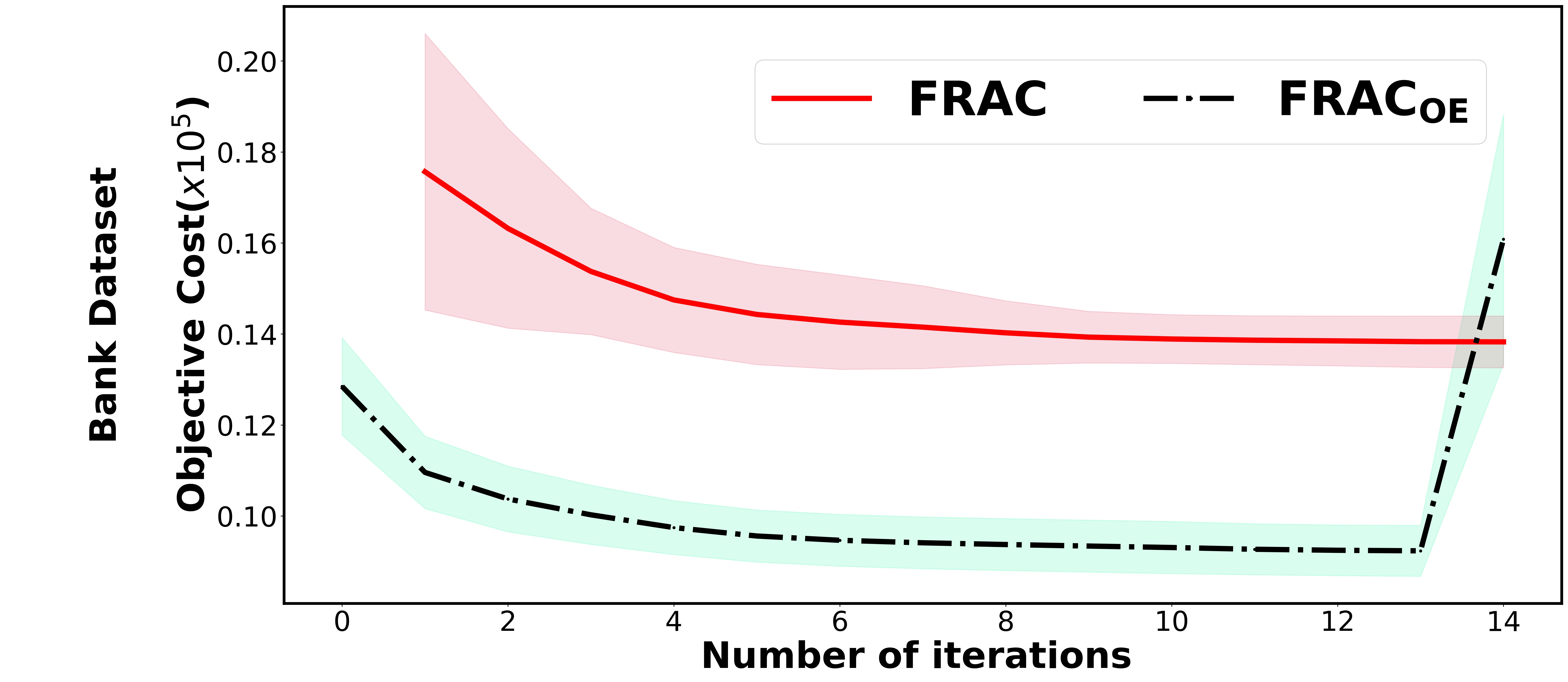}\\
\includegraphics[width=0.50\textwidth]{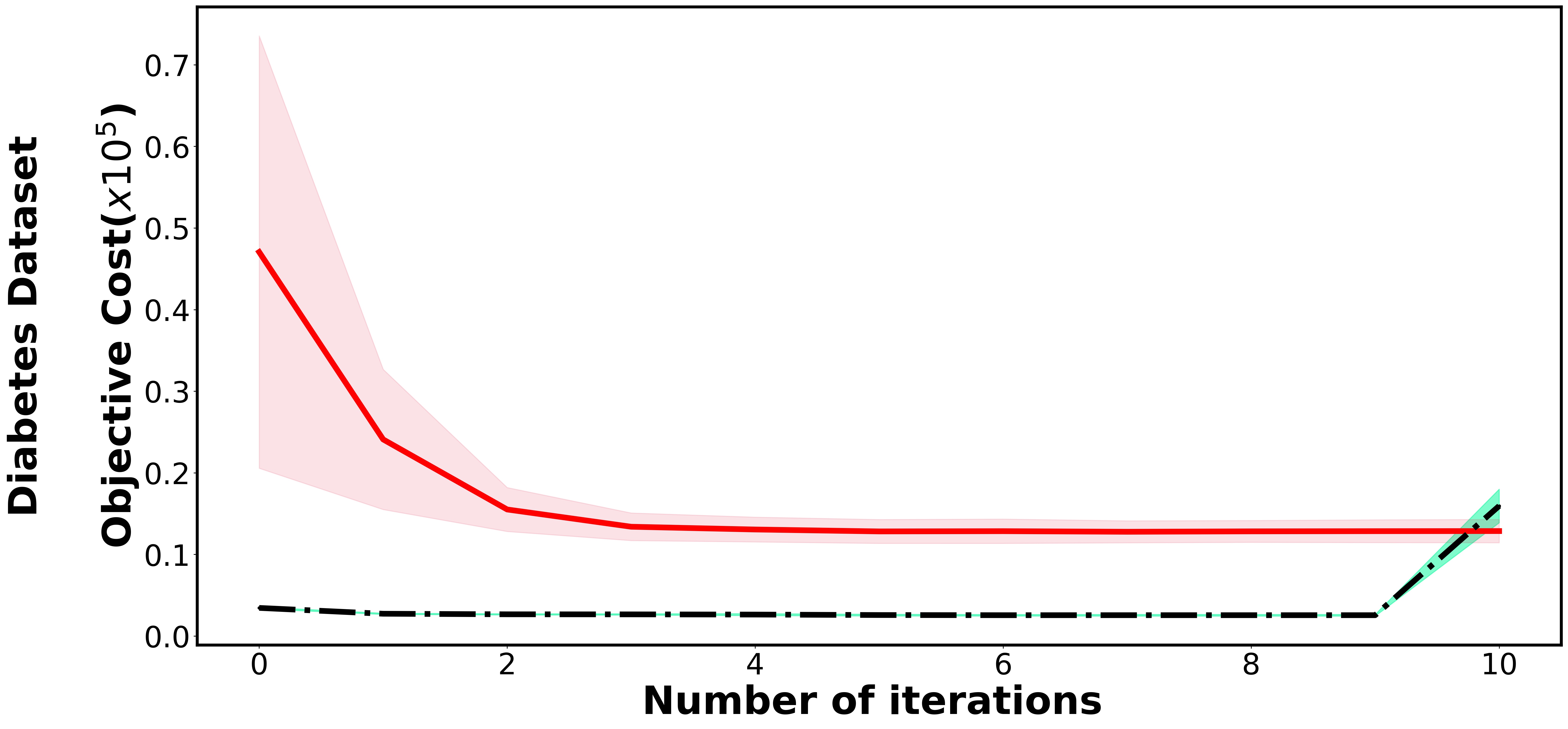}&\includegraphics[width=0.50\textwidth]{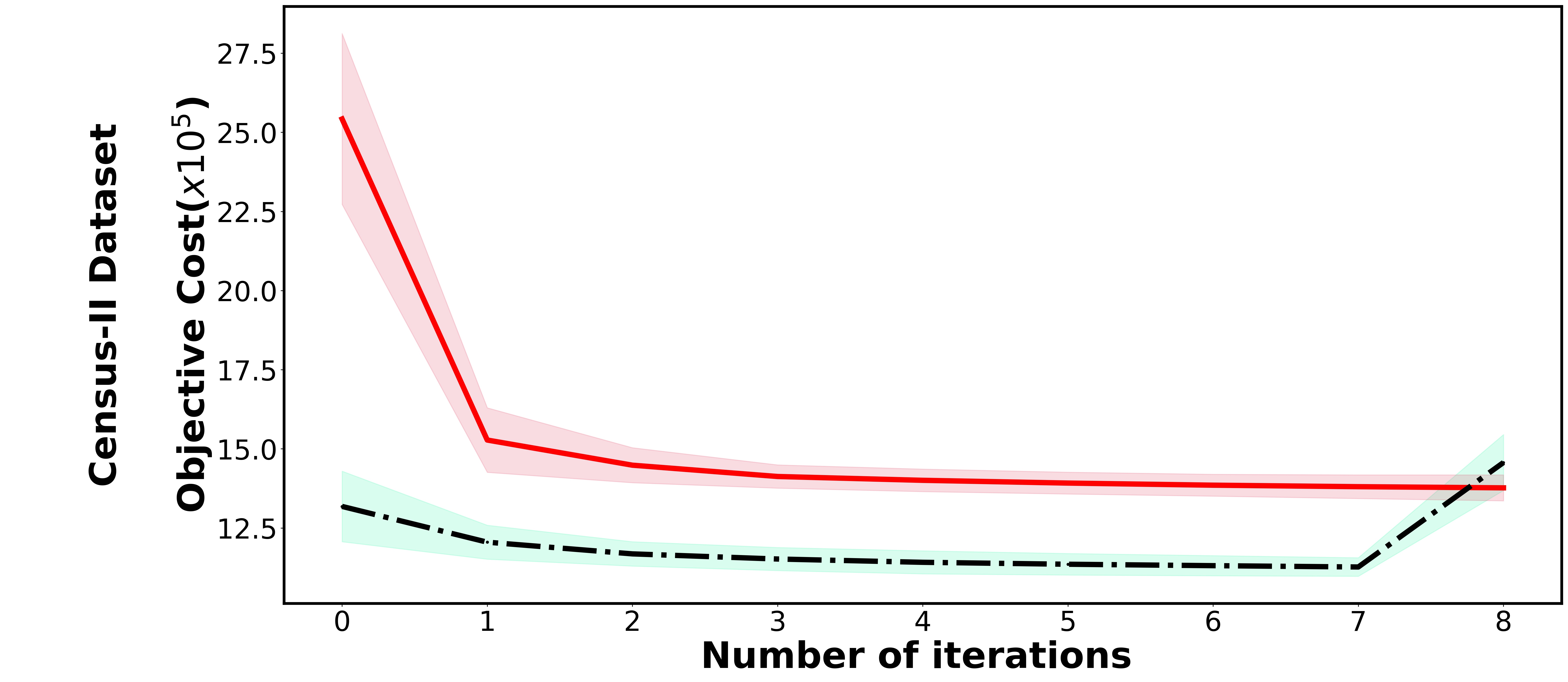} \\
\end{tabular}
  \caption{The cost variation over the iterations for different approaches in $k$-mean setting is plotted for $k$=10.}
    \label{fig:costVariation}
\end{figure}

\subsubsection{Impact of order in which the centers pick the data points }
\label{sec:CenterInvariant}
\RR\ assumes an arbitrary order of the centers for allocating data points at every iteration. We verify if the order in which the centers pick the data points impacts the clustering objective cost. We vary the order of the centers picking the data points for the $k$-mean clustering version with $k=10$. We report the objective cost variance computed across 100 permutations of the ten centers. Applying the permutations at every iteration in \RR\ is an expensive proposition. Hence we restrict the experiment to the \RROE\ version. The variance of the 100 final converged clustering objective costs (averaged over ten trials) is presented in Fig. \ref{fig:Permu} (a). It is evident from the plot that the variance is consistently extremely small for all datasets. Thus, we conclude that \RROE\ (and \RR\ by extension) is invariant to the order in which the centers pick the data points.

\begin{figure}
\centering
\begin{subfigure}[b]{0.5\linewidth}
\includegraphics[width=\textwidth]{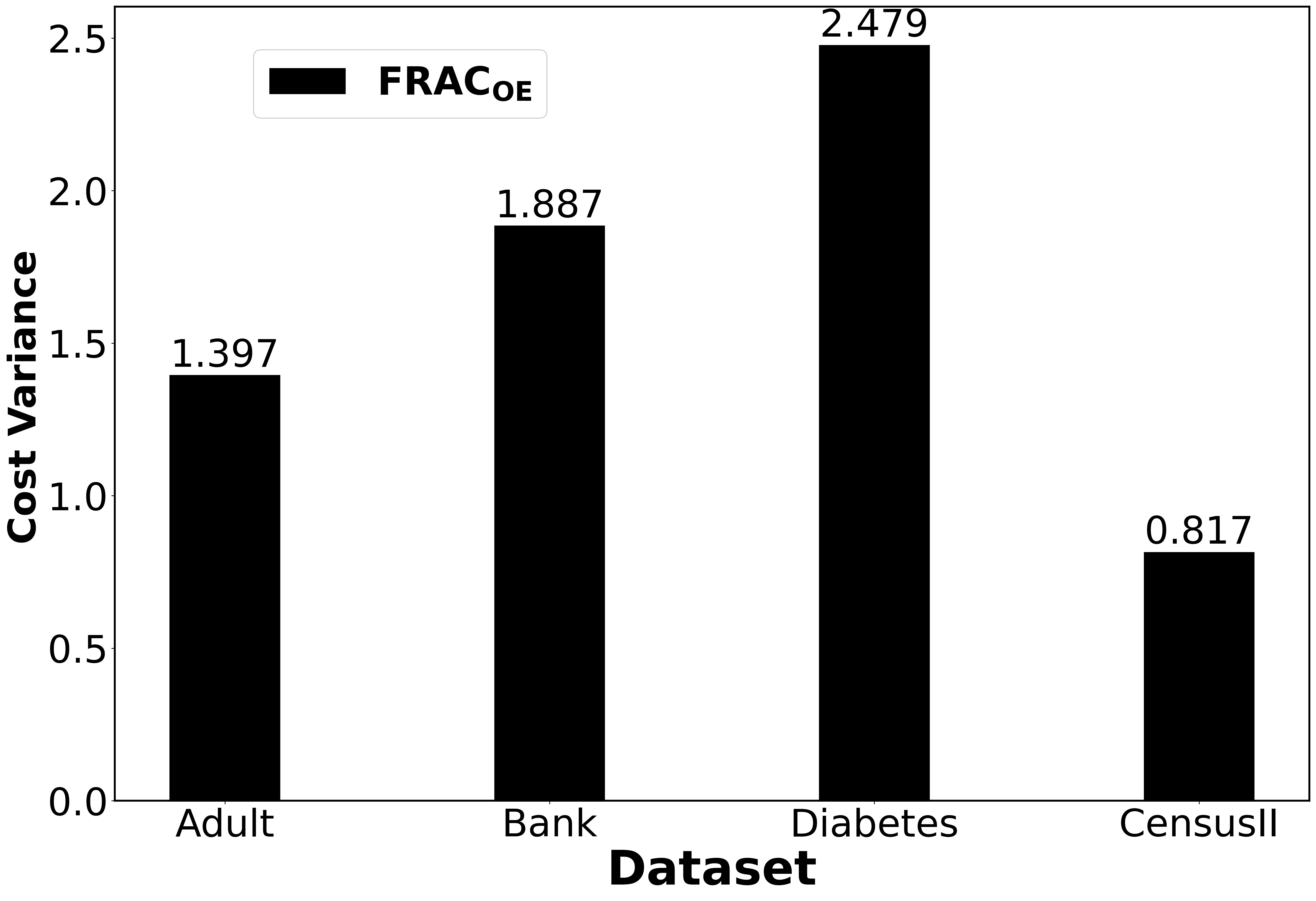}
  \caption{}
  \label{fig:sub1_permu}
\end{subfigure}%
\begin{subfigure}[b]{0.5\linewidth}
 \includegraphics[width=\textwidth]{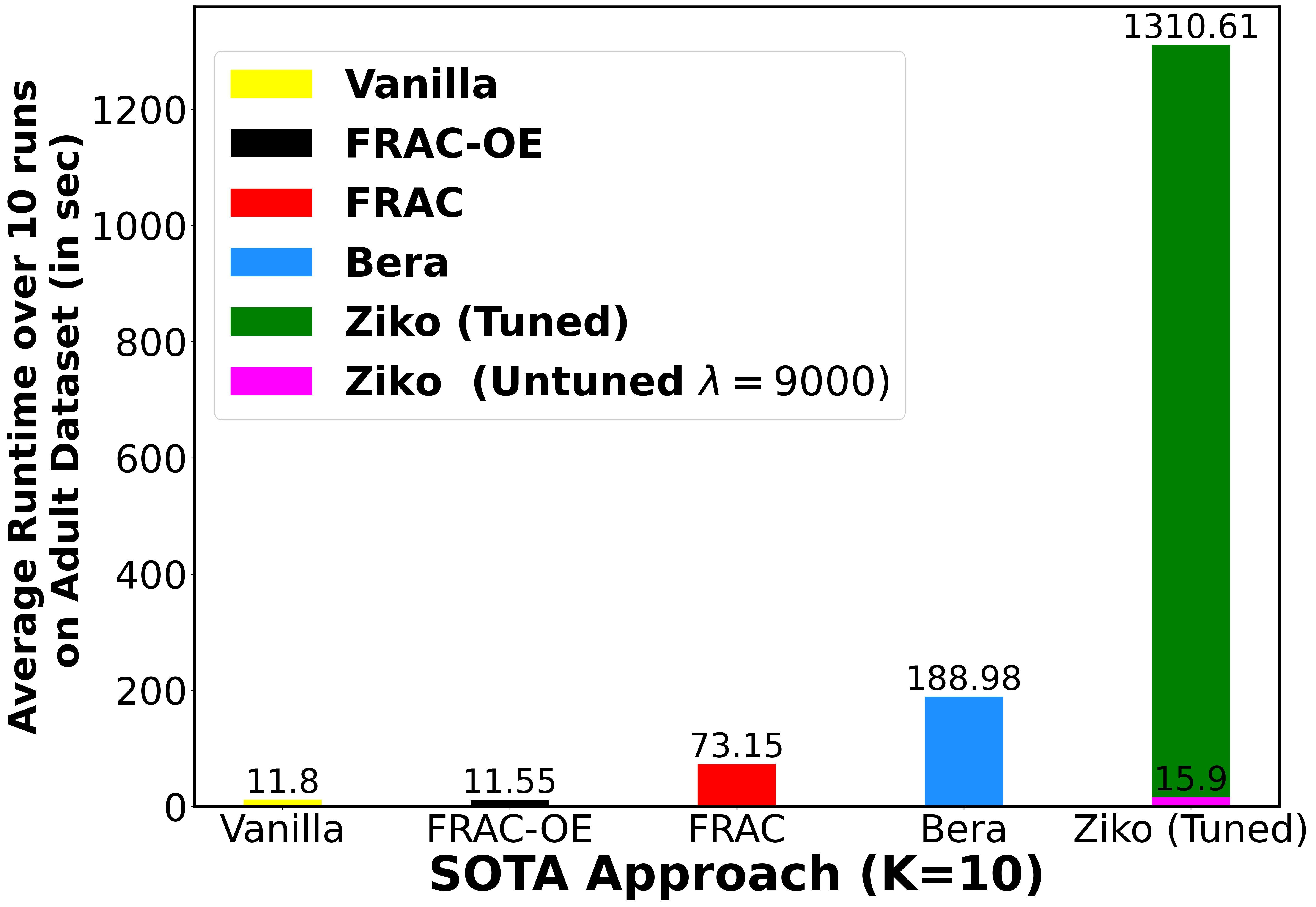}
  \caption{}
  \label{fig:sub2_permu}
\end{subfigure}
\caption{(a) Bar plot shows the variance in objective cost over different 100 random permutations of converged centers returned by standard unfair $k$-means clustering in \RROE. (b) $k$-means runtime analysis of different SOTA approaches on Adult dataset for $k$=10.}
    \label{fig:Permu}
\end{figure}

\subsubsection{Comparison for \tfair\ on fixed number of clusters($k$)}
\label{sec-tauRatio}
All the experiments till now considered the \balance\ to be same as the dataset ratio ($\tau_{\ell} = \frac{1}{k}$). But \RR\ and \RROE\ can be used to obtain any desired \tfair\ fairness constraints other than dataset proportion. The results for other $\tau$ vector values on $k$=$10$ number of clusters are reported in Table \ref{tab:tau-table}. We compare the performance of the proposed approach against \citeauthor{bera2019fair} that also allows for a desired \tfair\ fairness in a restrictive manner. 
\citeauthor{bera2019fair} reduces the degree of freedom  using $\delta$ parameter that controls the lower  and upper bound on number of points needed in each cluster belonging to a protected group. Experimentally $\delta$ can take values only in terms of dataset proportion $r_\ell$ for protected group $\ell \in [m]$, i.e. with lower bound as $r_\ell (1-\delta)$ and upper bound as $\frac{r_\ell}{(1-\delta)}$ . Further $\delta$ needs to be same across all the protected groups making it infeasible to achieve different lower bound for each protected group. Thus \citeauthor{bera2019fair} cannot be used to have any general fairness constraints for each protected group and can act as baseline only for certain $\tau_{\ell}$ values.
 In Table \ref{tab:tau-table} we present results for the $\tau$ corresponding to $\delta$=$0.2,0.8$. Additionally, our algorithms can achieve any generalized $\tau$ vectors like $[0.25,0.12]$, which makes more sense in real-world applications like requiring at least $25\%$ and $12\%$ points in each cluster for males and females. The objective cost obtained by \RR\ and \RROE\ is close to \cite{bera2019fair} but, the work by \cite{bera2019fair} is extendible to multi-valued problem. 

\begin{table}[!ht]
\begin{adjustbox}{width=\columnwidth}
\begin{tabular}{@{}llllll@{}}
\toprule
\multicolumn{1}{c}{\multirow{2}{*}{\textbf{Dataset}}} & \multicolumn{1}{c}{\multirow{2}{*}{\textbf{\begin{tabular}[c]{@{}c@{}}$\tau$- vector\end{tabular}}}} & \multicolumn{1}{c}{\multirow{2}{*}{\textbf{\begin{tabular}[c]{@{}c@{}}\RR\\ \\  Objective Cost\end{tabular}}}} & \multicolumn{1}{c}{\multirow{2}{*}{\textbf{\begin{tabular}[c]{@{}c@{}}\RROE\\  \\ Objective Cost\end{tabular}}}} & \multicolumn{2}{c}{\textbf{Bera et al.}} \\ \cmidrule(l){5-6} 
\multicolumn{1}{c}{} & \multicolumn{1}{c}{} & \multicolumn{1}{c}{} & \multicolumn{1}{c}{} & \multicolumn{1}{c}{\textbf{\begin{tabular}[c]{@{}c@{}}$\delta$\\  Value\end{tabular}}} & \multicolumn{1}{c}{\textbf{Objective Cost}} \\ \midrule
\multirow{3}{*}{\textbf{Adult}} & \textless 0.133, 0.066 \textgreater{} & 9804.65 ±  221.05 & 9616.51 ± 111.49 & 0.8 &  9515.30 ± 19.94 \\
 & \textless 0.535, 0.264 \textgreater{} & 10010.39 ±  211.27 & 10011.78 ± 239.73  & 0.2 &  9788.73 ± 23.32\\
 & \textless 0.25, 0.12 \textgreater{} & 9870.93  ±  261.24 & 9714.06 ± 157.45  & \multicolumn{2}{c}{\textit{Cannot be computed}} \\ \midrule
\multirow{3}{*}{\textbf{Bank}} & \textless 0.121, 0.056, 0.022 \textgreater{} & 9210.38 ±  640.76 & 9043.51 ± 461.23 & 0.2 &  9588.30 ± 48.82 \\
 & \textless 0.485, 0.225, 0.089 \textgreater{} & 10982.63 ± 1228.28 & 11317.61 ± 1310.32  & 0.8 &  8472.65 ± 37.30 \\
 & \textless 0.25, 0.10, 0.04 \textgreater{} & 9548.68 ± 540.86 & 9465.35 ± 476.88 & \multicolumn{2}{c}{\textit{Cannot be computed}} \\ \bottomrule
\end{tabular}
\end{adjustbox}
\caption{$k$-means objective cost for \tfair\ for adult and bank dataset  for $k$=$10$ clusters. }
\label{tab:tau-table}

\end{table}

\subsection{Run-time Analysis}
\label{sec:runtimeAnalysis}
Finally, we compare the run-time of the different approaches for the $k$(=10)-means clustering versions on the Adult dataset. The average run-time over 10 different runs is reported in Fig. \ref{fig:Permu} (b). It can be clearly seen that the run-time of \RR\ is significantly better than the fair SOTA approaches. The run-time of \citeauthor{ziko2019variational} is quite high due to hyper-parameter tuning required to find the best suited $\lambda$ value. The run-time of \citeauthor{ziko2019variational} without hyper-parameter tuning is comparable to vanilla clustering. However, without hyper-tuning it has been observed from previous sections that \citeauthor{ziko2019variational}'s performance can deteriorate significantly on the fairness constraints. \RROE\ runtime has marginal difference from vanilla clustering runtime since \RROE\ applies a single round of fair assignment following vanilla clustering. \citeauthor{bera2019fair} requires double the time of \RR. In general, LP formulations to fair clustering are observed to have higher complexities. In contrast, \RR\ is able to achieve better objective costs and comparable fairness measures with significantly less complexity.

Motivated by \cite{runtimeBlackAnalysis}, we further study the runtime behaviour across varying number of datapoints and varying number of clusters. For the scalablity study, we perform the analysis using Census-II as it is largest dataset. 

\subsubsection{Runtime comparison with    number of cluster(k)}
\label{sec:runtimeVaryk}
In this study we conduct experiment to find the variation in runtime as number of clusters $k$ varies from $2$ to $40$. We observe the results for $2, 5, 10, 15, 20, 30$ and $40$. From the results summarized in Fig. \ref{fig:VaryingKRuntimeCensusII}, we can observe that \cite{bera2019fair} is having significantly high execution time. Thus, we limit the results upto $k$(=$5, 10$)-clustering. As pointed out in previous section \cite{bera2019fair}, LP fails to converge for $k$=$2$. 

\begin{figure}[h!]
\centering
\begin{tabular}{@{}c@{}}
\includegraphics[width=0.7\textwidth]{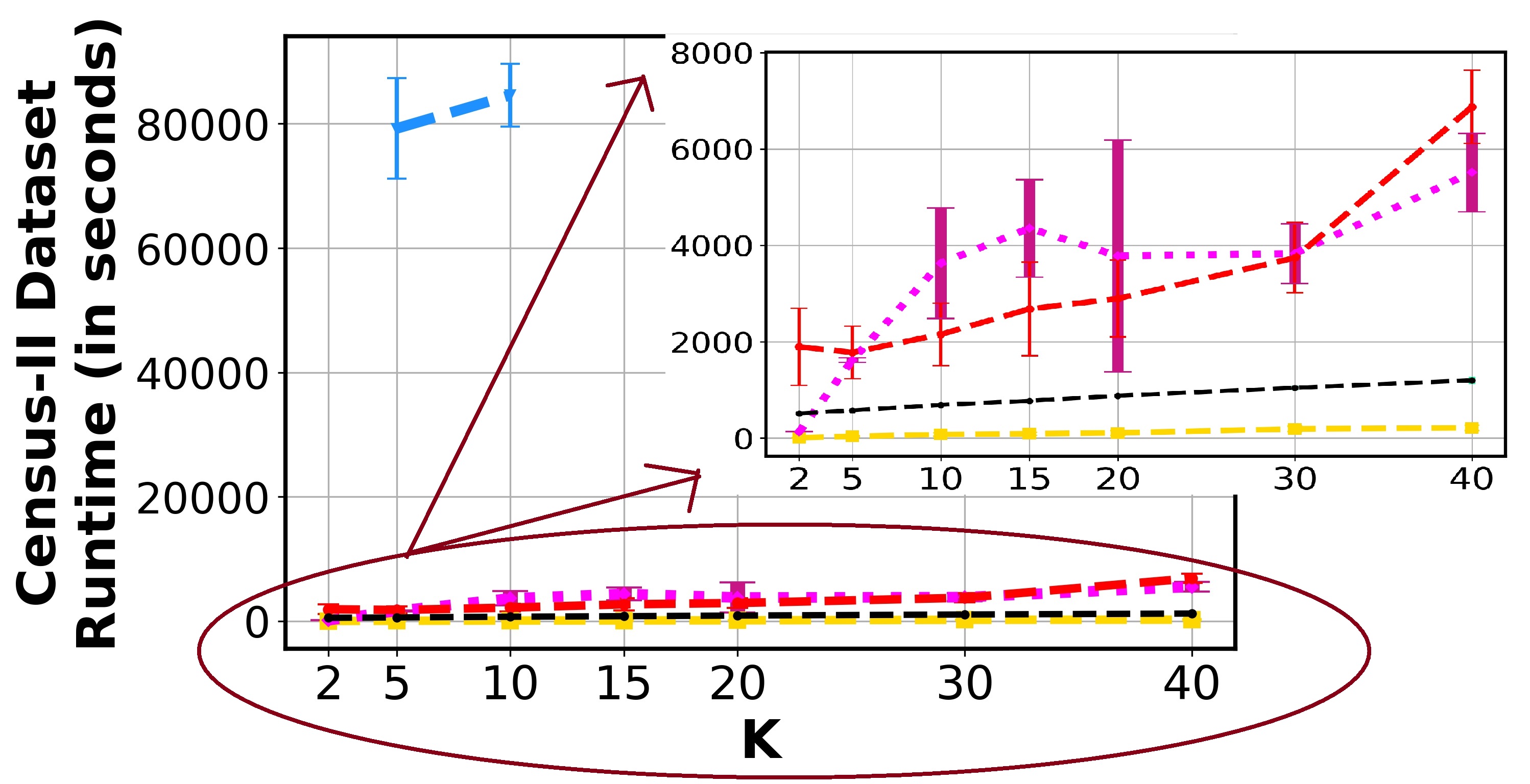}\\ 
\multicolumn{1}{l}{\includegraphics[width=\textwidth]{exp3_legend.png}}
\end{tabular}
  \caption{The line plot shows variation of runtime over varying number of clusters($k$) for $k$-means setting on complete dataset size. The hyper-parameter $\lambda$=$500000$ is taken same as that is reported in Ziko et al. paper for Census-II dataset due to expensive computational requirements.  On the similar reasons Bera et al. results for Census-II are evaluated for $k$=5 and $k$=10. For better visualization the results are zoomed out for approaches other than Bera et al. (Best viewed in color) }
    \label{fig:VaryingKRuntimeCensusII}
\end{figure}
We can clearly see from the plots that \RROE\ has runtime close to vanilla clustering. For \cite{ziko2019variational}, even in untuned version (using same hyper-parameter as reported in \citeauthor{ziko2019variational} paper) we still have runtime close to proposed \RR. Tuning the hyper-parameter will result in significant increase in overall runtime for the approach as observed in Section \ref{sec:runtimeAnalysis}.

\subsubsection{Runtime comparison across varying data set size}
\label{sec:runtimeVaryDatasetSize}

\begin{figure}[h!]
\centering
\begin{tabular}{@{}c@{}}
\includegraphics[width=0.7\textwidth]{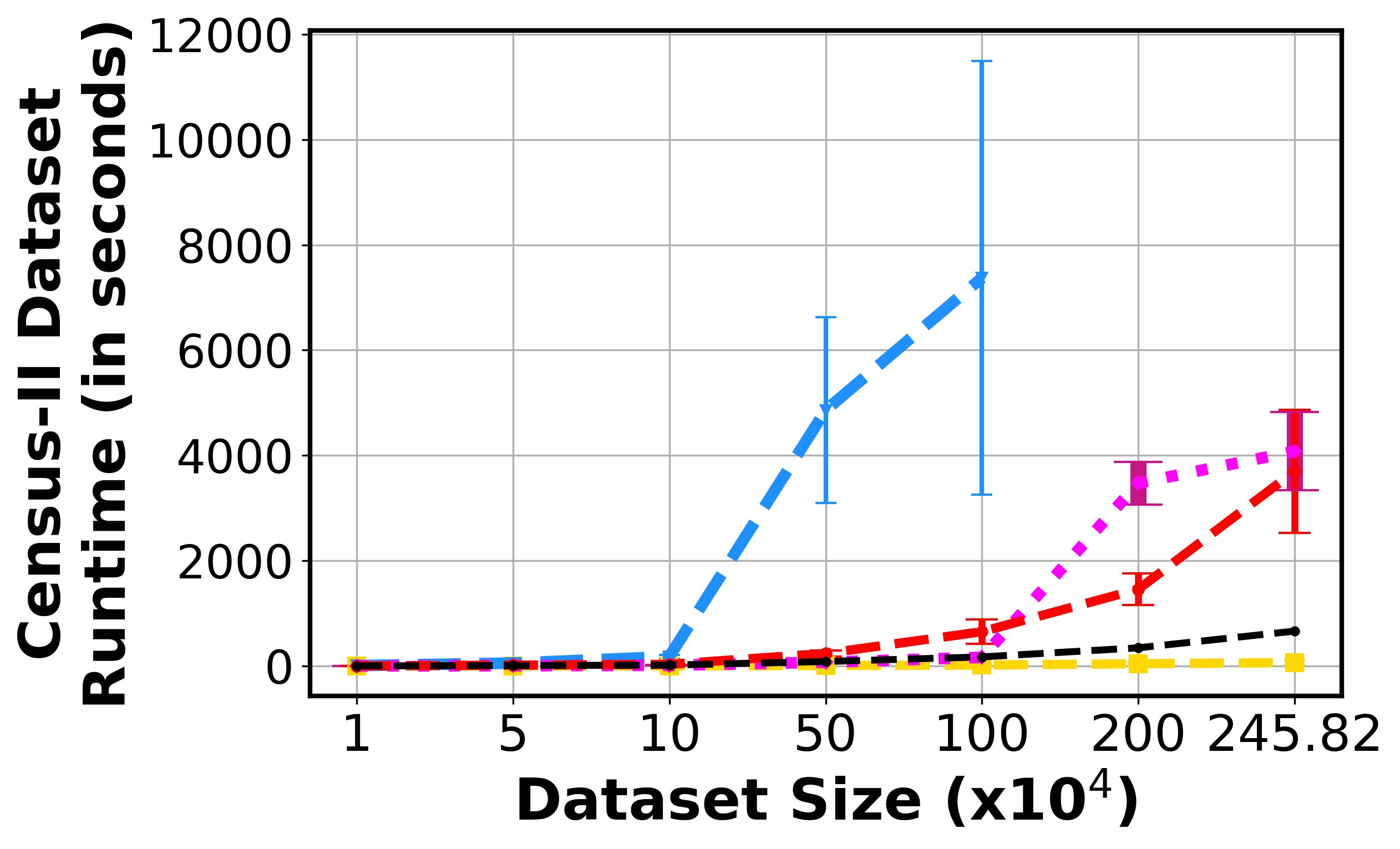}\\ 
\multicolumn{1}{l}{\includegraphics[width=\textwidth]{exp3_legend.png}}
\end{tabular}
  \caption{The line plot shows variation of runtime over varying dataset size (upto complete dataset size of $245.82 \times 10^4$) for $k$=$10$-means setting. The hyper-parameter $\lambda$=$500000$ is taken same as that is reported in Ziko et al. paper for Census-II dataset due to expensive computational requirements.  On the similar reasons Bera et al. results for Census-II are evaluated for dataset size  of $10,000$, $50,000$ and $100000$.  (Best viewed in color) }
    \label{fig:VaryingDatasetSizeRuntimeCensusII}
\end{figure}

We study the scalability of different approaches to increase in the data set size. We conduct experiments using the largest data set, Census-II at $k$=$10$. For \cite{bera2019fair}, plots in Fig. \ref{fig:VaryingDatasetSizeRuntimeCensusII} reveal that the run time significantly increases with $50\times10^4$ points in the data set. So we limit the study up to this size. The run time for untuned \cite{ziko2019variational} is close to vanilla clustering. However, the gap starts to widen after a certain number of data points. On the contrary, our proposed \RROE\ follows a similar trend close to vanilla and does not deteriorate with the varying number of clusters showing the efficiency of \RROE. The \RR\ being an in-processing heuristic has a run time larger than vanilla clustering but is comparable to untuned \cite{ziko2019variational}. Tuning the \cite{ziko2019variational} will result in additional overhead.

\section{Discussion}
We proposed a novel \tfair\ fairness notion. The new notion generalizes the existing \balance\ notion and admits an efficient round-robin algorithm to the corresponding fair assignment problem. We also showed that our  proposed algorithm, \RROE, (i) achieves $2(\alpha+2)$-approximate solution up to three clusters, and (ii) achieves  $2^{k-1}(\alpha+2)$-approximate guarantees to general $k$ with $\tau$=$1/k$. Current proof techniques for $k \leq 3$ requires intricate  case analysis  which becomes  intractable  for larger $k$. However,  our experiments show that \RR\  outperforms SOTA approaches in terms of objective cost and fairness measures even for $k>$3. We also proof the cost approximation for general $\tau$ vector and show convergence analysis for \RROE. An immediate future direction is to analytically prove   $2(\alpha+2)$-approximation guarantee for  general $k$. 

It is worth noting here that the  $\tfair$ fairness ensures the  \balance\ property. However,  if one is to use \balance\ as a constraint, one could get a better approximation guarantee. Surprisingly, we observe from our experiments that this is not the case. We leave the theoretical and experimental  analysis of these two notions of fairness in the presence of large data as an interesting future work. 
Apart from above mentioned immediate  future directions, extending the current  work to multi-valued multiple protected attributes similar to the one proposed by \cite{bera2019fair}, or 
achieving  the notion of individual fairness along while maintaining group fairness are also interesting research problems.  

\section*{Declaration}

\textbf{Funding}: The research is funded by Department of Science \& Technology, India under grant number SRG/2020/001138 (Recipient name- Dr. Shweta Jain).\\

\noindent\textbf{Conflicts of interest/Competing interests}: No potential competing interest was reported by the authors.\\

\noindent\textbf{Availability of data and material}: All datasets used in the experiments are publicly available on UCI repository.\\

\noindent\textbf{Code availability}: The code has been made publicly available at \url{https://github.com/shivi98g/Fair-k-means-Clustering-via-Algorithmic-Fairness}\\

\noindent\textbf{Ethics approval}: Not applicable\\

\noindent\textbf{Consent for publication }: The paper is the authors' own original work, which has not been previously published elsewhere. The paper is not currently being considered for publication elsewhere. The paper reflects the authors' own research and analysis in a truthful and complete manner. The paper properly credits the meaningful contributions of co-authors and co-researchers. 

\bibliography{references}
\end{document}